\documentclass{article}

\PassOptionsToPackage{numbers, compress}{natbib}
\PassOptionsToPackage{table}{xcolor}

 \usepackage[main, preprint]{neurips_2026}

\usepackage[utf8]{inputenc} 
\usepackage[T1]{fontenc}    
\usepackage{hyperref}       
\usepackage{url}            
\usepackage{booktabs}       
\usepackage{amsfonts}       
\usepackage{nicefrac}       
\usepackage{microtype}      
\usepackage{xcolor}         

\usepackage{amsmath,amssymb,amsthm,mathtools}

\usepackage{amsfonts}

\theoremstyle{plain}
\newtheorem{theorem}{Theorem}
\newtheorem{lemma}[theorem]{Lemma}
\newtheorem{proposition}[theorem]{Proposition}
\newtheorem{corollary}[theorem]{Corollary}
\theoremstyle{definition}
\newtheorem{definition}{Definition}
\theoremstyle{remark}
\newtheorem{remark}{Remark}

\newcounter{algorithm}

\usepackage{titletoc}

\usepackage{makecell}

\usepackage{adjustbox}
\usepackage{booktabs}
\usepackage{makecell}
\usepackage{multirow}
\usepackage{graphicx}
\usepackage{xcolor}

\definecolor{cvprblue}{rgb}{0.21,0.49,0.74}

\usepackage[T1]{fontenc}
\IfFileExists{inconsolata.sty}{%
  \usepackage[scaled=0.92]{inconsolata}%
  \newcommand{\algttfamily}{\fontfamily{zi4}}%
}{%
  \usepackage{courier}%
  \newcommand{\algttfamily}{\ttfamily}%
}
\usepackage{amsmath,amssymb}
\usepackage[most]{tcolorbox}
\usepackage{xcolor}
\usepackage{lmodern}

\usepackage{booktabs}
\usepackage{xcolor}
\usepackage{subcaption}

\definecolor{algcomment}{RGB}{74,133,135}
\definecolor{algfunc}{RGB}{219,66,145}
\definecolor{algframe}{RGB}{30,30,30}
\definecolor{algbg}{RGB}{248,248,248}

\newcommand{\cmt}[1]{\textcolor{algcomment}{#1}}
\newcommand{\fn}[1]{\textcolor{algfunc}{\texttt{#1}}}
\newcommand{\code}[1]{\texttt{#1}}
\newcommand{\algcodefont}{\algttfamily\fontsize{9.2}{11.2}\selectfont}

\AtBeginDocument{%
  \setlength{\abovedisplayskip}{1pt}%
  \setlength{\belowdisplayskip}{1pt}%
  \setlength{\jot}{1pt}%
}

\title{Drift Flow Matching}

\author{%
  \textbf{Chenrui Ma\textsuperscript{1,2} \quad
  Xi Xiao\textsuperscript{3} \quad
  Lin Zhao\textsuperscript{4}} \\
  \textbf{Tianyang Wang\textsuperscript{3} \quad
  Ferdinando Fioretto\textsuperscript{2} \quad
  Yanning Shen\textsuperscript{1}} \\[0.2cm]
  \normalfont\textsuperscript{1}University of California, Irvine \quad
  \textsuperscript{2}University of Virginia \\
  \normalfont\textsuperscript{3}University of Alabama at Birmingham \quad
  \textsuperscript{4}Northeastern University \\
  \texttt{\{chenrum,yannings\}@uci.edu} \quad
  \texttt{\{zzz8fa,fioretto\}@virginia.edu} \\
  \texttt{\{xxiao,tw2\}@uab.edu} \quad
  \texttt{zhao.lin1@northeastern.edu} \\
}

\begin{document}

\maketitle

\begin{abstract}
Iterative generative models such as Flow Matching and Diffusion models have demonstrated strong test-time scaling behavior, where additional inference computation can improve generation quality.
In contrast, Drift Models offer efficient one-step generation, but their direct generation paradigm limits such flexibility.
In this work, we propose \textbf{Drift Flow Matching (DFM)}, a framework that connects drifting generative modeling with flow-based iterative generation.
DFM preserves the efficiency of direct transport maps while enabling generation to be refined through multiple inference steps when desired.
This bridges the gap between one-step Drift Models and multi-step Flow Matching methods, and provides a novel generative paradigm that can adapt sampling computation to different quality--efficiency requirements. 
Extensive experiments across different tasks and datasets demonstrate the effectiveness and generality of the proposed framework.
\end{abstract}

\section{Introduction}
\label{sec:intro}

Flow Matching~\cite{lipman2023flow, albergo2023building, esser2024scaling} and Diffusion~\cite{albergo2023stochastic, cai2025diffusion, song2021scorebased} models describe generation as the evolution of distributions through continuous-time dynamics, usually formulated by ordinary differential equations (ODEs)~\cite{lipman2024flow} and stochastic differential equations (SDEs)~\cite{lai2025principlesdiffusionmodels}, respectively.
At inference time, they generate samples by iterative simulation, which requires multiple inference steps~\cite{ma2025learningstraightflowsvariational}.
This iterative procedure also provides an important test-time scaling capability: using more inference steps can often lead to better generation quality, stronger controllability, and improved distribution coverage.
In contrast, Drift Models learn to transport a source distribution directly to a target distribution through a pushforward map that evolves during training~\cite{deng2026generative,he2026sinkhorn,lai2026unified}.
Such models enable efficient one-step generation, but they lack a natural mechanism to improve results by increasing inference steps at test time.

To address these limitations, this paper proposes \textbf{Drift Flow Matching (DFM)}, a framework that equips Drift Models~\cite{deng2026generative,he2026sinkhorn,lai2026unified} with test-time scaling capability.
The key idea is to learn a two-time transport model that can move samples from any current time $t$ to any future time $r>t$ along a Flow Matching probability path from $p_0$ to $p_1$. 
However, the learning principle is different: rather than regressing a pointwise conditional velocity target, for each sampled time pair $(t,r)$, it transports the intermediate state $X_t$ from its corresponding marginal $p_t$ to a predicted distribution $q^\theta_{t,r}$ and applies drift supervision that pushes this predicted marginal toward the true target marginal $p_r$. 
This distinction allows DFM to interpolate between two generation regimes: With one-step inference, DFM behaves like a Drift Model and directly maps the source distribution to the target distribution. With multi-step inference, DFM progressively transports samples along intermediate states, thereby inheriting the test-time scaling behavior of Flow Matching and Diffusion~\cite{ma2025scaling,singhal2025general,li2025reflect}. 
Figure~\ref{fig:illustration} illustrates the relationship between \textbf{DFM} and prior work.
Extensive experiments across different tasks and datasets demonstrate the effectiveness of DFM.
Our results show that DFM provides a flexible trade-off between sampling speed and generation quality, bridging efficient one-step Drift Models and iterative Flow Matching methods.

\section{Related Works} \label{sec:rw}

\textbf{Flow Matching and Diffusion.}
An important direction in Flow Matching~\cite{lipman2023flow, albergo2023building, esser2024scaling} and Diffusion models~\cite{albergo2023stochastic, cai2025diffusion, song2021scorebased} is to reduce the number of inference steps while retaining the quality gains of iterative sampling~\cite{lipman2024flow,lai2025principlesdiffusionmodels,ma2025learningstraightflowsvariational,liu2023flow}. Consistency Models~\cite{song2023consistency, geng2025consistency, kim2024consistency, lu2025simplifying, song2024improved, frans2025one} impose consistency constraints along the generation trajectory, while Mean Velocity Models~\cite{geng2025mean, zhang2025alphaflow, hu2025cmt, geng2025improved, lu2026one, ma2026transition} optimize objectives derived from the relationship between instantaneous velocity and mean velocity to enable one- or few-step generation. However, these approaches can be limited by instability in self-distillation~\cite{sabour2025align,zheng2025large,zhang2026t3d,wang2024rectified, zhang2025towards, lee2024improving, geng2023onestep, salimans2022progressive} and may generalize poorly across different tasks and model scales~\cite{sheng2026mp1,wang2026one,li2025meanaudio}. In contrast, our proposed method introduces the optimization perspective of Drift Models~\cite{deng2026generative,he2026sinkhorn,lai2026unified}, inheriting their strong one-step generation capability while preserving the test-time scaling behavior~\cite{ma2025scaling,singhal2025general,li2025reflect} and connection to Flow Matching and Diffusion models.

\textbf{Drift Models.}
Drift Models~\cite{deng2026generative,he2026sinkhorn,lai2026unified} shift multi-step generation from inference time to training time, thereby naturally enabling one-step generation at inference time. To this end, they introduce a drift field that transports the model-generated distribution toward the target distribution~\cite{deng2026generative,he2026sinkhorn,lai2026unified}. Several recent works have developed more principled designs for drift fields~\cite{he2026sinkhorn} or established connections between Drift Models and Score-Based Models~\cite{lai2026unified}. In contrast, our proposed method aims to build a framework that unlocks the test-time scaling capability of drift-based methods. It preserves the one-step generation ability of Drift Models while inheriting the test-time scaling potential of Flow Matching and Diffusion Models~\cite{ma2025scaling,singhal2025general,li2025reflect}, which is particularly advantageous for controlled or guided generation~\cite{christopher2025constrained}. Our formulation is also orthogonal to the particular drift-field instantiation: it only requires a valid drift signal between the current model distribution and the target marginal, so it can be combined with stronger drift constructions such as Sinkhorn-based variants~\cite{he2026sinkhorn}.

\section{Preliminaries} \label{sec:prelim}

\subsection{Flow Matching}
\label{pre:sec:fm}

Flow Matching constructs a continuous probability path $\{p_t\}_{t\in[0,1]}$ that transports a source distribution $p_0$ to a target distribution $p_1$. Let $X_0\sim p_0$ denote a source sample, e.g., Gaussian noise, and let $X_1\sim p_1$ denote a target sample, e.g., an image. A coupling between these endpoints is specified by a joint density $\pi$ on $\mathbb{R}^d\times\mathbb{R}^d$ whose marginals are $p_0$ and $p_1$. In this work, we follow the standard Flow Matching setting and use the independent coupling: $\pi(x_0,x_1)=p_0(x_0)p_1(x_1)$.

The Flow Matching construction is often described through conditional probability paths~\cite{lipman2023flow}. Let $Z$ be a conditioning random variable that indexes such paths. Conditioned on $Z=z$, we obtain endpoint variables $(X_0^Z,X_1^Z)$ and a conditional path $\{X_t^Z\}_{t\in[0,1]}$ with conditional density $p_{t\mid Z}(\cdot\mid z)$. The corresponding marginal path $\{X_t\}_{t\in[0,1]}$ has density $p_t(\cdot)$ obtained by marginalizing over $Z$:
\begin{equation}
\label{eq:bayes_xt}
p_t(x_t)
=
\int p_{t\mid Z}(x_t\mid z)\,p_Z(z)\,dz,
\qquad
X_t\sim p_t,\quad
X_t\mid (Z=z)\sim p_{t\mid Z}(\cdot\mid z).
\end{equation}
Thus, conditional paths describe simpler transports indexed by $Z$, while the marginal path is their mixture.

We consider a general interpolant between the marginal endpoints $X_0$ and $X_1$, defined by scalar schedules $\alpha,\beta:[0,1]\to\mathbb{R}$:
\begin{equation}
\label{eq:generalInterpolant}
X_t
=
\alpha(t)X_0+\beta(t)X_1,
\qquad t\in[0,1].
\end{equation}
The schedules satisfy the boundary conditions \( \alpha(0)=1,\quad \beta(0)=0, \alpha(1)=0,\quad \beta(1)=1\),
so that $X_{t=0}=X_0$ and $X_{t=1}=X_1$. Notice also that when $\alpha$ and $\beta$ are differentiable, the path has instantaneous velocity \( \frac{d}{dt}X_t = \dot{\alpha}(t)X_0+\dot{\beta}(t)X_1. \)
The same construction applies conditionally: given $Z=z$, 
\( X_t^Z = \alpha(t)X_0^Z+\beta(t)X_1^Z\), and  \( \frac{d}{dt}X_t^Z = \dot{\alpha}(t)X_0^Z+\dot{\beta}(t)X_1^Z\).

\textbf{Flow Matching vector fields~\cite{lipman2023flow}.}
Let $v(x_t,t\mid z)\in\mathbb{R}^d$ denote a \emph{conditional} velocity field that transports the conditional density $p_{t\mid Z}(\cdot\mid z)$ along time.
The corresponding \emph{marginal} velocity field is defined by conditional expectation:
\begin{equation}
\label{eq:marginalV}
v(x_t,t)
\;=\;
\int v(x_t,t\mid z)\,p_{Z\mid t}(z\mid x_t)\,dz
\;=\;
\mathbb{E}\!\left[\,v(X_t,t\mid Z)\,\middle|\,X_t=x_t\,\right].
\end{equation}
A marginal trajectory, i.e., Flow Matching generation trajectory, follows the ODE:
\begin{equation}
\label{eq:ode_marginal}
\frac{d x_t}{dt} = v(x_t,t), \qquad t\in[0,1],
\end{equation}
and similarly a conditional trajectory follows $\frac{d x_t^z}{dt}=v(x_t,t\mid z)$.

For brevity, the main paper includes only the material needed to introduce Drift Flow Matching. A more complete presentation of Flow Matching is provided in Appendix~\S~\ref{proof:sec:fm}.

\subsection{Drift Method}
\label{pre:sec:drift}

\textbf{Drift Velocity Field.}
Drift methods define a distribution-level update that moves a current model distribution $q$ toward a fixed target distribution $p$ on $\mathbb{R}^d$~\cite{deng2026generative,he2026sinkhorn,lai2026unified}. The update is specified by a drift field
\(V_{q,p}:\mathbb{R}^d\to\mathbb{R}^d\), which decomposes into an attraction toward the target distribution and a self-correction term associated with the current model distribution:
\begin{equation}
V_{q,p}(x) = V_p^{+}(x) - V_q^{-}(x),
\label{eq:drift-velocity-field}
\end{equation}
where each term is a kernel-weighted (and normalized) average of displacement vectors~\cite{deng2026generative,he2026sinkhorn,lai2026unified}. Concretely, let $k : \mathbb{R}^d \times \mathbb{R}^d \to \mathbb{R}_{+}$ be a positive kernel, and define the normalizing factors
$
Z_p(x) := \int k(x,y)\,dp(y), \ \
Z_q(x) := \int k(x,y)\,dq(y)
$.
The Drift update is then given by
\begin{equation}
V_p^{+}(x) := \frac{1}{Z_p(x)} \int k(x,y)(y-x)\,dp(y),
\qquad
V_q^{-}(x) := \frac{1}{Z_q(x)} \int k(x,y)(y-x)\,dq(y).
\label{eq:drift-positive-negative-components}
\end{equation}
Intuitively, $V_p^{+}(x)$ attracts $x$ toward the target $p$, while $V_q^{-}(x)$ subtracts an analogous attraction toward the current model $q$, producing a repulsive (self-correction) effect~\cite{deng2026generative,he2026sinkhorn}.

In this work, we use a Gibbs kernel
\(k(x,y)=\exp\!\left(-\frac{C(x,y)}{\tau}\right),\)
where $C(x,y)$ is a cost function and $\tau>0$ is a temperature hyperparameter.\footnote{The kernel used in~\cite{deng2026generative} is $\exp(-\lVert x-y\rVert/\tau)$, whereas the kernel used in~\cite{he2026sinkhorn} is $\exp(-\frac{1}{2}\lVert x-y\rVert^2/\tau)$.}~\cite{deng2026generative,he2026sinkhorn,lai2026unified}
Unless otherwise stated, we use the squared Euclidean cost \(C(x,y)=\frac{1}{2}\lVert x-y\rVert^2.\)

The empirical form of the drift field is for minibatch training. Let \(p=\frac{1}{n}\sum_{j=1}^n \delta_{y_j}\), and \( q=\frac{1}{n}\sum_{i=1}^n \delta_{x_i}\) be empirical measures with samples $Y=\{y_j\}_{j=1}^n$ and $X=\{x_i\}_{i=1}^n$. Then, for each sample $x_i$ from the current model distribution, the discrete drift update is
\begin{equation}
\begin{aligned}
&V_{q,p}(x_i)
=
\sum_{j=1}^n P_{XY}^{\mathrm{drift}}[i,j]\,y_j
-
\sum_{j=1}^n P_{XX}^{\mathrm{drift}}[i,j]\,x_j, \\
\text{where} \quad
&P_{XY}^{\mathrm{drift}}[i,j]
:=
\frac{k(x_i,y_j)}{\sum_{l=1}^n k(x_i,y_l)},
\qquad
P_{XX}^{\mathrm{drift}}[i,j]
:=
\frac{k(x_i,x_j)}{\sum_{l=1}^n k(x_i,x_l)}.
\label{eq:discrete-drift-update}
\end{aligned}
\end{equation}
Thus, the empirical drift moves each generated sample toward a kernel-weighted average of target samples while subtracting a kernel-weighted average of generated samples.

\textbf{Drift generative model.}
Drift generative models define a one-step generative model~\cite{deng2026generative,he2026sinkhorn,lai2026unified}:
$
f_\theta : \mathbb{R}^d \to \mathbb{R}^d,
\epsilon \mapsto x_\theta := f_\theta(\epsilon)
$,
where $\epsilon \sim p_\epsilon$ for some prior distribution $p_\epsilon$. The drift flow is the probability path
$
q_\theta := (f_\theta)_{\#}p_\epsilon
$
generated by $V_{q_\theta,p}$~\cite{lai2026unified}:
$
\dot{x}_\theta = V_{q_\theta,p}(x_\theta),
\forall x_\theta = f_\theta(\epsilon)
$,
which is approximated by its forward Euler process in the time discretization scheme~\cite{lai2026unified}:
$x_{\theta_{k+1}} = x_{\theta_k} + V_{q_\theta,p}(x_{\theta_k})$.
The training loss is then set to be
\begin{equation}
\mathcal{L}^{\mathrm{drift}}
:=
\frac{1}{2}\mathbb{E}_{\epsilon}
\left[
\left\|
f_\theta(\epsilon)
-
\operatorname{sg}\!\left(
f_\theta(\epsilon)+V_{q_\theta,p_{\mathrm{data}}}(f_\theta(\epsilon))
\right)
\right\|^2
\right],
\label{eq:drift-training-loss}
\end{equation}
where $\operatorname{sg}$ is the stop-gradient operator, and the optimization scheme is gradient descent:
\begin{equation}
\theta \leftarrow \theta - \eta \nabla_\theta \mathcal{L}^{\mathrm{drift}},
\qquad
\nabla_\theta \mathcal{L}^{\mathrm{drift}}
=
-\mathbb{E}_{\epsilon}
\left[
J_f(\theta,\epsilon)^\top
V_{q_\theta,p_{\mathrm{data}}}(f_\theta(\epsilon))
\right],
\label{eq:drift-gradient-update}
\end{equation}
where $\eta$ is the learning rate, and $J_{f}(\theta,\epsilon) \in \mathbb{R}^{d \times \dim(\theta)}$ denotes the Jacobian. It can be verified that the above parameter update~\eqref{eq:drift-gradient-update} induces the drift flow $\dot{x}_\theta = V_{q_\theta,p}(x_\theta)$~\cite{he2026sinkhorn}.

\begin{remark}[Identifiability of the drift objective]
\label{remark:drift_identifiability}
The drifting objective Eq.~\eqref{eq:drift-training-loss} minimizes the magnitude of the drift field, i.e., $\|V_{q,p}\|^2$.
By construction, the equilibrium condition $q=p$ implies $V_{q,p}(x)=0$ for all $x$~\cite{deng2026generative,he2026sinkhorn,lai2026unified}.
However, the converse implication does not hold in general and depends on the choice of kernel and the distribution family.
Under suitable non-degeneracy conditions (e.g., sufficiently expressive kernels or feature representations), the condition $V_{q,p} \approx 0$ can serve as a useful surrogate for distribution matching, in the sense that it encourages $q \approx p$. This heuristic is supported both empirically and by recent theoretical analyses of drifting-based methods~\cite{deng2026generative,he2026sinkhorn,lai2026unified}.
We note that stronger drift constructions can improve identifiability by making the zero-drift condition more directly correspond to distribution matching~\cite{he2026sinkhorn,lai2026unified}. Our formulation is compatible with such improvements.
\end{remark}

\section{Drift Flow Matching}
\label{sec:dfm}

Building on the conditional-path construction of Flow Matching and the distribution-level drift update introduced in Sec.~\ref{sec:prelim}, this section introduces \textbf{Drift Flow Matching (DFM)}, a two-time transport model trained with a drift-based objective on marginal pairs $(p_t, p_r)$. This formulation enables direct generation from any state $x_t$ to any state $x_r$ with an arbitrary step size $(r - t)$, generalizing Flow Matching through a drifting mechanism.


\subsection{Two-Time Marginal Construction}
Let
$
\Delta := \{(t,r)\in[0,1]^2 : 0 \le t \le r \le 1\}
$
and let $(t,r)\sim \rho$ be a sampling distribution on $\Delta$ (e.g., uniform on $\Delta$ or logit-normal (lognorm) on $\Delta$).
We adopt the same coupling $\pi(x_0,x_1)$ and conditional-path construction as in Flow Matching~\cite{lipman2023flow,lipman2024flowmatchingguidecode}, shown in Sec~\ref{pre:sec:fm}.
Let the conditional variable $Z \sim \pi(x_0,x_1)$ with independent coupling $\pi(x_0,x_1)=p_0(x_0)p_1(x_1)$, therefore:
$
Z=(X_0,X_1),
X_0\sim p_0,X_1\sim p_1
$. 
Note that other coupling strategies in Flow Matching are compatible with the proposed method~\cite{albergo2024stochastic,silvestri2025vct,pooladian2023multisample,zhang2025hierarchical}.
To maintain the same general interpolant form as in Flow Matching, given in Eq.~\eqref{eq:generalInterpolant}, we define the conditional paths as:
\begin{equation}
X_t^Z = \alpha(t)X_0^Z + \beta(t)X_1^Z,
\qquad
X_r^Z = \alpha(r)X_0^Z + \beta(r)X_1^Z.
\label{eq:dfm_conditional_pair}
\end{equation}
Conditioned on $Z=z$, we obtain conditional coupling $(X_0^Z, X_1^Z)$ and conditional states $\{X_t^Z\}_{t\in[0,r)}$ and $\{X_r^Z\}_{r\in(t,1]}$ within conditional distributions $p_{t\mid Z}(\cdot\mid z)$ and $p_{r\mid Z}(\cdot\mid z)$ respectively.
Marginalizing over $Z$ yields the marginal distributions $p_t(x_t)$ and $p_r(x_r)$:
\begin{equation}
p_t(x_t)
=
\int p_{t\mid Z}(x_t\mid z)\,p_Z(z)\,dz,
\qquad
p_r(x_r)
=
\int p_{r\mid Z}(x_r\mid z)\,p_Z(z)\,dz.
\label{eq:dfm_two_marginals}
\end{equation}
Therefore, by sampling multiple realizations of $Z$ and evaluating Eq.~\eqref{eq:dfm_conditional_pair}, we obtain direct samples $\{X_t\}_{t\in[0,r)}$ and $\{X_r\}_{r\in(t,1]}$ from the marginal distributions $p_t$ and $p_r$. These together constitute the marginal distribution pair $(p_t, p_r)$ required by the proposed \textbf{DFM}.
It is worth noting that sampling multiple realizations of $Z$ is equivalent to standard mini-batch sampling used in model optimization frameworks (e.g., SGD~\cite{bottou2012stochastic}), including but not limited to Flow/Diffusion models~\cite{lipman2023flow,lipman2024flowmatchingguidecode,lai2025principlesdiffusionmodels,karras2022elucidating}, and Drifting Models~\cite{deng2026generative,he2026sinkhorn,lai2026unified}.
The difference between Flow Matching~\cite{lipman2023flow,lipman2024flowmatchingguidecode} and \textbf{DFM} is therefore \emph{not} the interpolant itself, but the learning principle: Flow Matching regresses a pointwise conditional target, as shown in Appendix~\S~\ref{proof:sec:fm}, whereas \textbf{DFM} directly learns the marginal distribution transport between $p_t$ and $p_r$.

\begin{proposition}[Wasserstein control of DFM marginal pairs]
\label{prop:dfm_main_w2_control}
For any \((t,r)\in\Delta\), the conditional construction in Eq.~\eqref{eq:dfm_conditional_pair} induces the marginal distribution pair $(p_t,p_r)$. Consequently,
\[
W_2^2(p_t,p_r)
\le
\mathbb{E}\!\left[
\|(\alpha(r)-\alpha(t))X_0+(\beta(r)-\beta(t))X_1\|^2
\right].
\]
In the linear-interpolant case, namely \(\alpha(t)=1-t\) and \(\beta(t)=t\), this gives
\[
W_2(p_t,p_r)
\le
|r-t|
\Big(\mathbb{E}\|X_1-X_0\|^2\Big)^{1/2}.
\]
Moreover, for any grid \(0=t_0<\cdots<t_N=1\), the linear-interpolant case implies
\[
\sum_{m=0}^{N-1}
\frac{W_2^2(p_{t_m},p_{t_{m+1}})}
{t_{m+1}-t_m}
\le
\mathbb{E}\|X_1-X_0\|^2.
\]
\end{proposition}
This result shows that DFM decomposes endpoint generation into short-range marginal transport problems along a \(W_2\)-controlled path.
See Appendix~\S~\ref{app:dfm_ot} for details and proof.

\subsection{Mean-Velocity Transport Parameterization}
\textbf{DFM} parameterizes a \emph{mean velocity} field~\cite{geng2025mean, zhang2025alphaflow, hu2025cmt, geng2025improved, lu2026one, ma2026transition} to transport the current distribution $p_t$ at time step $t$ to future distribution $p_r$ at time step $r$.
\begin{equation}
x_r = x_t + \int_t^r v(x_\tau,\tau)\,d\tau = x_t + (r-t) u(x_t,t,r),
\qquad
u(x_t,t,r)
:=
\frac{1}{r-t}\int_t^r v(x_\tau,\tau)\,d\tau,
\label{eq:dfm_meanV}
\end{equation}
where \( v(x_\tau,\tau) \) denotes the instantaneous velocity solving the ODE $\frac{d x_\tau}{d\tau} = v(x_\tau,\tau), \tau \in [t,r]$, which transports the distribution from \( p_t \) to \( p_r \).
For example, \( v(x_\tau,\tau) \) can be the marginal velocity in Flow Matching at time step \( \tau \), as shown in Eq.~\eqref{eq:marginalV}, or the drift velocity (analogous to Eq.~\eqref{eq:drift-velocity-field}), as discussed in the next paragraph.
\textbf{DFM} parameterizes the \emph{mean velocity} field $u$ by parameter $\theta$:
$
u^\theta(x_t,t,r) : \mathbb{R}^d \times \Delta \to \mathbb{R}^d
$,
and the associated transport map is therefore:
\begin{equation}
T_{t,r}^\theta(x_t)
:=
x_t + (r-t)\,u^\theta(x_t,t,r).
\label{eq:dfm_transport_map}
\end{equation}
When $X_t\sim p_t$, the model induces the predicted state:
$
\widehat X_r^\theta := T_{t,r}^\theta(X_t), \ \
q_{t,r}^\theta := (T_{t,r}^\theta)_{\#}p_t,
$
where $q_{t,r}^\theta$ is the model distribution at time $r$ obtained by transporting $p_t$ through $T_{t,r}^\theta$.


Thus, since \textbf{DFM} adopts the same probability path construction Eq.~\eqref{eq:bayes_xt} as Flow Matching, the marginal pair $(p_t,p_r)$ lies on the same Flow Matching probability path~\cite{lipman2023flow}. 
In particular, if the learned map \(T_{t,r}^\theta\) Eq.~\eqref{eq:dfm_transport_map} is first-order consistent with the canonical local flow map $\Phi_{t,r}(x_t) = x_t+\int_t^r v(x_\tau,\tau)\,d\tau$ induced by the marginal Flow Matching ODE Eq.~\eqref{eq:ode_marginal}, then
\begin{equation}
\lim_{r\to t}u^\theta(x_t,t,r)=v(x_t,t),
\label{eq:dfm_infinitesimal_limit}
\end{equation}
where \(v(x_t,t)\) is the marginal Flow Matching velocity in Eq.~\eqref{eq:marginalV}. Hence, \textbf{DFM} recovers standard Flow Matching in the infinitesimal-step limit; see Appendix~\S~\ref{app:dfm_infinitesimal_limit} for the details.

\subsection{Drift-Based Supervision and Training}
Unlike Flow Matching~\cite{lipman2023flow,lipman2024flowmatchingguidecode}, which relies on the conditional--marginal gradient equivalence in Theorem~\ref{proof:thm:MequivC_FM}, \textbf{DFM} is trained directly on the marginal distribution transport problem from $p_t$ to $p_r$:
$
q_{t,r}^\theta \rightarrow p_r
$.
For a fixed pair $(t,r)$, we define the Drift field $V_{q_{t,r}^\theta,p_r} : \mathbb{R}^d \to \mathbb{R}^d$ on the target-time space by adopting Eq.\eqref{eq:drift-velocity-field}:
\begin{equation}
V_{q_{t,r}^\theta,p_r}(x)
=
V_{p_r}^{+}(x)-V_{q_{t,r}^\theta}^{-}(x),
\label{eq:dfm_drift_field}
\end{equation}
where, using the same positive kernel construction $k : \mathbb{R}^d \times \mathbb{R}^d \to \mathbb{R}_{+}$ as in Eq.~\eqref{eq:drift-positive-negative-components},
\begin{equation}
V_{p_r}^{+}(x)
:=
\frac{1}{Z_{p_r}(x)}
\int k(x,y)(y-x)\,dp_r(y),
\
V_{q_{t,r}^\theta}^{-}(x)
:=
\frac{1}{Z_{q_{t,r}^\theta}(x)}
\int k(x,y)(y-x)\,dq_{t,r}^\theta(y),
\label{eq:dfm_drift_components}
\end{equation}
with the normalizers
$
Z_{p_r}(x):=\int k(x,y)\,dp_r(y), \ \
Z_{q_{t,r}^\theta}(x):=\int k(x,y)\,dq_{t,r}^\theta(y)
$.
The positive term attracts the predicted samples toward the true marginal $p_r$, while the negative term subtracts the analogous attraction toward the current model distribution $q_{t,r}^\theta$.
Therefore, $V_{q_{t,r}^\theta,p_r}$ pushes the generated distribution $q_{t,r}^\theta$ toward the target distribution $p_r$~\cite{deng2026generative,he2026sinkhorn}.

Using the stop-gradient formulation of the Drift method Eq.\eqref{eq:drift-training-loss}, where $\operatorname{sg}$ is the stop-gradient operator, we define the \textbf{DFM} objective as
\begin{equation}
\mathcal{L}_{\mathrm{DFM}}(\theta)
:=
\frac{1}{2}
\mathbb{E}_{(t,r)\sim\rho,\;X_t\sim p_t}
\left[
\left\|
\widehat X_r^\theta
-
\operatorname{sg}\!\left(
\widehat X_r^\theta
+
V_{q_{t,r}^\theta,p_r}(\widehat X_r^\theta)
\right)
\right\|^2
\right],
\qquad
\widehat X_r^\theta=T_{t,r}^\theta(X_t).
\label{eq:dfm_loss}
\end{equation}
Compared with the one-step Drift loss in Eq.~\eqref{eq:drift-training-loss}, whose expectation is taken over the noise variable $\epsilon$, the \textbf{DFM} objective is indexed by the time pair $(t,r)\sim\rho$ and by the marginal states $X_t\sim p_t$. 
Equivalently, the \textbf{DFM} training problem is defined on the marginal distribution pair $(p_t,p_r)$, rather than only on $(\epsilon, p_{\mathrm{data}})$, i.e., the \textbf{DFM} objective involves $(t,r,X_t,X_r)$ while the Drift loss in Eq.~\eqref{eq:drift-training-loss} only involves $(X_0,X_1)$, where $X_0\sim\epsilon$ and $X_1 \sim p_{\mathrm{data}}$.

The optimization scheme of \textbf{DFM} is based on gradient descent~\cite{bottou2012stochastic}: $\theta \leftarrow \theta - \eta \nabla_\theta \mathcal{L}_{\mathrm{DFM}}(\theta)$, where $\eta$ is the learning rate, and The corresponding gradient of Eq.\eqref{eq:dfm_loss} is
\begin{equation}
\nabla_\theta \mathcal{L}_{\mathrm{DFM}}(\theta)
=
-
\mathbb{E}_{(t,r)\sim\rho,\;X_t\sim p_t}
\left[
J_{T_{t,r}}(\theta,X_t)^\top
V_{q_{t,r}^\theta,p_r}\!\left(T_{t,r}^\theta(X_t)\right)
\right],
\label{eq:dfm_gradient}
\end{equation}
where $J_{T_{t,r}}(\theta,X_t)$ denotes the Jacobian of $T_{t,r}^\theta(X_t)$ with respect to $\theta$.
Hence, the parameter update induced by Eq.~\eqref{eq:dfm_loss} moves the predicted distribution $q_{t,r}^\theta$ toward the true marginal $p_r$ for each sampled pair $(t,r)$.
We refer to Appendix~\S~\ref{app:dfm_gradient_descent} for details and proof.

A crucial practical point is that each time pair $(t,r)$ induces a distinct marginal transport problem between $(p_t,p_r)$, because both marginals depend on the chosen pair $(t,r)$. Figure~\ref{fig:dfm_group_drift} illustrates this effect: even with the same endpoint distributions, different time-pairs $(t,r)$ lead to different target-time marginals $p_r$, and consequently to different model-output distributions $q_{t,r}^{\theta}$, since the same model parameters $\theta$ are evaluated on different source-time marginal inputs $p_t$. Therefore, samples associated with different pairs $(t,r)$ cannot be mixed when estimating the Drift field Eq.~\eqref{eq:dfm_drift_field}. Instead, both the Drift field Eq.~\eqref{eq:dfm_drift_field} and its corresponding loss Eq.~\eqref{eq:dfm_loss} must be evaluated \emph{group-wise}, within each time-pair bin.
\begin{figure*}[htbp]
  \centering
  \setlength{\tabcolsep}{0pt}
  \renewcommand{\arraystretch}{0.95}
  \vspace{-1em}

  \begin{tabular}{ccccc}
    \includegraphics[width=0.198\linewidth]{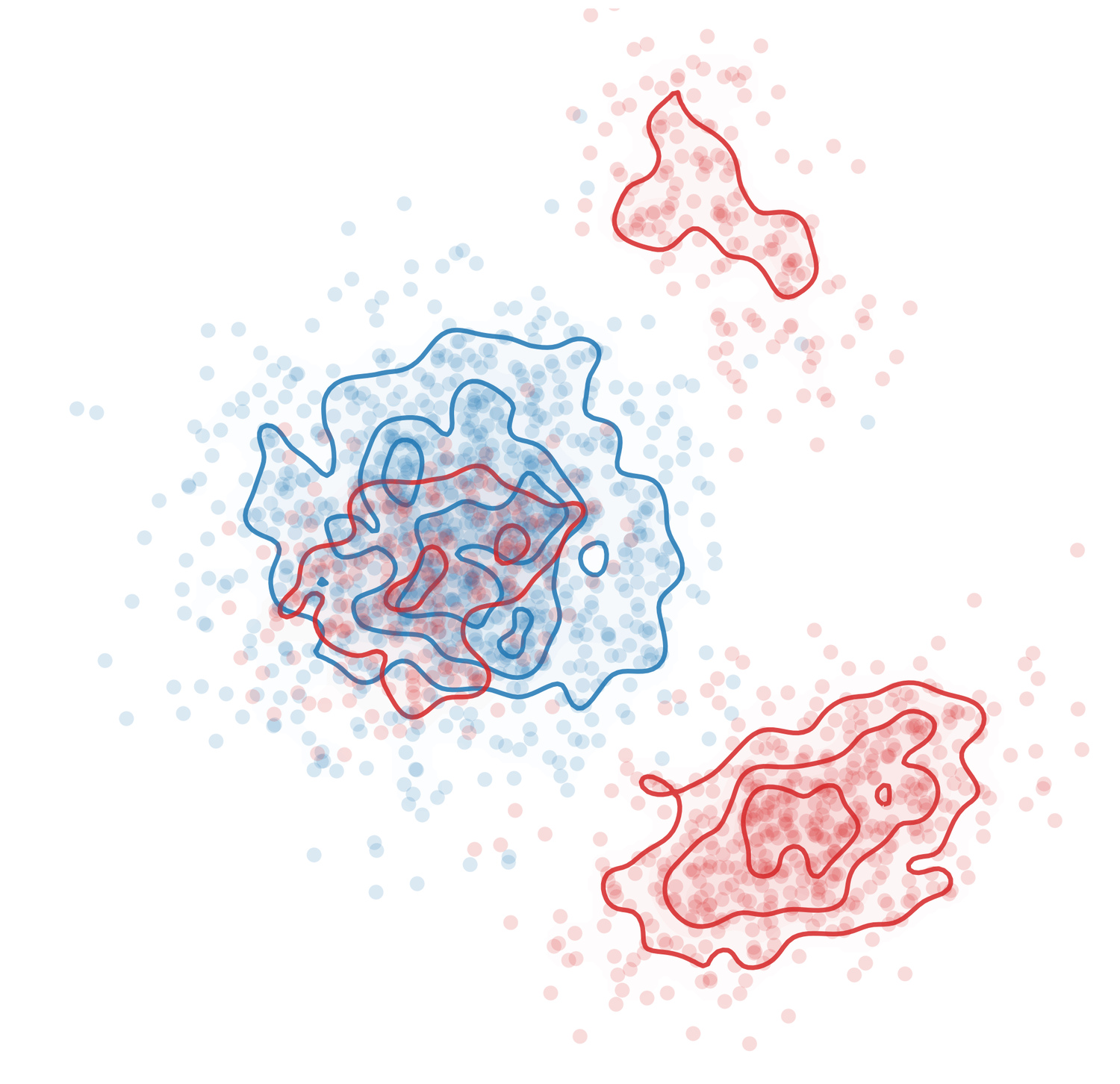} &
    \includegraphics[width=0.198\linewidth]{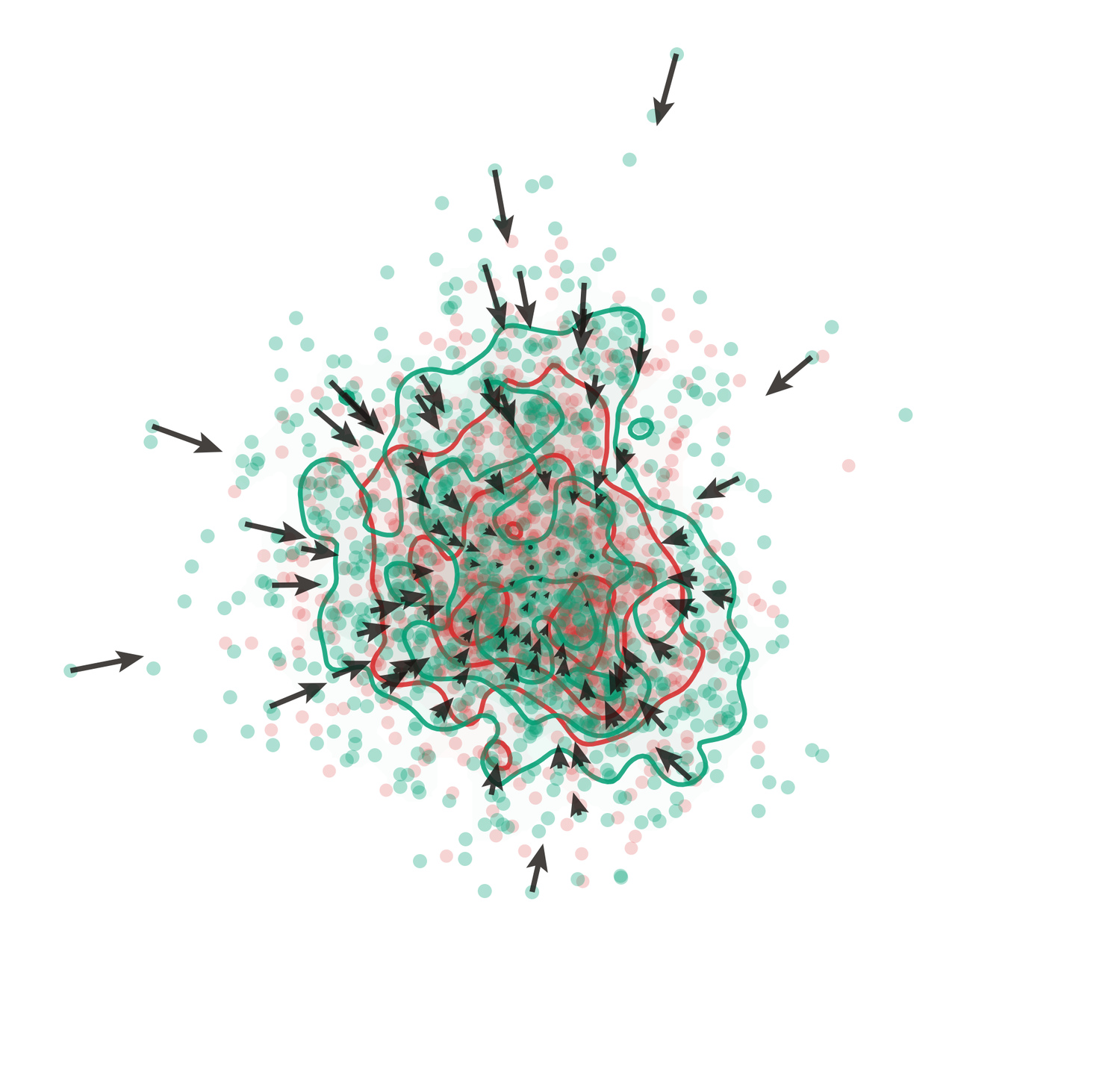} &
    \includegraphics[width=0.198\linewidth]{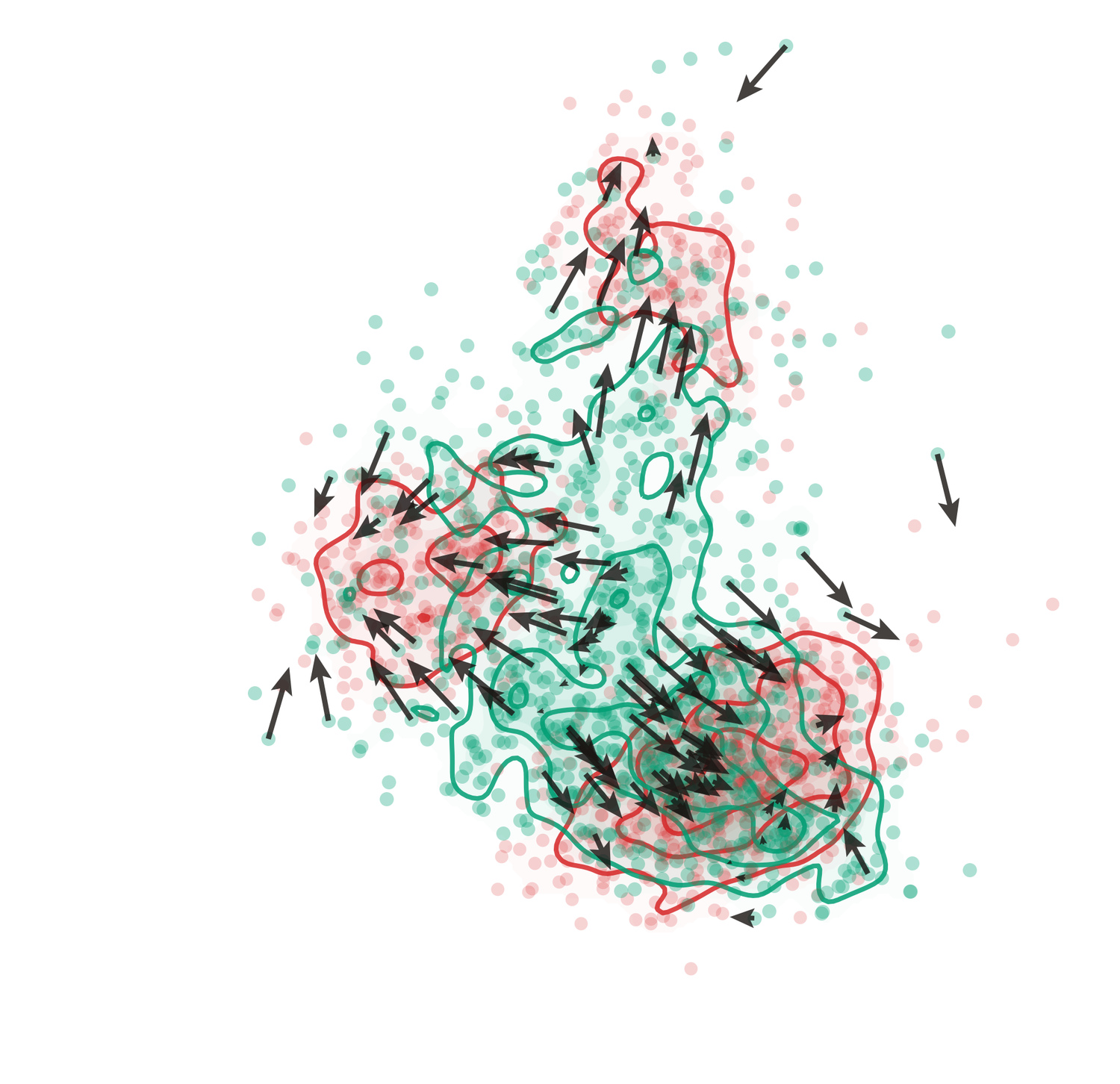} &
    \includegraphics[width=0.198\linewidth]{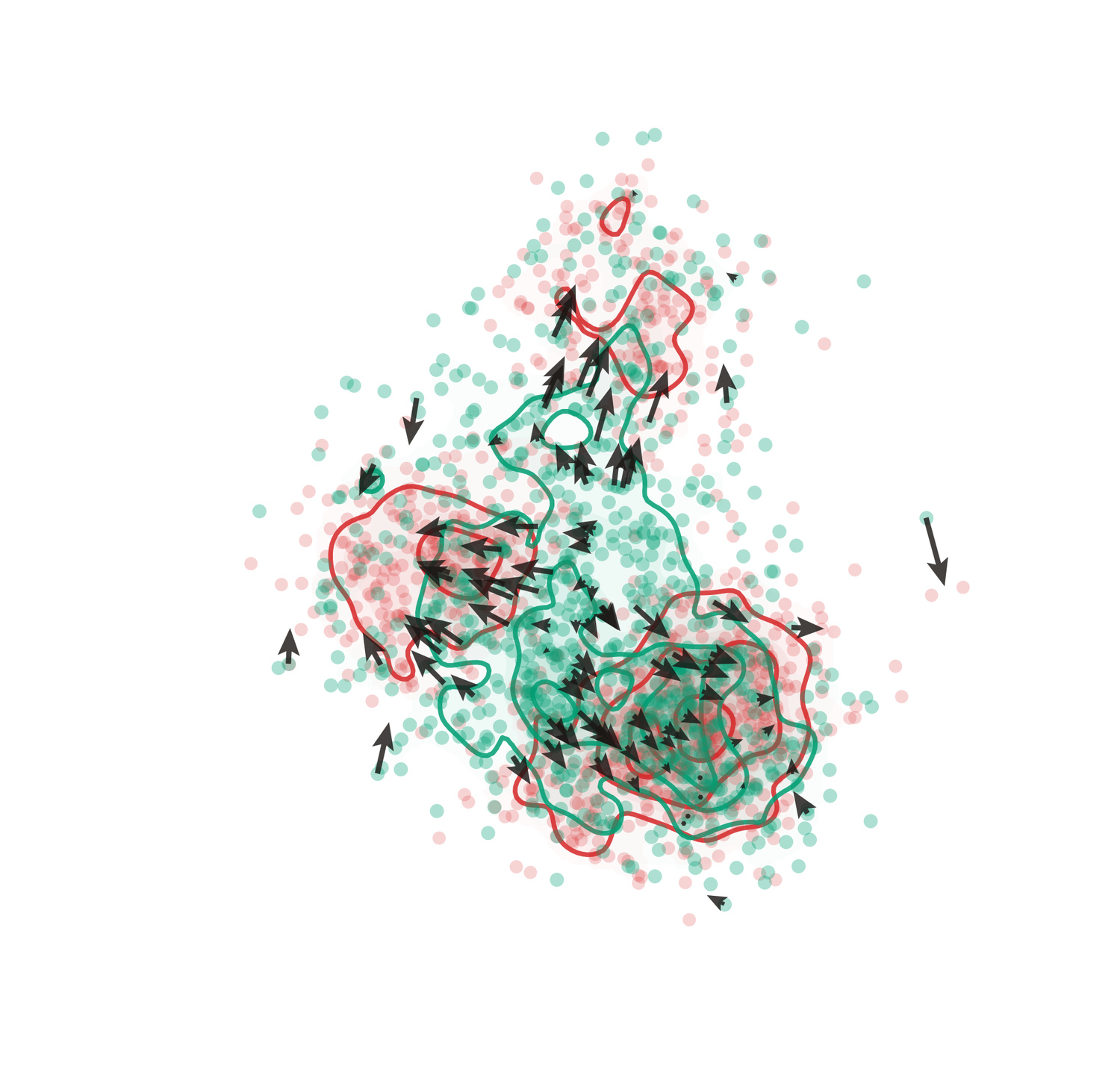} &
    \includegraphics[width=0.198\linewidth]{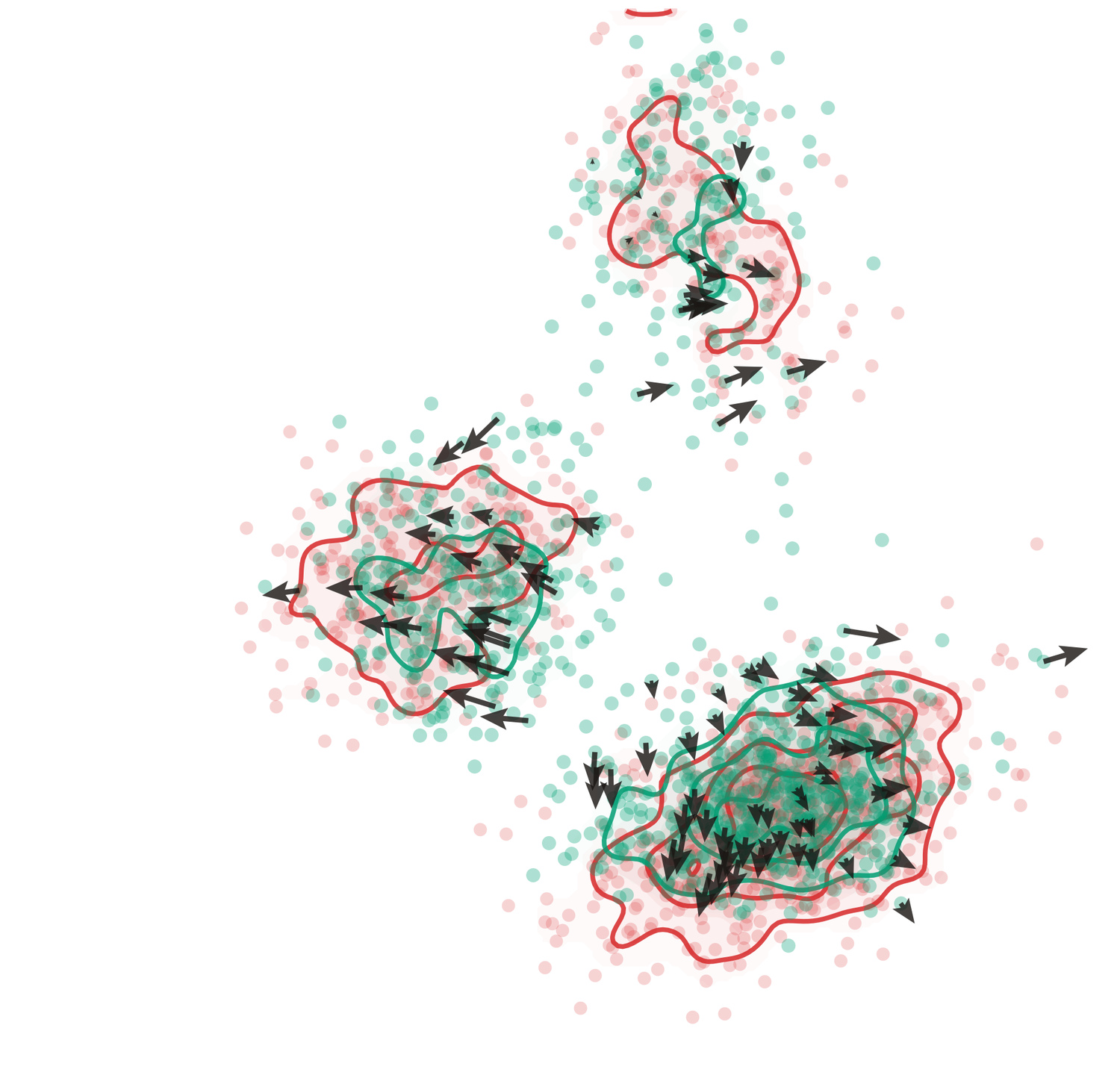} \\[-0.2em]
    {\scriptsize \makecell{Endpoint\\$p_0 \rightarrow p_1$}} &
    {\scriptsize \makecell{Group 1\\$(t,r)=(0.05,0.30)$}} &
    {\scriptsize \makecell{Group 2\\$(t,r)=(0.20,0.80)$}} &
    {\scriptsize \makecell{Group 3\\$(t,r)=(0.35,0.65)$}} &
    {\scriptsize \makecell{Group 4\\$(t,r)=(0.70,0.95)$}}
  \end{tabular}

  \vspace{-0.5em}
  \caption{\textbf{Grouped drift fields for different time-pairs.}
  The first panel shows the endpoint distributions used to construct the marginal path, with source $p_0$ in blue and target $p_1$ in red. The remaining panels visualize four different time-pair groups. In each group, red denotes the target-time marginal $p_r^{(g)}$, green denotes the model output distribution $q_{t,r}^{\theta,(g)}$ (under same model parameters $\theta$), and dark arrows denote the grouped drift field $V_{q_{t,r}^\theta,p_r}^{(g)}$.}
  \label{fig:dfm_group_drift}
  \vspace{-0.5em}
\end{figure*}

Concretely, suppose a minibatch uses $G$ sampled time pairs
$\{(t_g,r_g)\}_{g=1}^G \subset \Delta$.
For group $g$, which samples $n_g$ realizations of the conditional variable $Z=(X_0,X_1)$ shown in Eq.~\eqref{eq:dfm_conditional_pair}, let
$
\{(x_{0,i}^{(g)},x_{1,i}^{(g)})\}_{i=1}^{n_g}
\sim \pi(x_0,x_1)=p_0(x_0)p_1(x_1)
$
be endpoint samples assigned to that time pair $(t_g,r_g)$. We construct samples of time pair $(t_g,r_g)$ by evaluating conditional paths Eq.~\eqref{eq:dfm_conditional_pair}:
$
x_{t,i}^{(g)}
=
\alpha(t_g)x_{0,i}^{(g)}+\beta(t_g)x_{1,i}^{(g)}, \
x_{r,i}^{(g)}
=
\alpha(r_g)x_{0,i}^{(g)}+\beta(r_g)x_{1,i}^{(g)}
$.
The prediction of \textbf{DFM} in group $g$ parametrized by $\theta$ is then produced by Eq.\eqref{eq:dfm_transport_map}:
$
\widehat x_{r,i}^{(g)}
=
x_{t,i}^{(g)}
+
(r_g-t_g)\,u^\theta(x_{t,i}^{(g)},t_g,r_g)
$.
Let samples $X_t^{(g)}=\{x_{t,i}^{(g)}\}_{i=1}^{n_g}$, $X_r^{(g)}=\{x_{r,i}^{(g)}\}_{i=1}^{n_g}$, and $\widehat X_r^{(g)}=\{\widehat x_{r,i}^{(g)}\}_{i=1}^{n_g}$, which therefore define the empirical measures
$
p_t^{(g)}
=
\frac{1}{n_g}\sum_{i=1}^{n_g}\delta_{x_{t,i}^{(g)}}
$,
$
p_r^{(g)}
=
\frac{1}{n_g}\sum_{i=1}^{n_g}\delta_{x_{r,i}^{(g)}}
$,
$
q_{t,r}^{\theta,(g)}
=
\frac{1}{n_g}\sum_{i=1}^{n_g}\delta_{\widehat x_{r,i}^{(g)}}
$.
The discrete Drift field in group $g$, $V_{q_{t,r}^\theta,p_r}^{(g)}$, is then computed only within group $g$:
\begin{equation}
\begin{aligned}
&V_{q_{t,r}^\theta,p_r}^{(g)}(\widehat x_{r,i}^{(g)})
=
\sum_{j=1}^{n_g}
P_{\widehat X_r X_r}^{(g)}[i,j]\,
x_{r,j}^{(g)}
-
\sum_{j=1}^{n_g}
P_{\widehat X_r \widehat X_r}^{(g)}[i,j]\,
\widehat x_{r,j}^{(g)}, \\
\text{where} \quad
&P_{\widehat X_r X_r}^{(g)}[i,j]
:=
\frac{
k(\widehat x_{r,i}^{(g)},x_{r,j}^{(g)})
}{
\sum_{\ell=1}^{n_g}
k(\widehat x_{r,i}^{(g)},x_{r,\ell}^{(g)})
},
\qquad
P_{\widehat X_r \widehat X_r}^{(g)}[i,j]
:=
\frac{
k(\widehat x_{r,i}^{(g)},\widehat x_{r,j}^{(g)})
}{
\sum_{\ell=1}^{n_g}
k(\widehat x_{r,i}^{(g)},\widehat x_{r,\ell}^{(g)})
}.
\end{aligned}
\label{eq:dfm_group_drift}
\end{equation}
The group-wise loss of group $g$, which corresponds to time pair $(t_g,r_g)$, is
\begin{equation}
\mathcal{L}_{\mathrm{DFM}}^{(g)}(\theta)
=
\frac{1}{2n_g}
\sum_{i=1}^{n_g}
\left\|
\widehat x_{r,i}^{(g)}
-
\operatorname{sg}\!\left(
\widehat x_{r,i}^{(g)}
+
V_{q_{t,r}^\theta,p_r}^{(g)}(\widehat x_{r,i}^{(g)})
\right)
\right\|^2,
\label{eq:dfm_group_loss}
\end{equation}
and the minibatch objective, which is equivalent to and derived from Eq.\eqref{eq:dfm_loss}, is therefore
\begin{equation}
\mathcal{L}_{\mathrm{DFM}}(\theta)
=
\mathbb{E}_{(t,r)\sim\rho}
\big[
\mathcal{L}_{\mathrm{DFM}}^{(t,r)}(\theta)
\big]
=
\mathbb{E}_{(t,r)\sim\rho}
\big[
\mathcal{L}_{\mathrm{DFM}}^{(g)}(\theta)
\big]
=
\frac{1}{G}\sum_{g=1}^G \mathcal{L}_{\mathrm{DFM}}^{(g)}(\theta).
\label{eq:dfm_batch_loss}
\end{equation}
Notably, the loss in Eq.~\eqref{eq:dfm_group_loss} is defined over the empirical marginal distribution pair $(p_t^{(g)}, p_r^{(g)})$ within each group $g$. It does \emph{not} enforce a pointwise regression of $\widehat x_{r,i}^{(g)}$ toward its paired sample $x_{r,i}^{(g)}$. Instead, this formulation enables \textbf{DFM} to directly learn transport between the marginal distributions $(p_t^{(g)}, p_r^{(g)})$.
The overall objective in Eq.~\eqref{eq:dfm_batch_loss} is obtained by taking the expectation of the group-wise loss in Eq.~\eqref{eq:dfm_group_loss} over the set of time pairs $\{(t_g,r_g)\}_{g=1}^G$. Consequently, Eq.~\eqref{eq:dfm_batch_loss} serves as the realization of the theoretical objective in Eq.~\eqref{eq:dfm_loss}.

\subsection{Inference and Extensions}
\vspace{-0.5em}

At inference time, starting from $x_{t_0}\sim p_0$ and using a time grid
$
0=t_0<t_1<\cdots<t_N=1
$,
we sample by iterating the learned transport:
$
x_{t_{m+1}}
=
x_{t_m}
+
(t_{m+1}-t_m)\,
u^\theta(x_{t_m},t_m,t_{m+1}),
\
m=0,\dots,N-1
$.
For small step sizes (i.e $\lim_{r\to t}$ in Eq.\eqref{eq:dfm_infinitesimal_limit}), it approximates the marginal Flow Matching ODE in Eq.~\eqref{eq:ode_marginal}, see Appendix~\S~\ref{app:dfm_infinitesimal_limit} for the details; for larger step sizes, it exploits the transport learned by \textbf{DFM} to transport from $p_t$ to $p_r$.
The same transport parameterization also supports conditional generation; we defer the formulation and examples to Appendix~\S~\ref{app:conditionalG}.
Training and inference pseudocode are summarized in Appendix~\S~\ref{app:dfm_pseudocode}.

\begin{figure*}[htbp]
  \centering
  \setlength{\tabcolsep}{0pt}
  \renewcommand{\arraystretch}{0.95}
\vspace{-1em}
  \begin{tabular}{ccccccc}
    \includegraphics[width=0.142\linewidth]{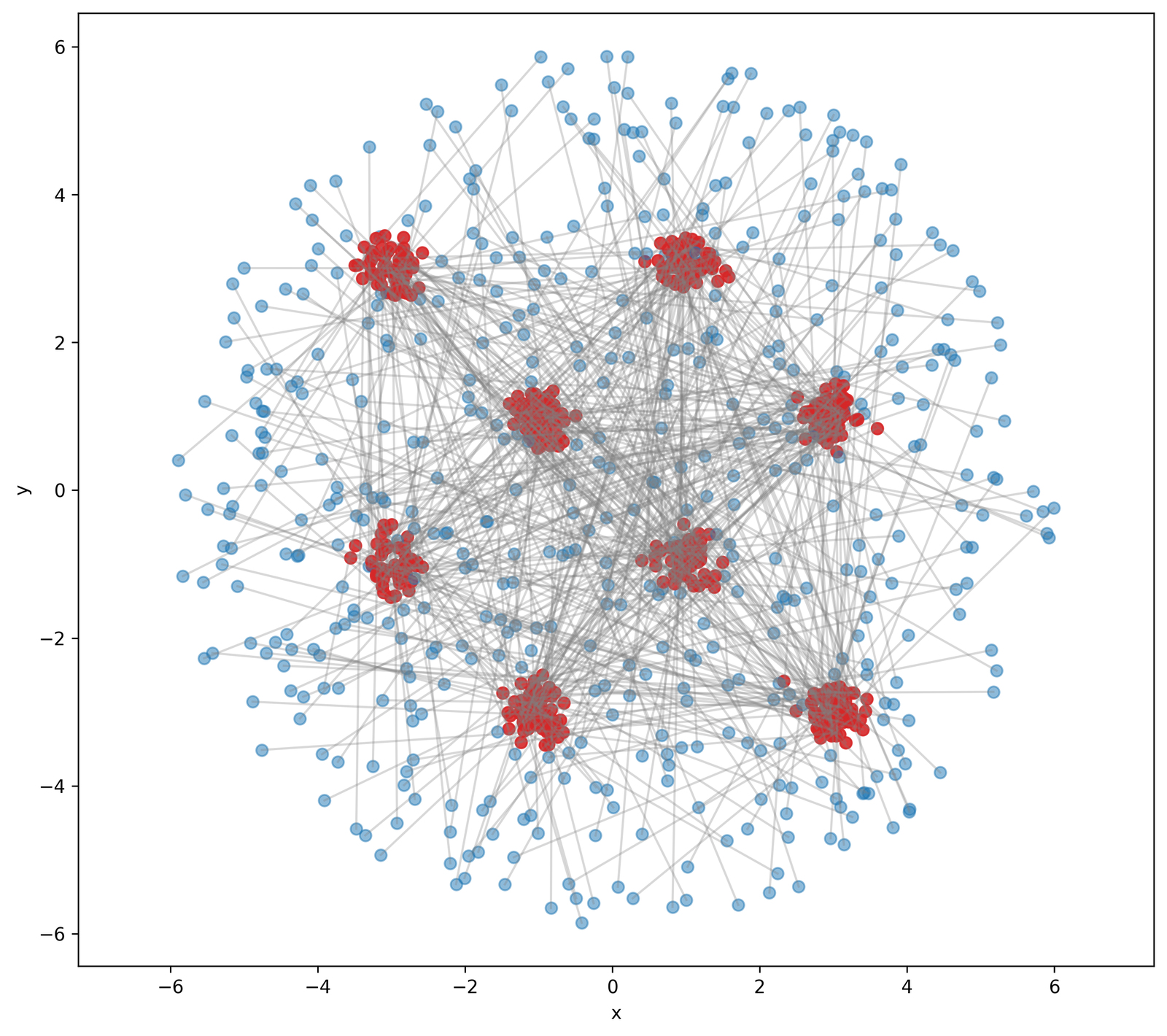} &
    \includegraphics[width=0.142\linewidth]{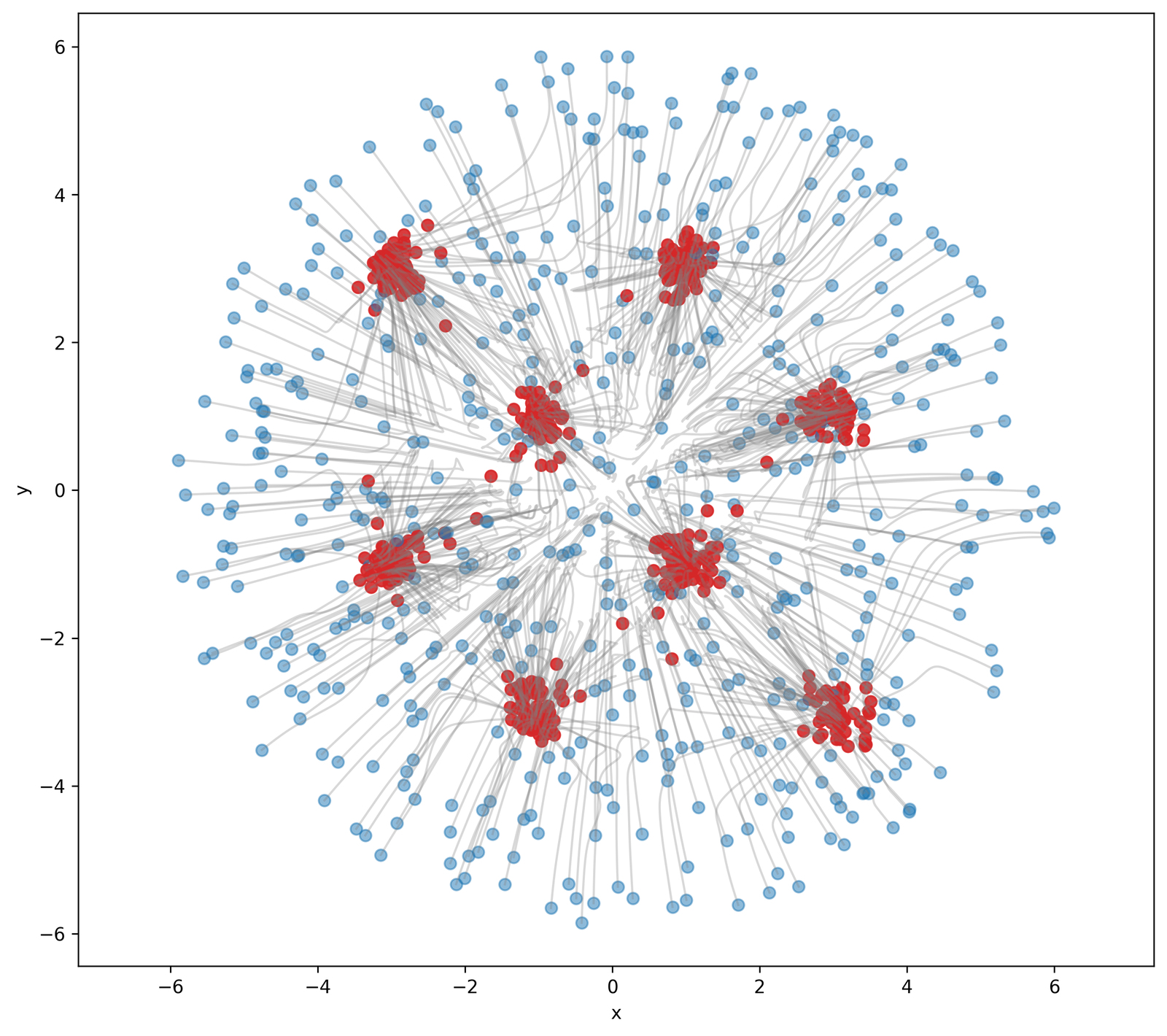} &
    \includegraphics[width=0.142\linewidth]{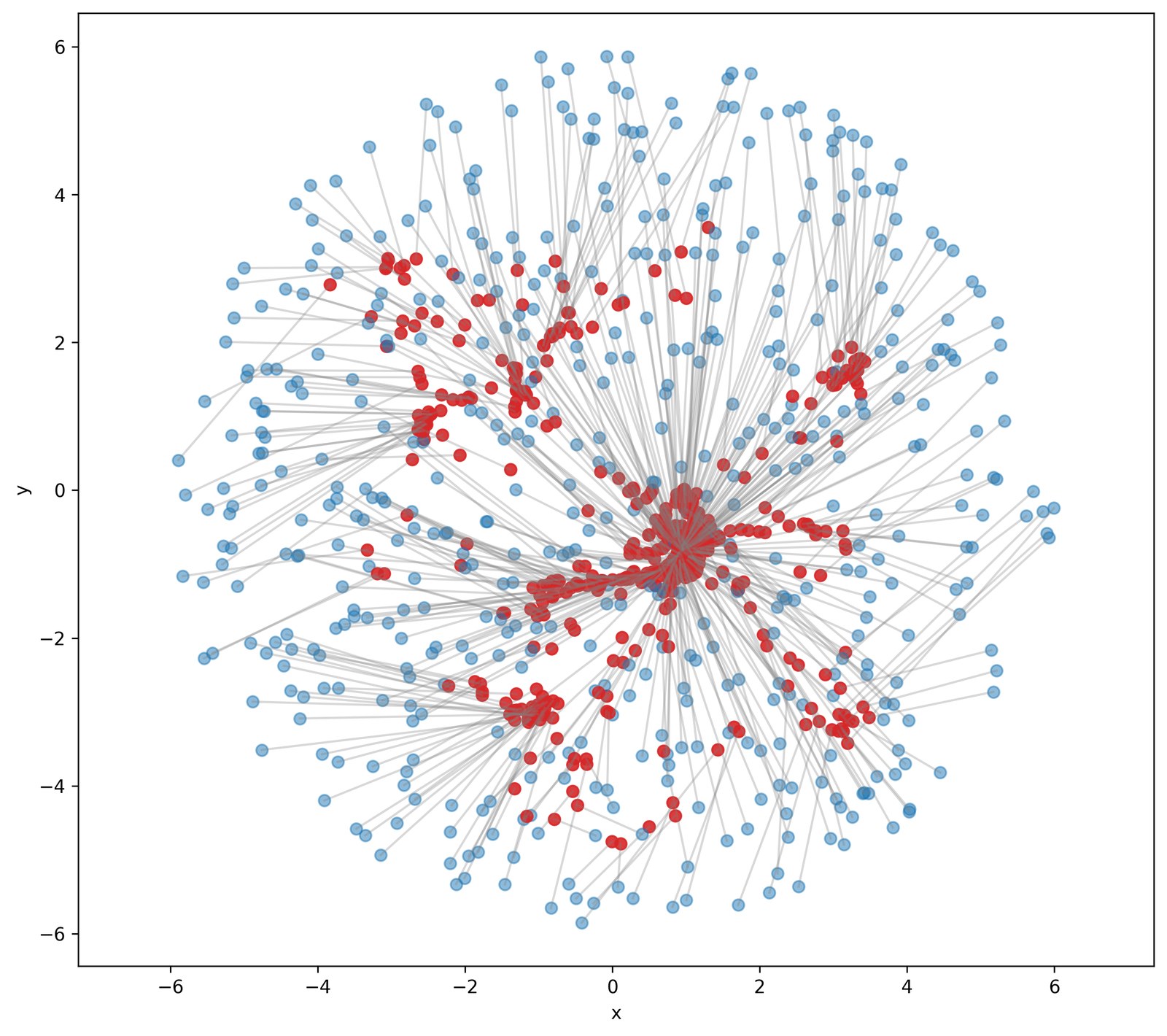} &
    \includegraphics[width=0.142\linewidth]{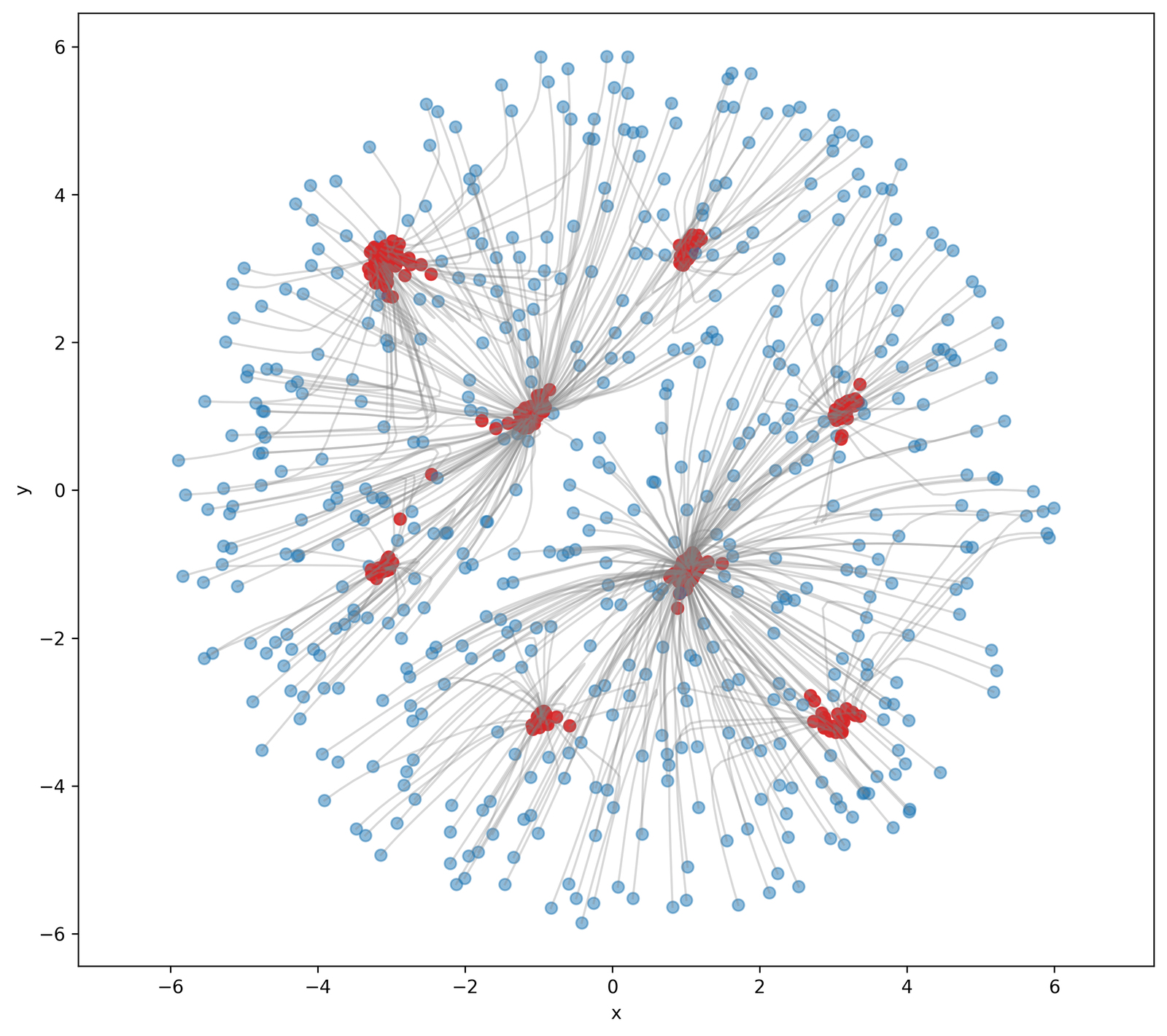} &
    \includegraphics[width=0.142\linewidth]{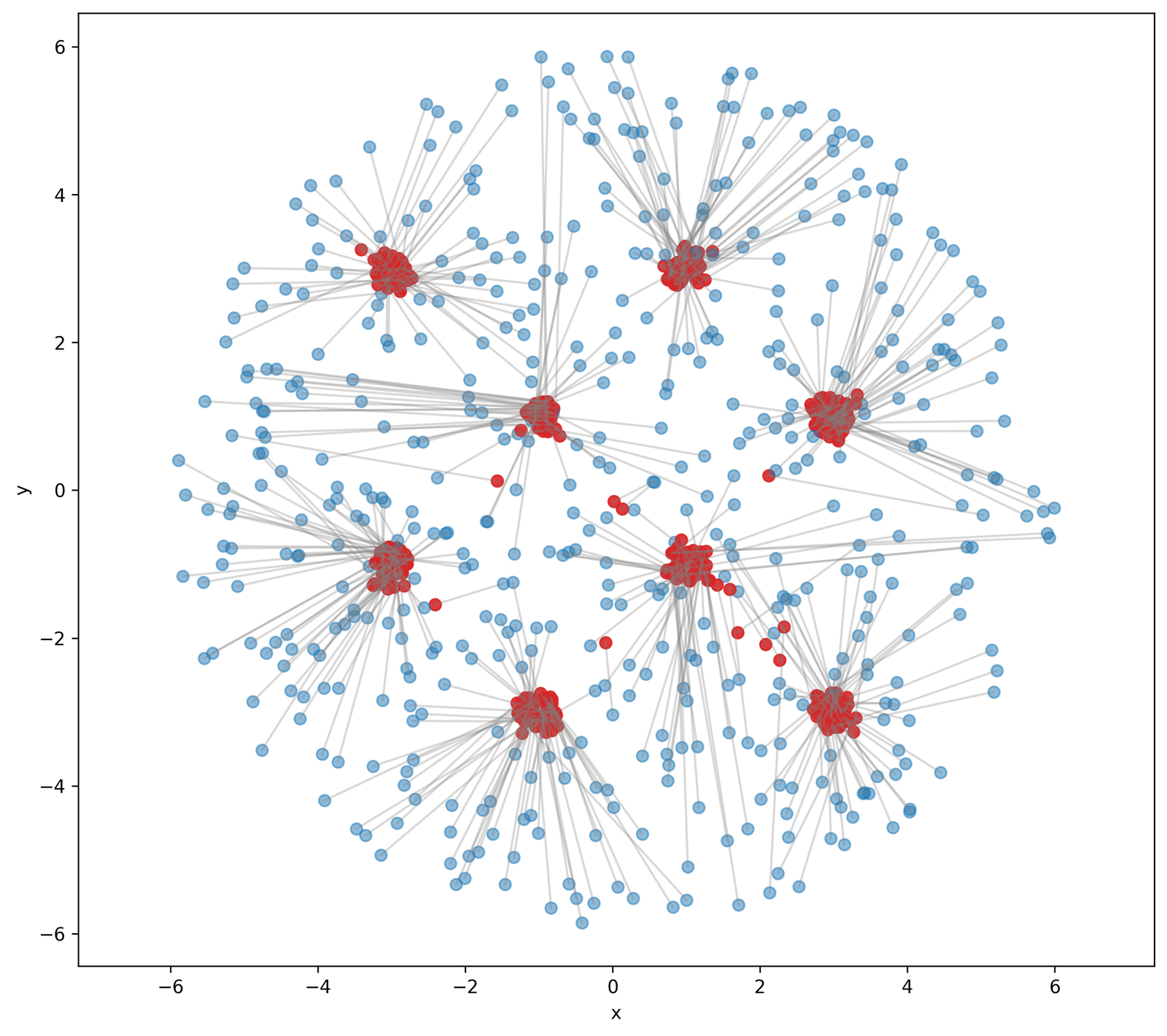} &
    \includegraphics[width=0.142\linewidth]{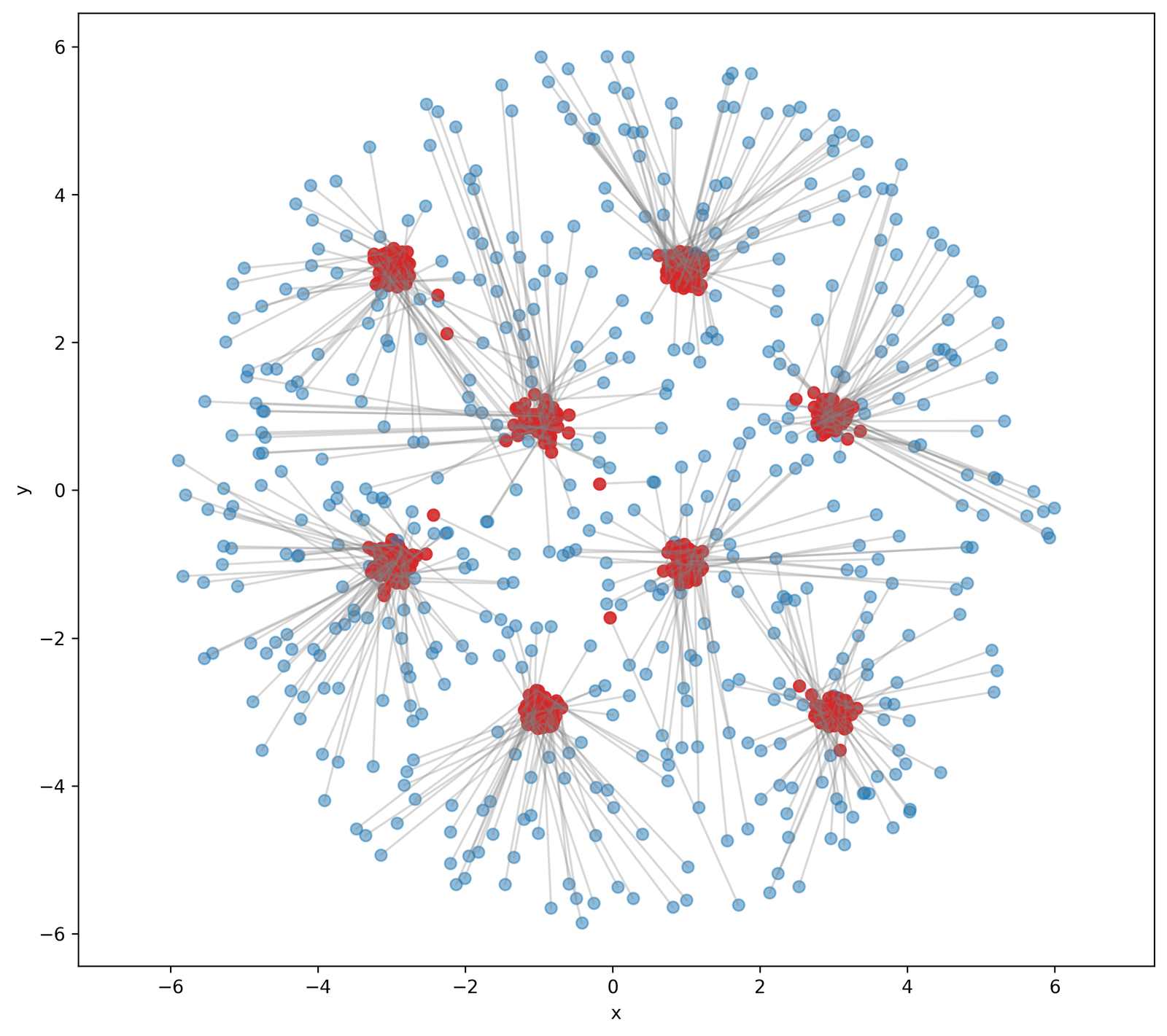} &
    \includegraphics[width=0.142\linewidth]{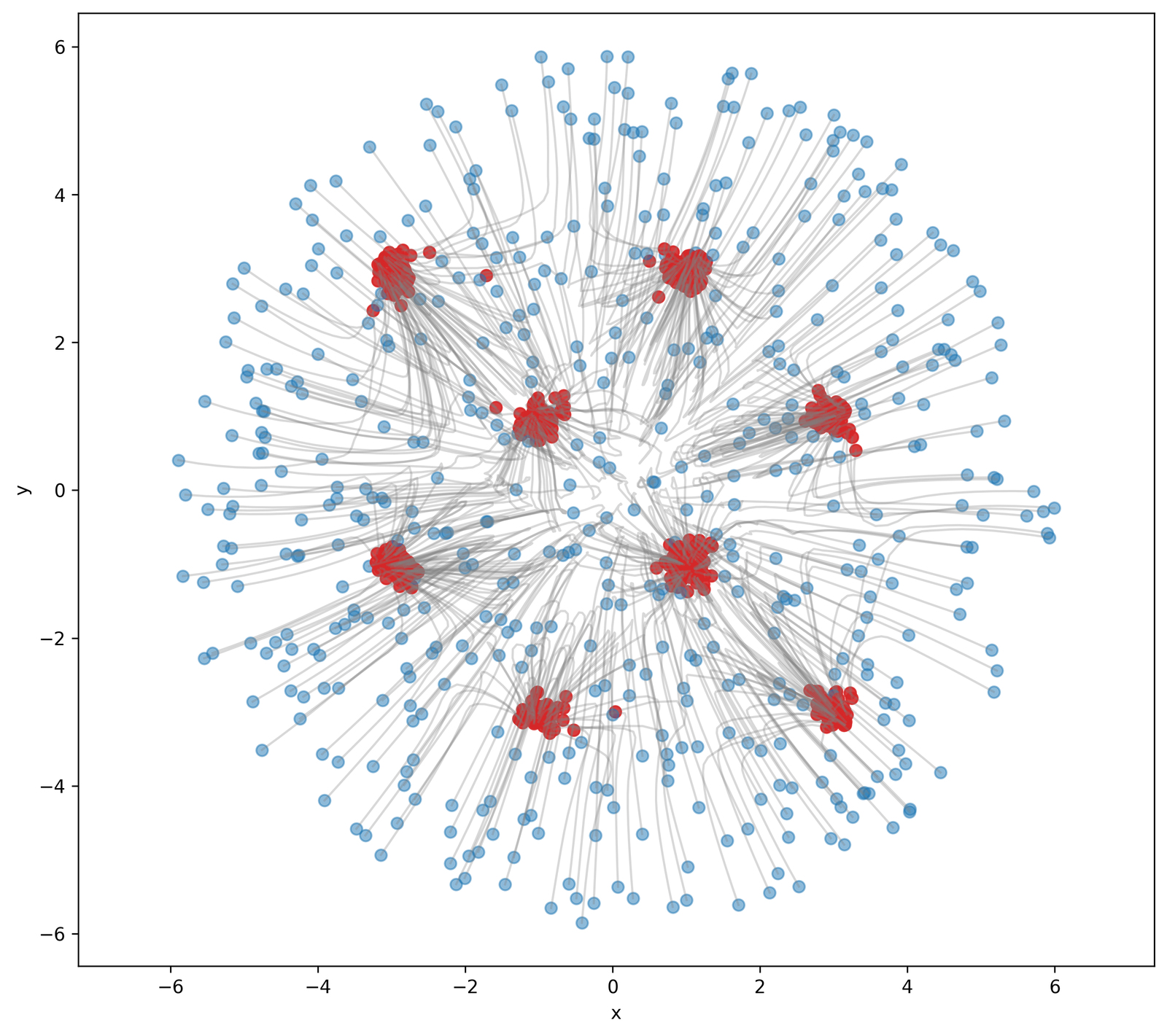} \\[-0.2em]

    \includegraphics[width=0.142\linewidth]{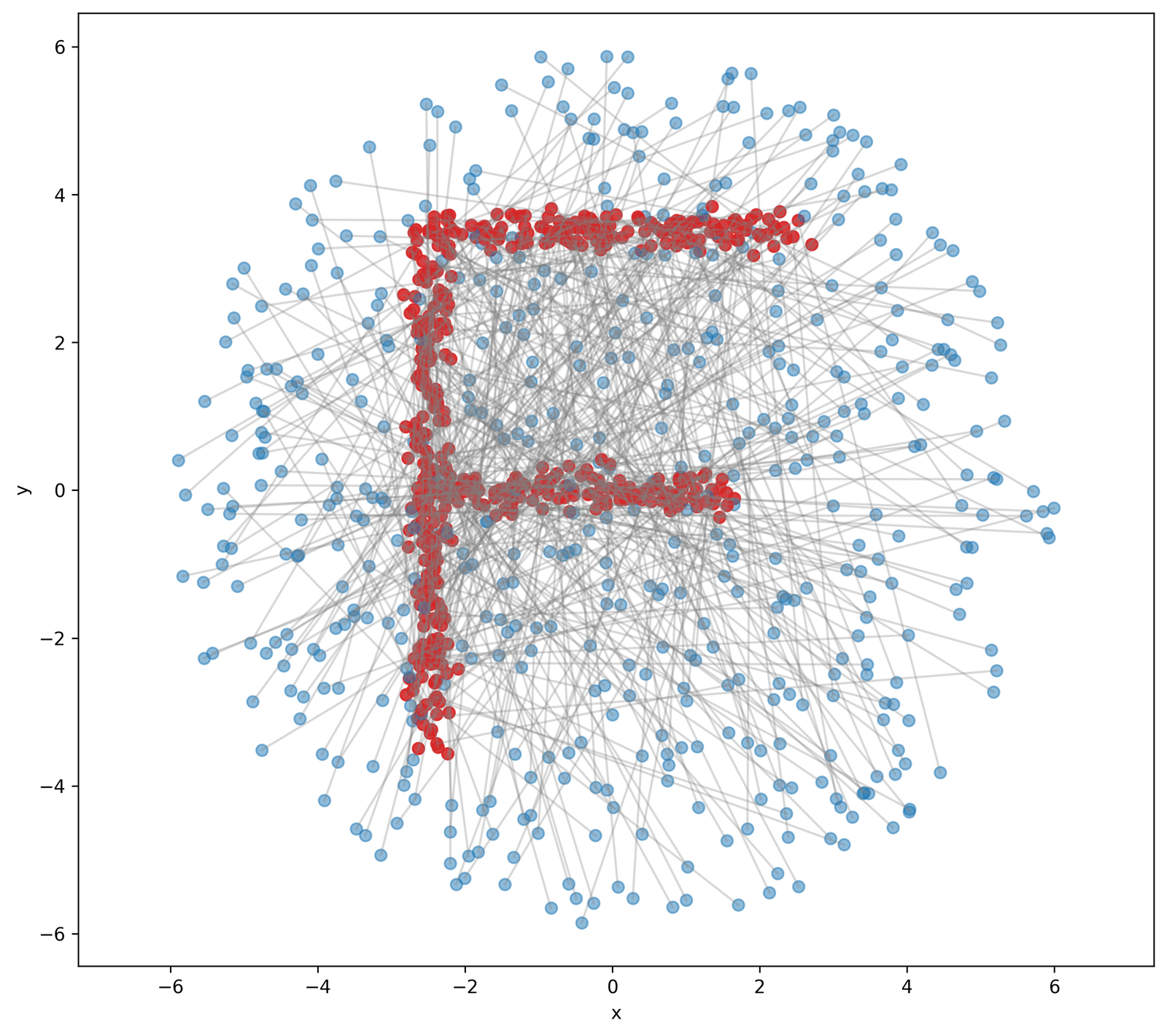} &
    \includegraphics[width=0.142\linewidth]{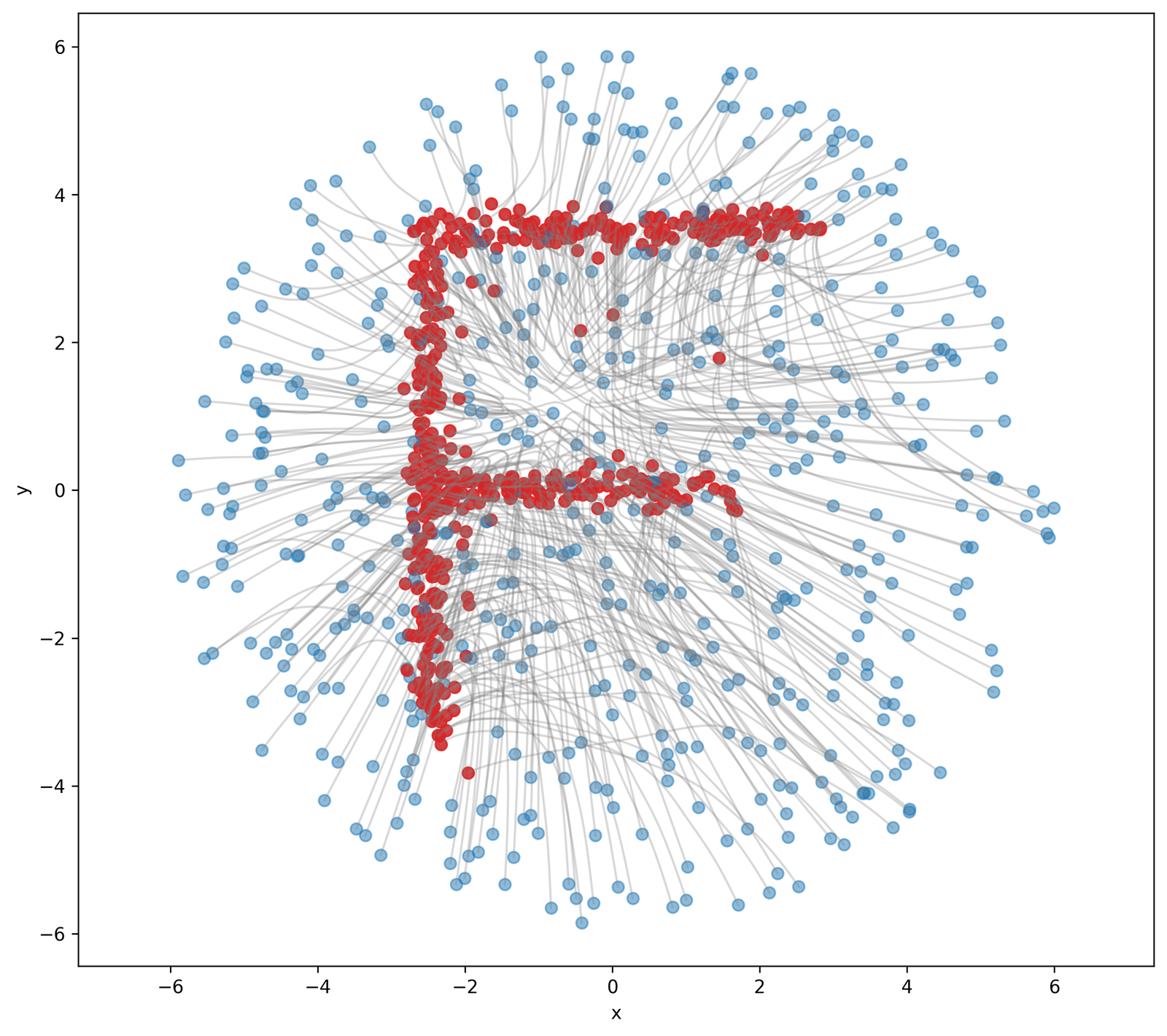} &
    \includegraphics[width=0.142\linewidth]{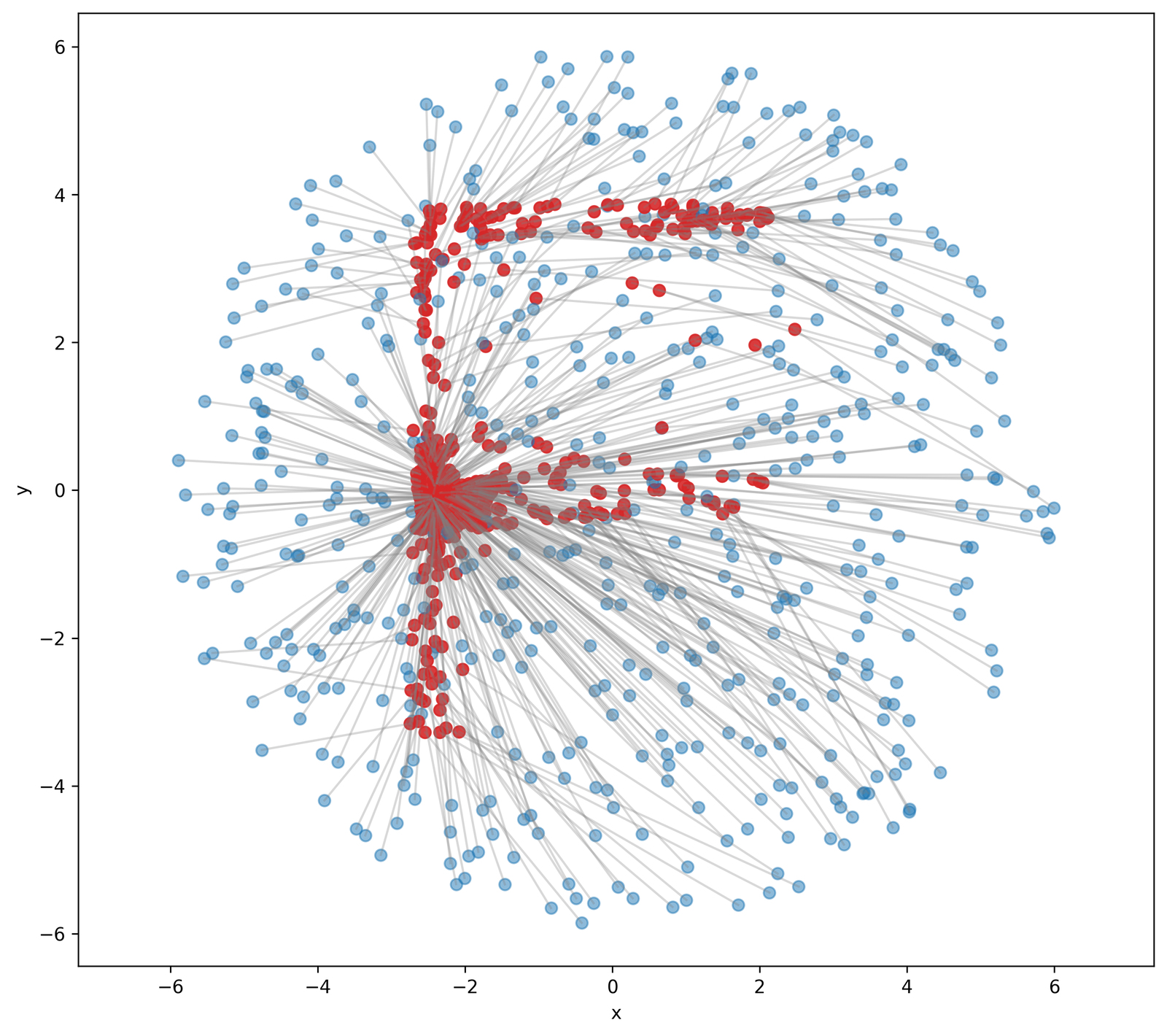} &
    \includegraphics[width=0.142\linewidth]{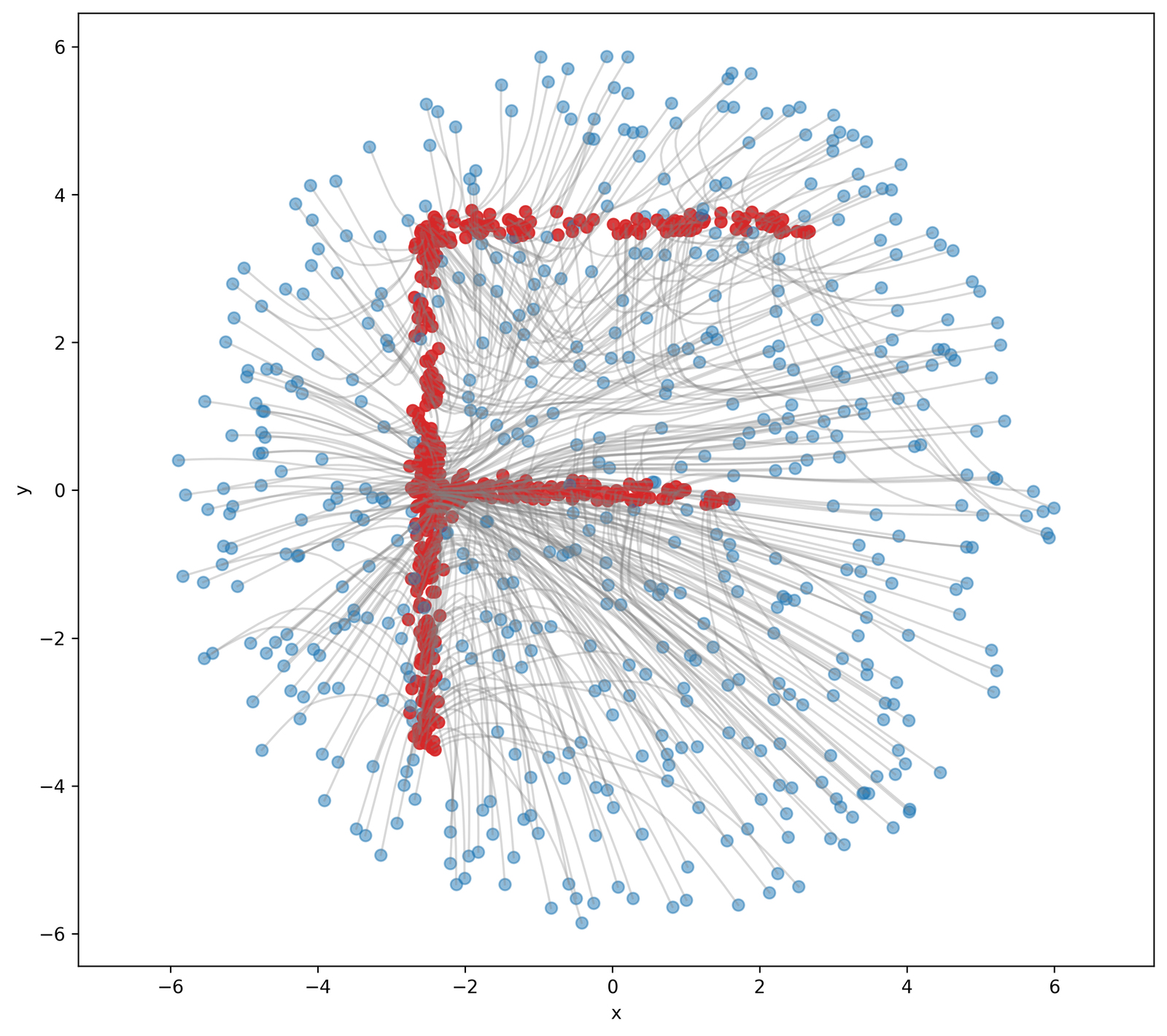} &
    \includegraphics[width=0.142\linewidth]{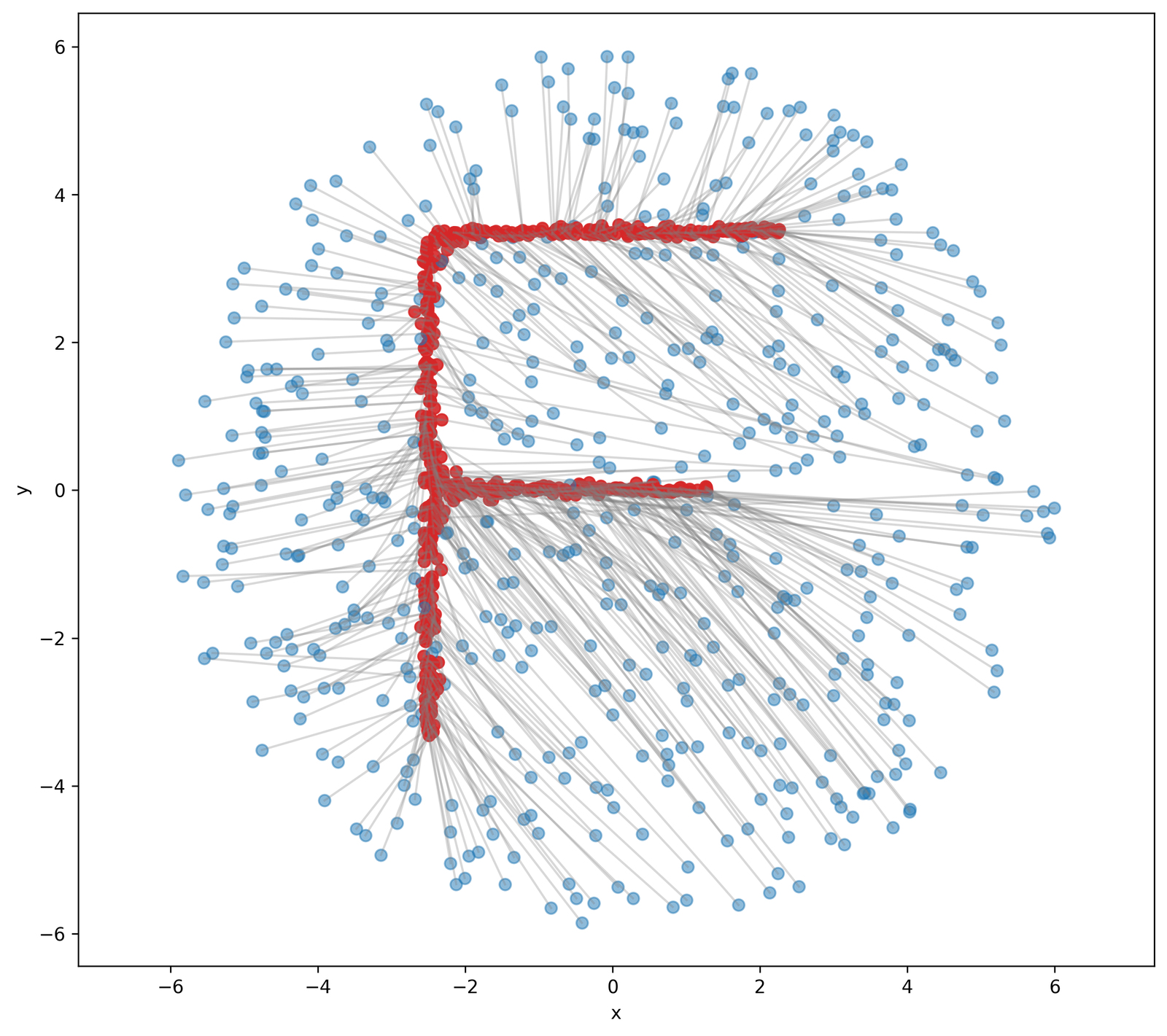} &
    \includegraphics[width=0.142\linewidth]{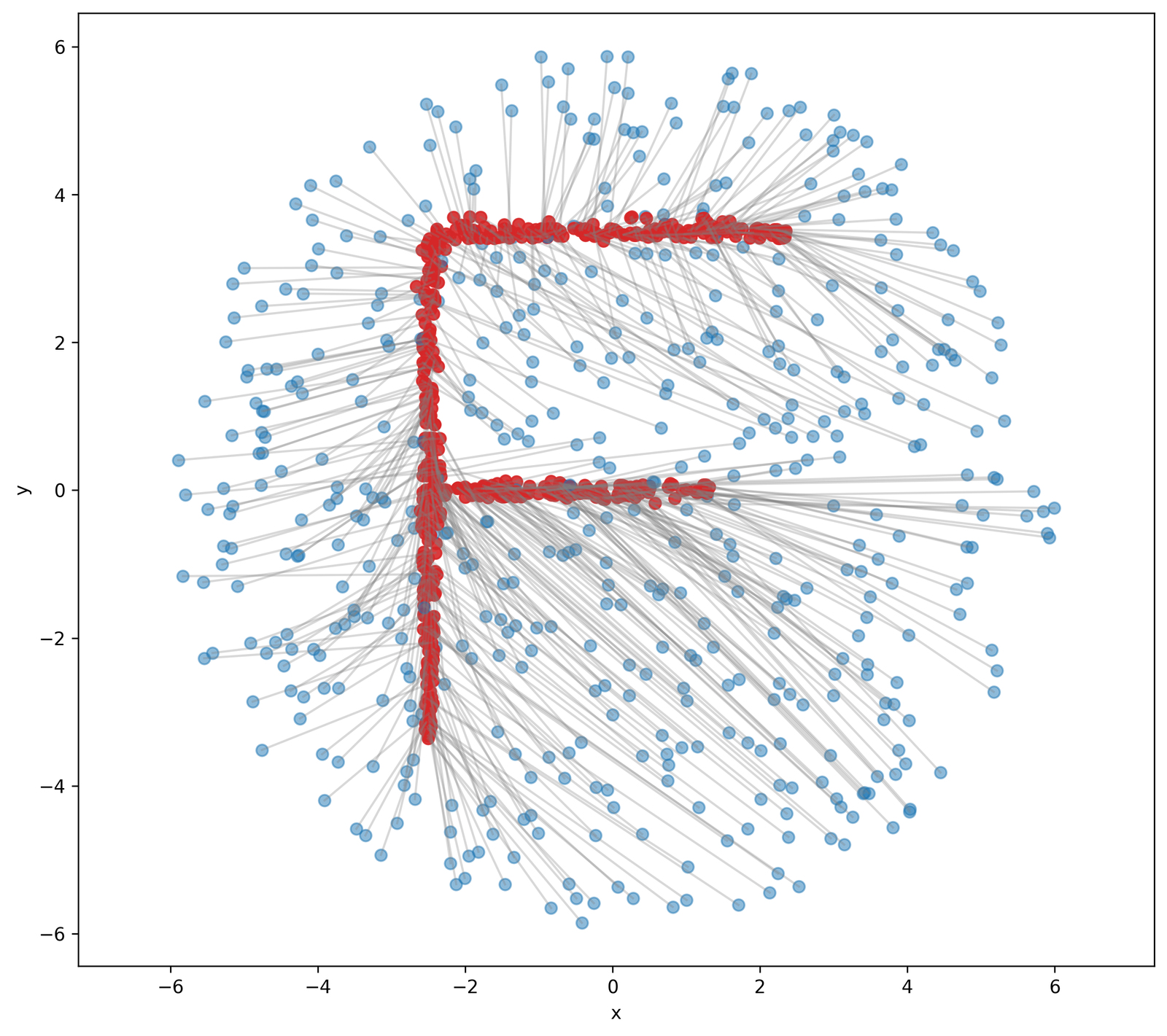} &
    \includegraphics[width=0.142\linewidth]{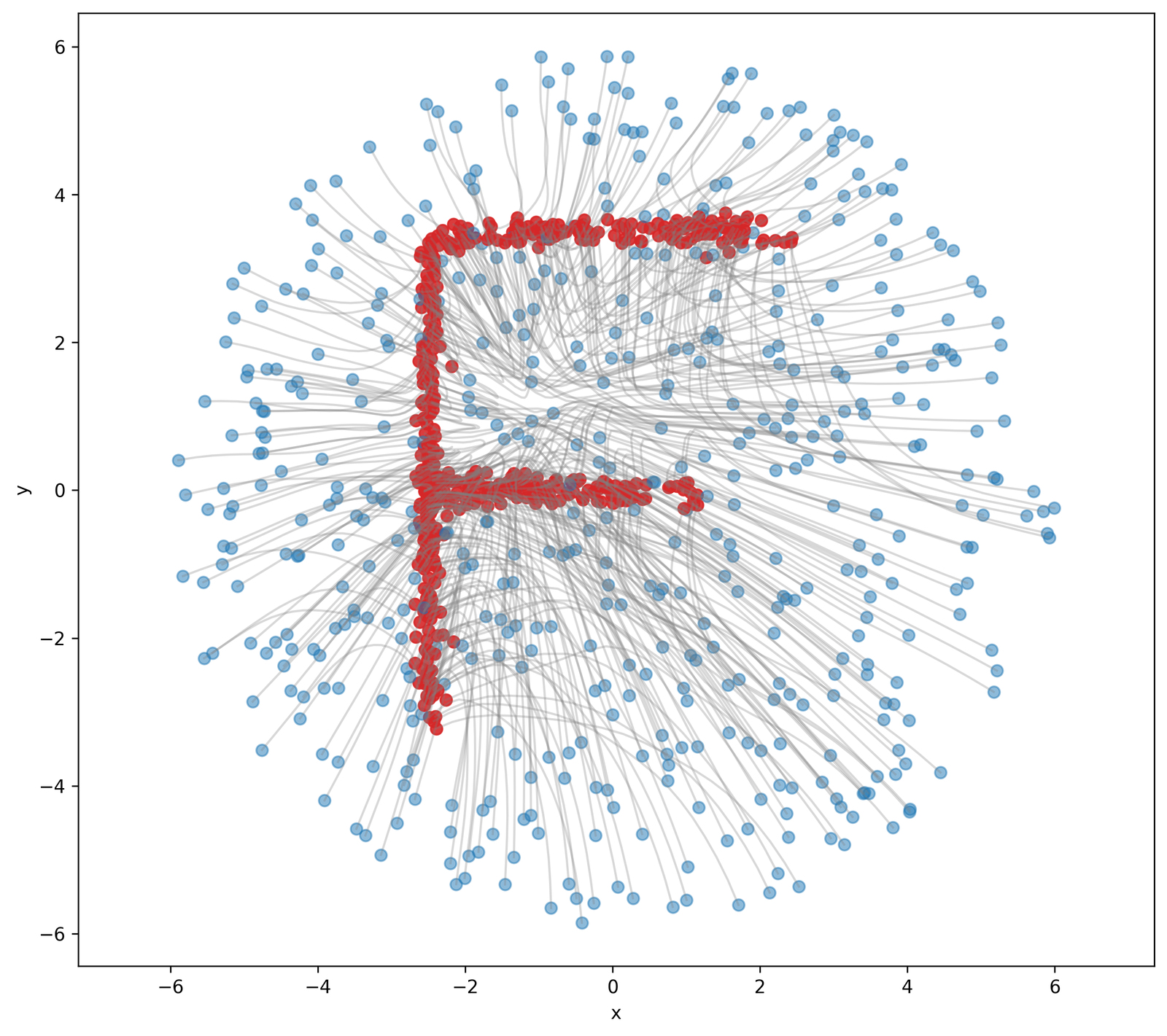} \\[-0.2em]

    \includegraphics[width=0.142\linewidth]{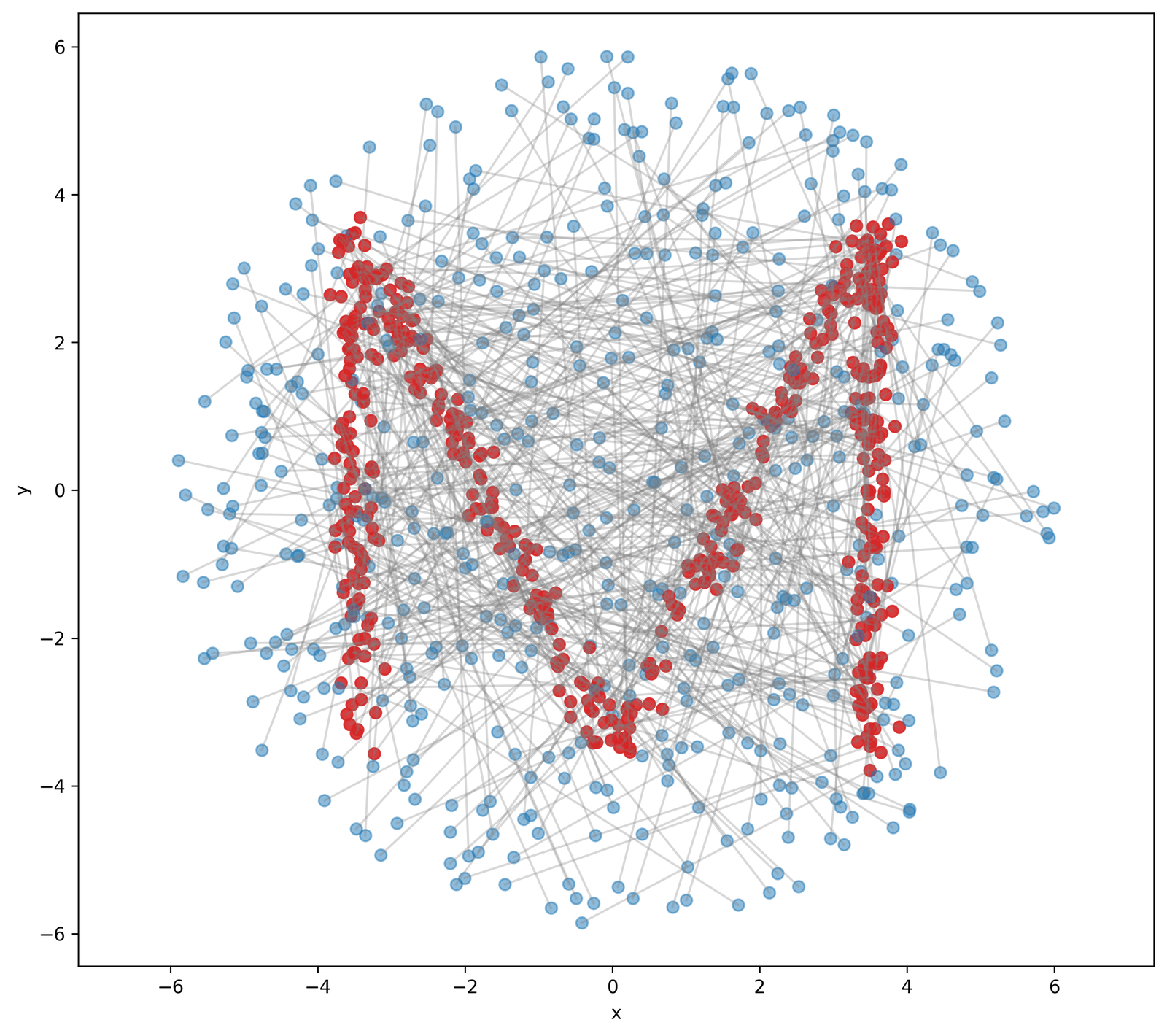} &
    \includegraphics[width=0.142\linewidth]{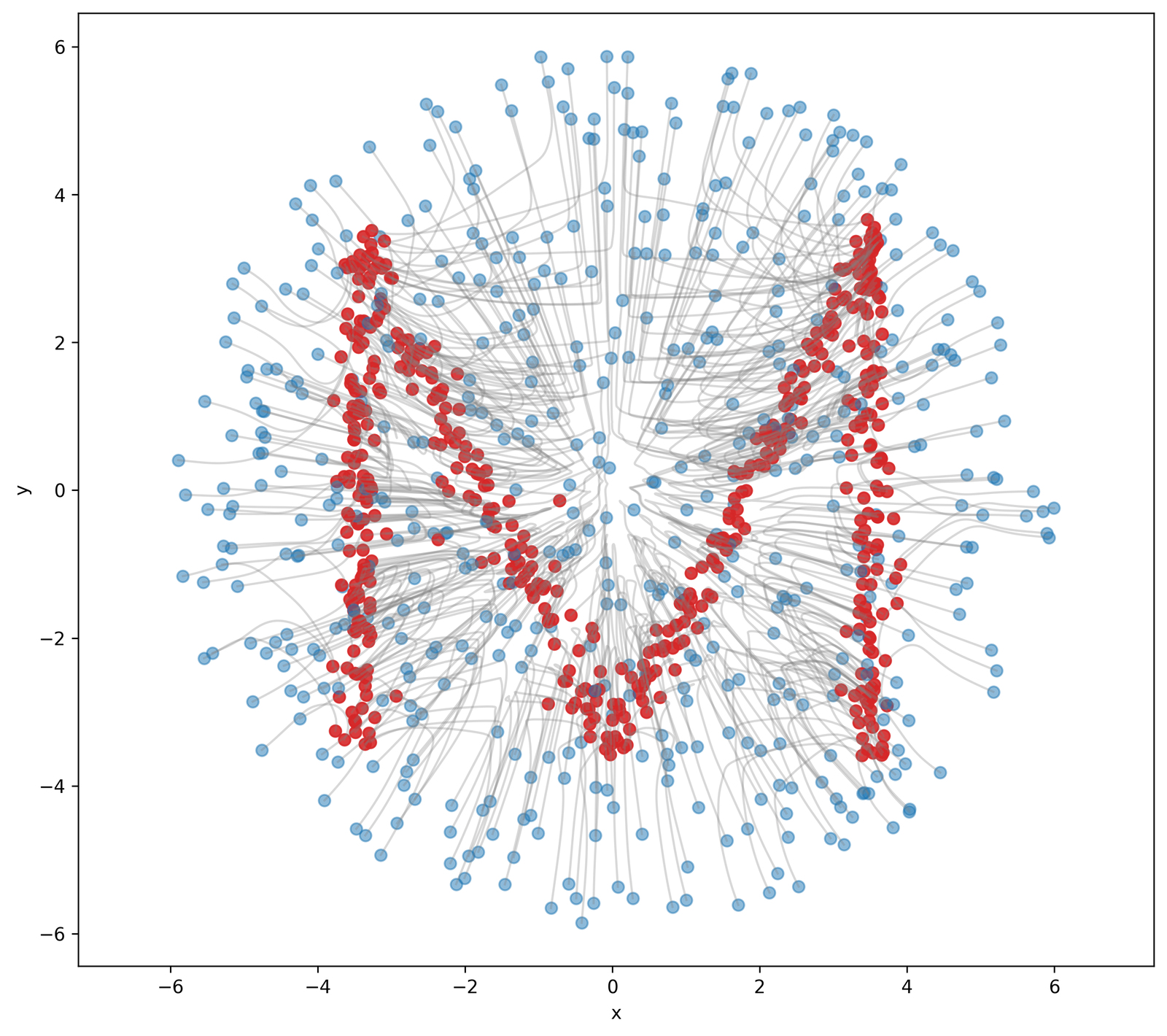} &
    \includegraphics[width=0.142\linewidth]{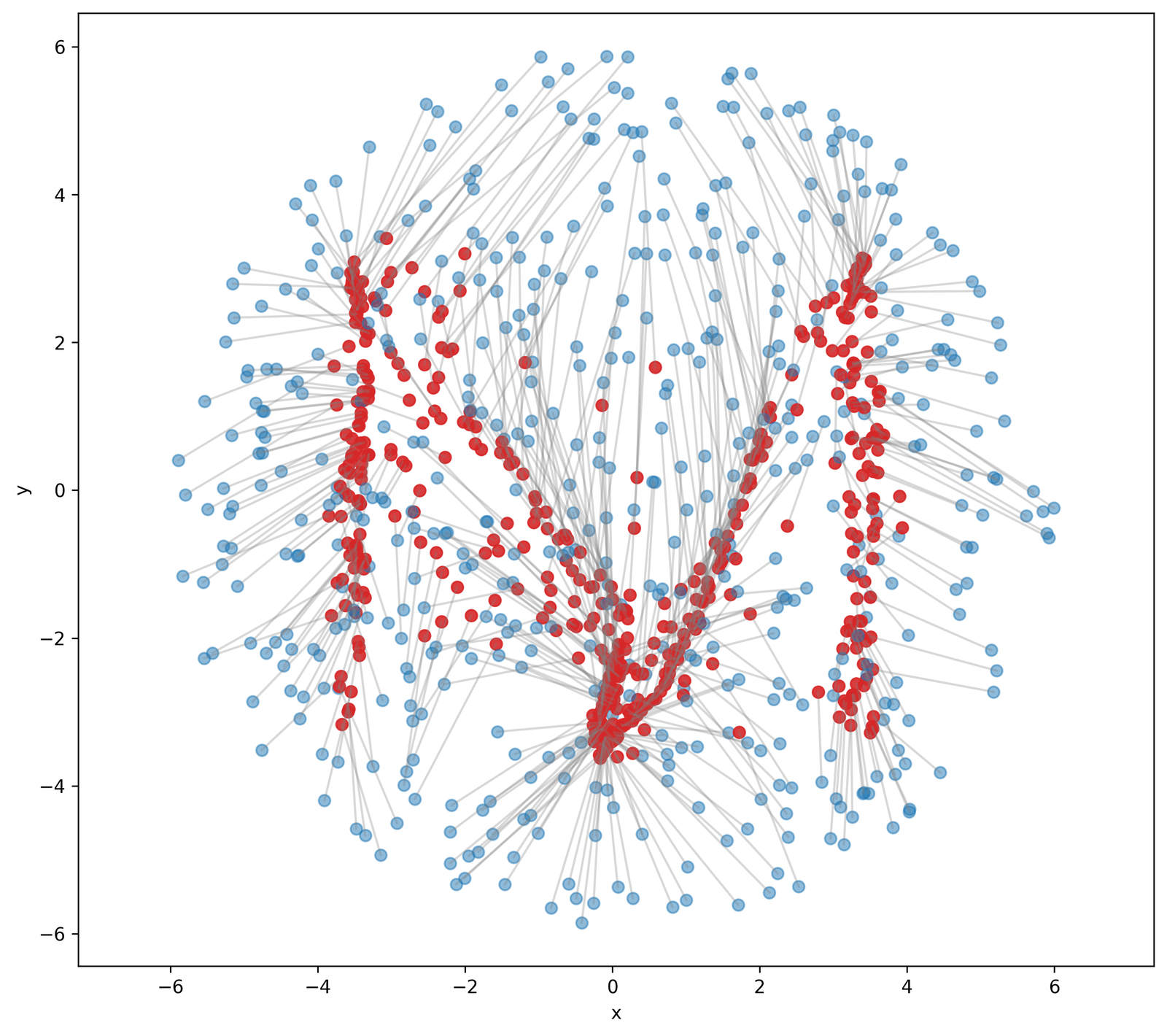} &
    \includegraphics[width=0.142\linewidth]{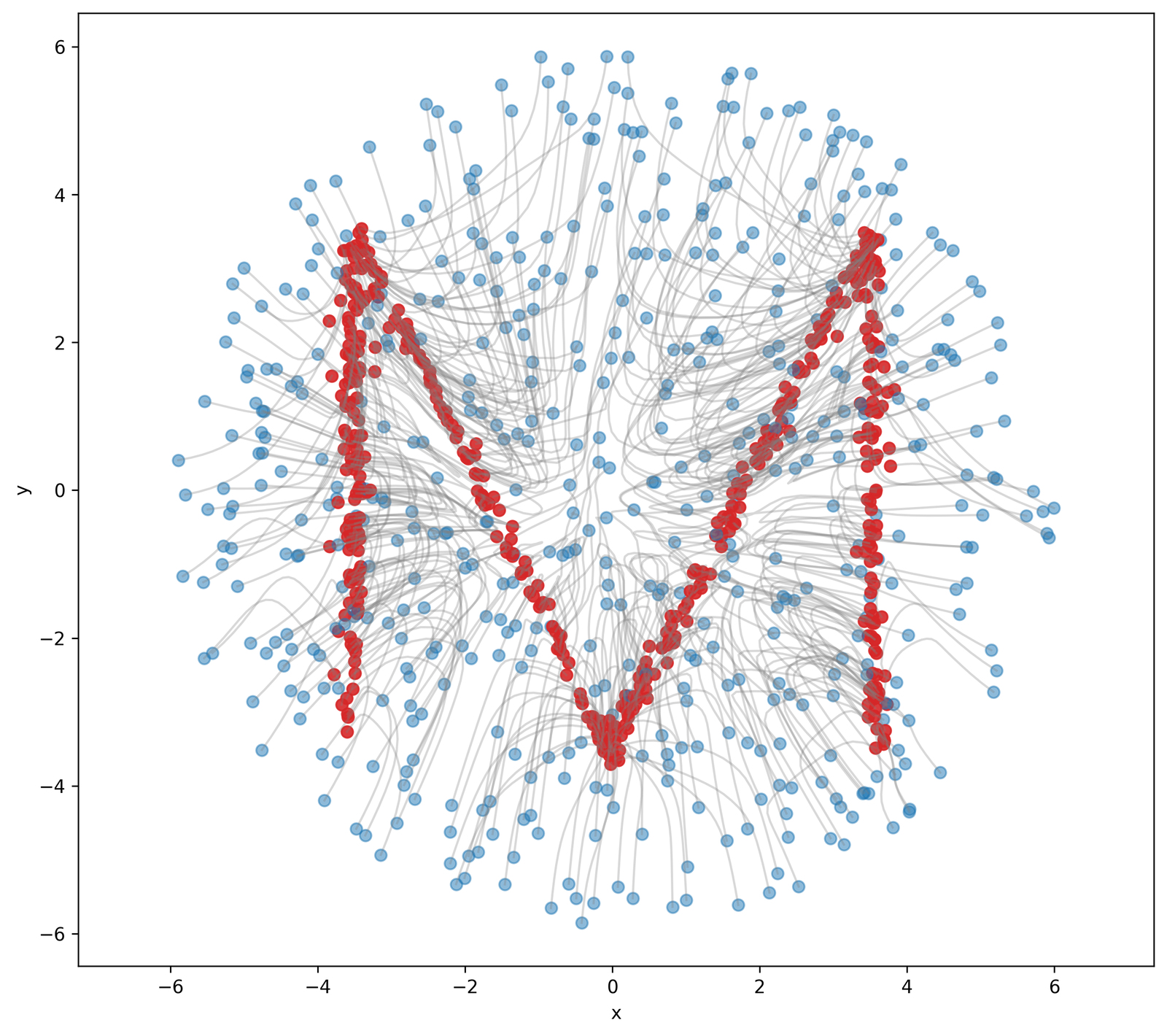} &
    \includegraphics[width=0.142\linewidth]{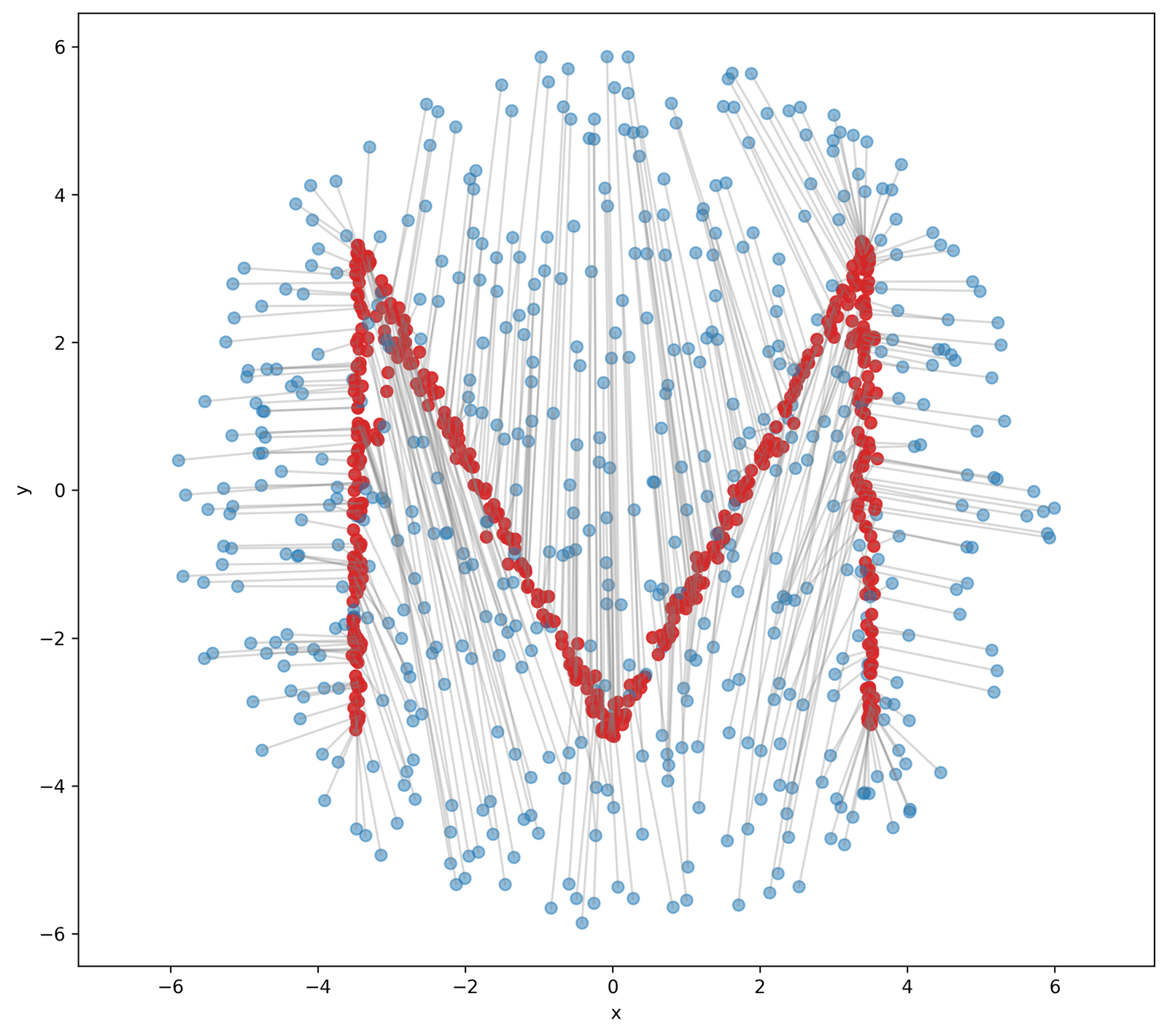} &
    \includegraphics[width=0.142\linewidth]{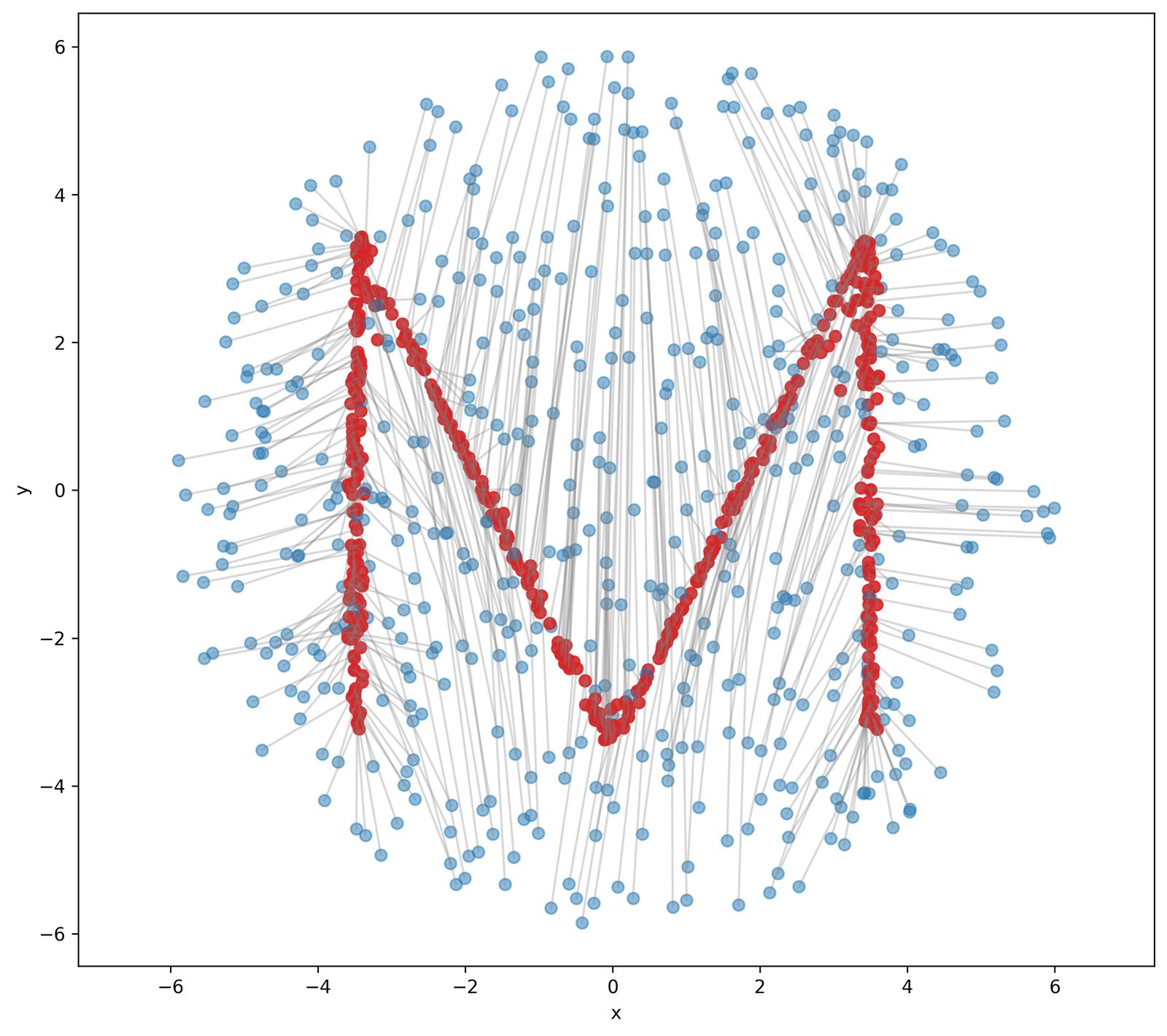} &
    \includegraphics[width=0.142\linewidth]{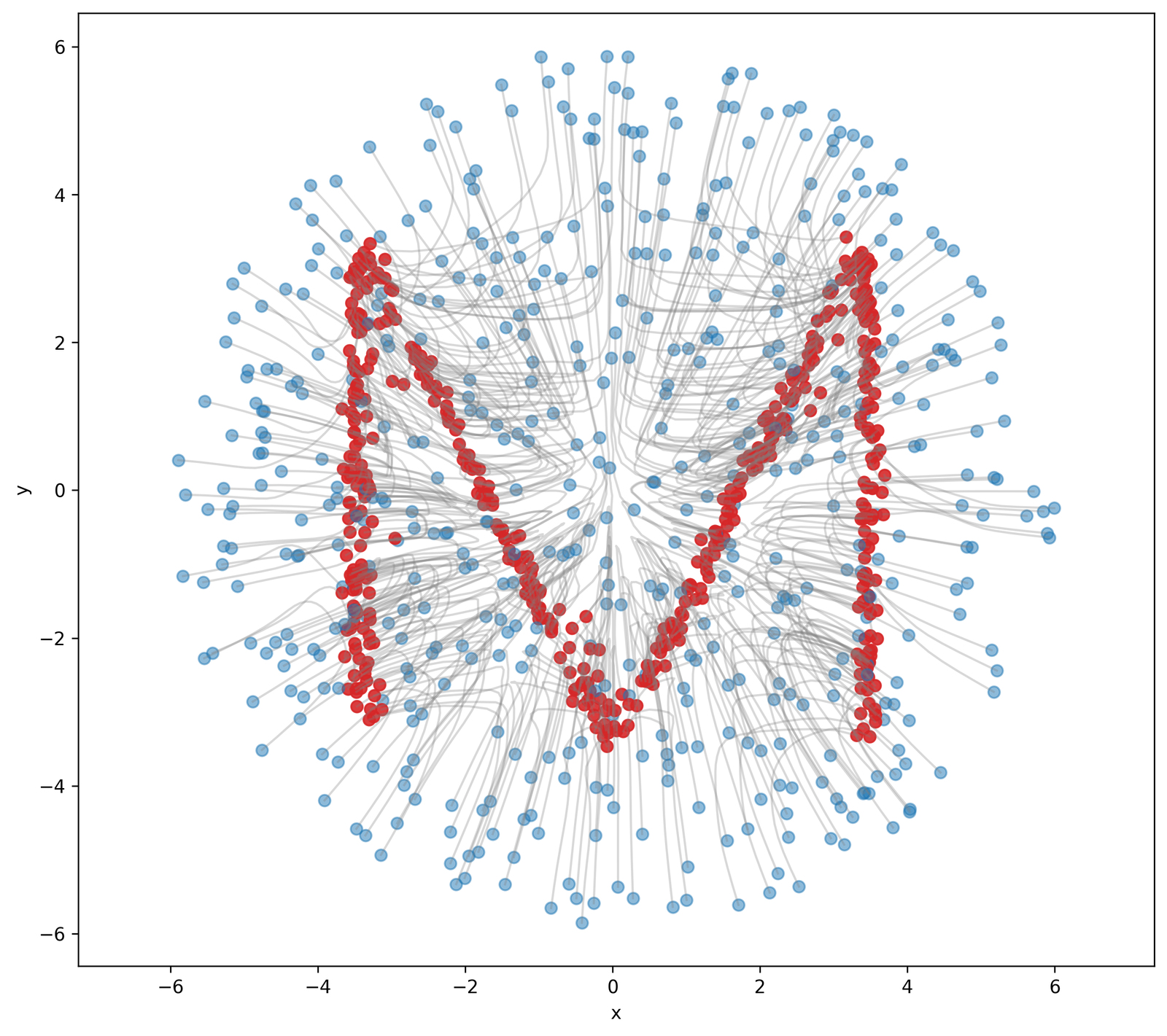} \\[-0.2em]

    \includegraphics[width=0.142\linewidth]{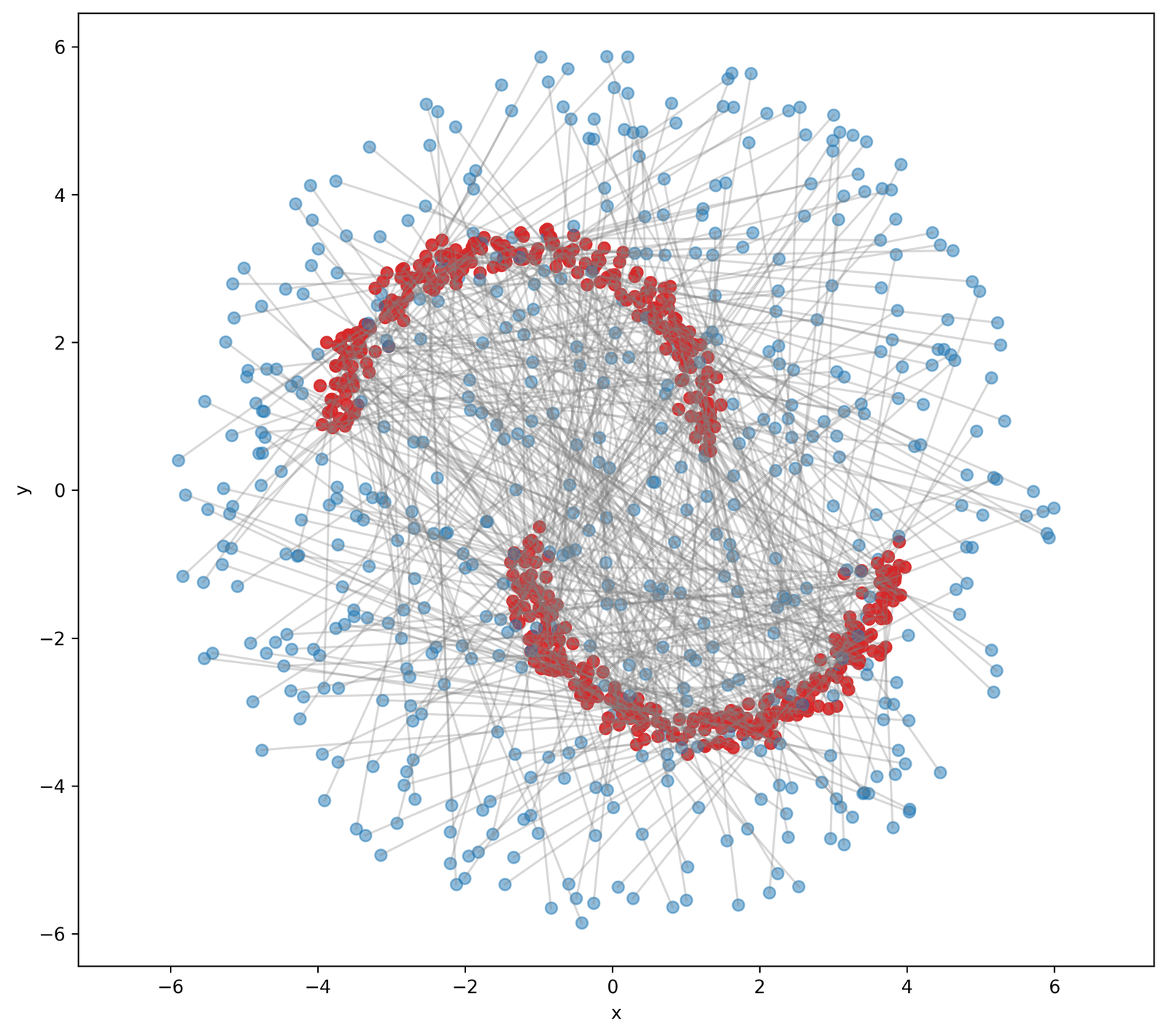} &
    \includegraphics[width=0.142\linewidth]{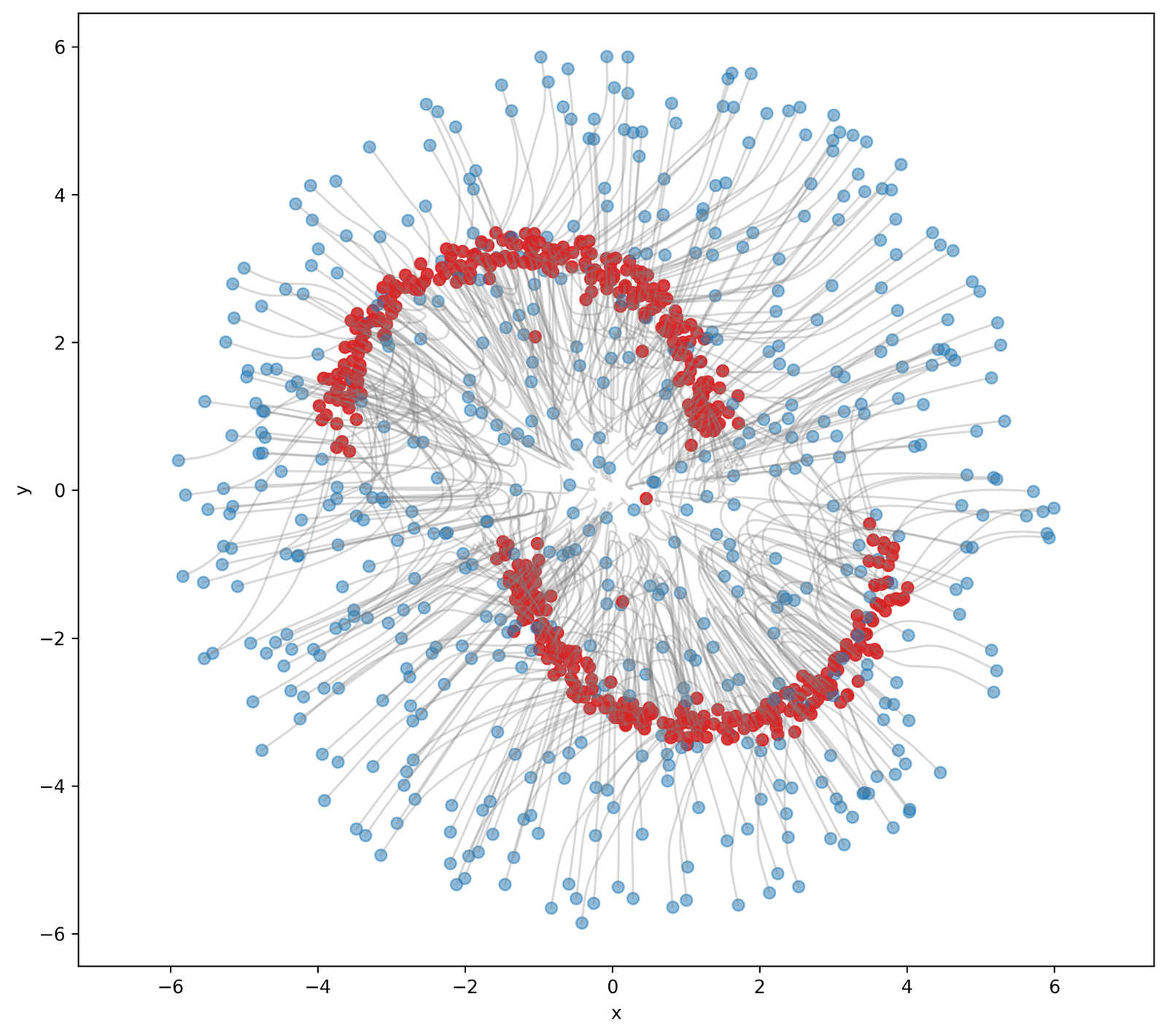} &
    \includegraphics[width=0.142\linewidth]{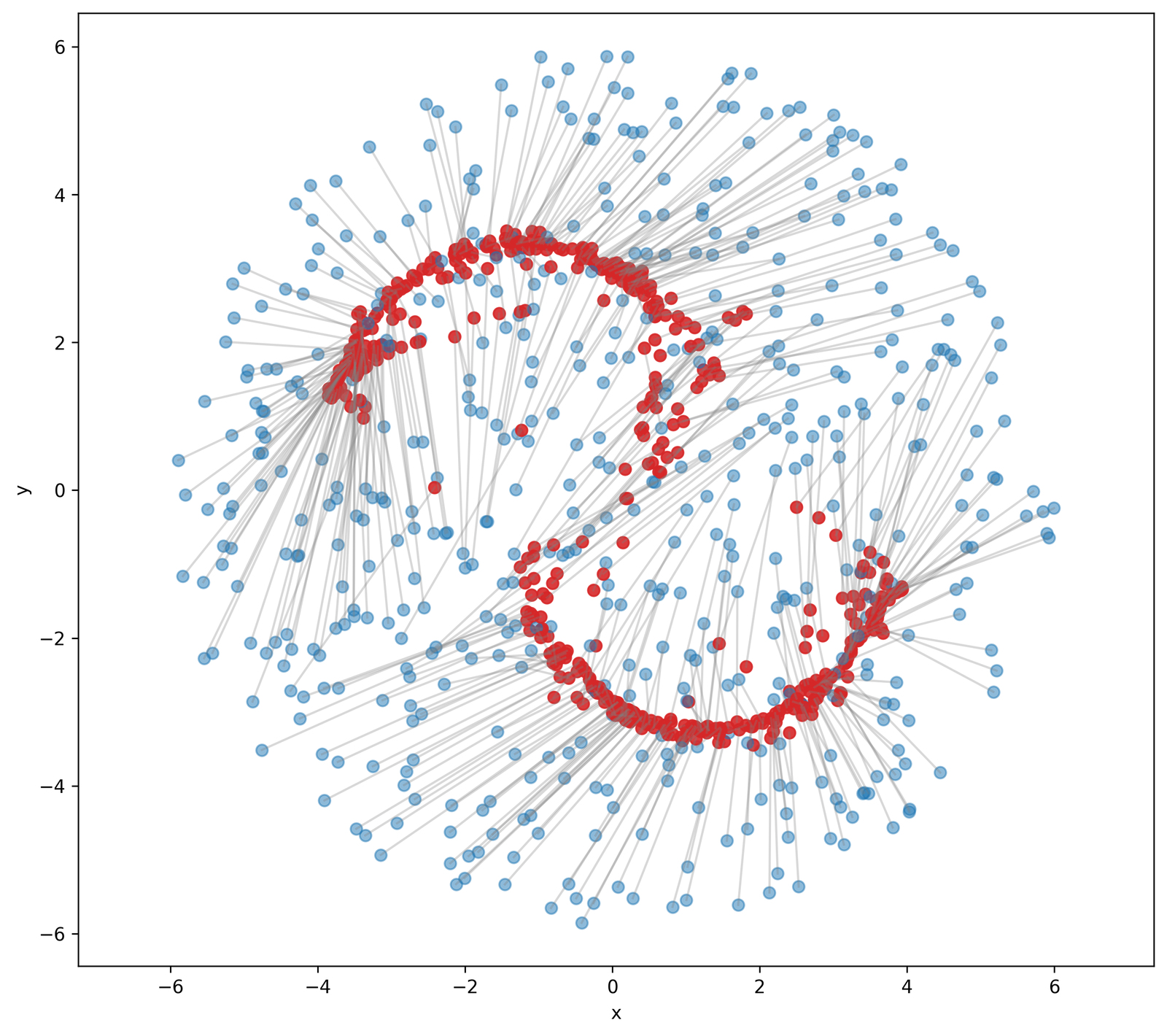} &
    \includegraphics[width=0.142\linewidth]{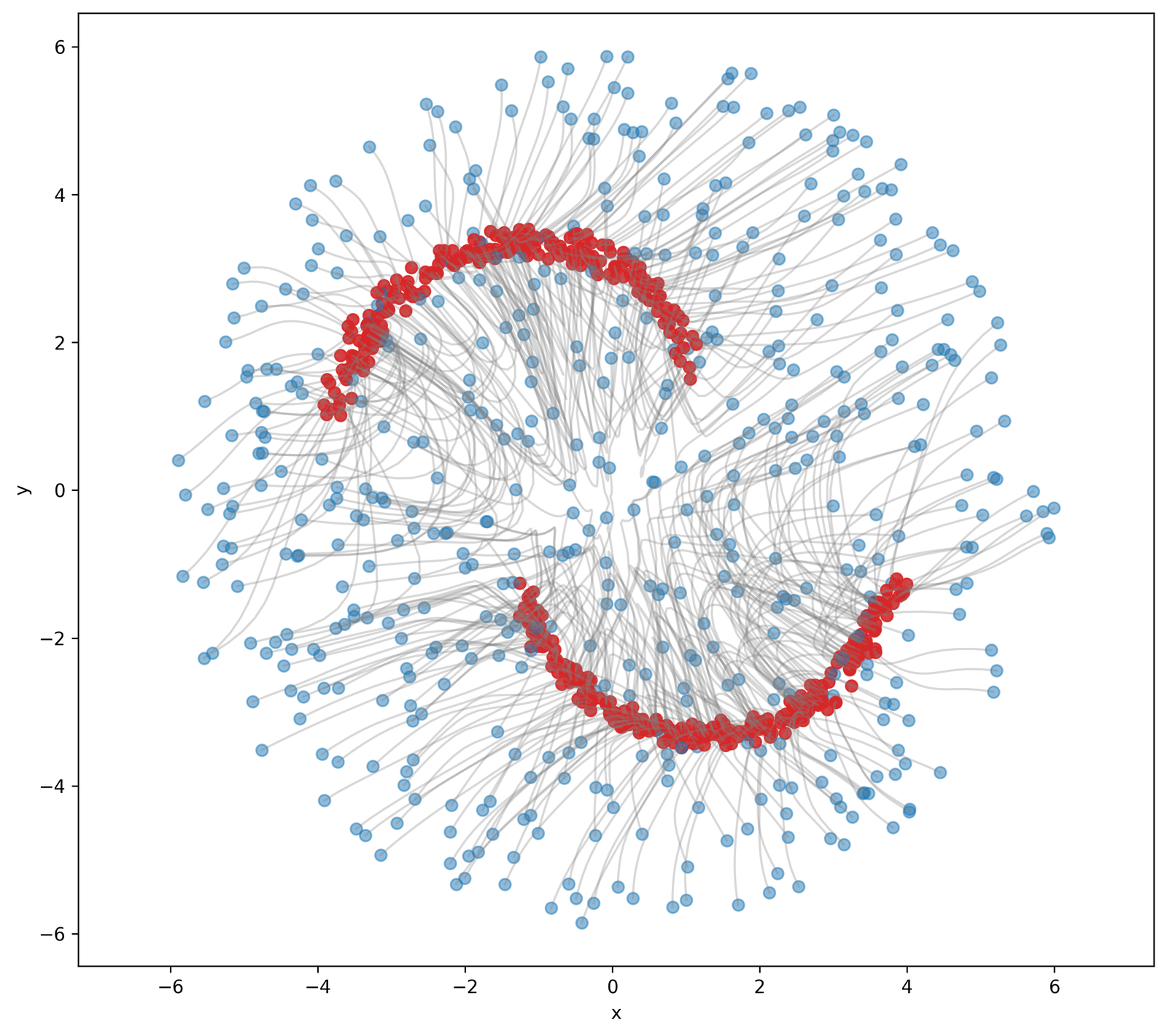} &
    \includegraphics[width=0.142\linewidth]{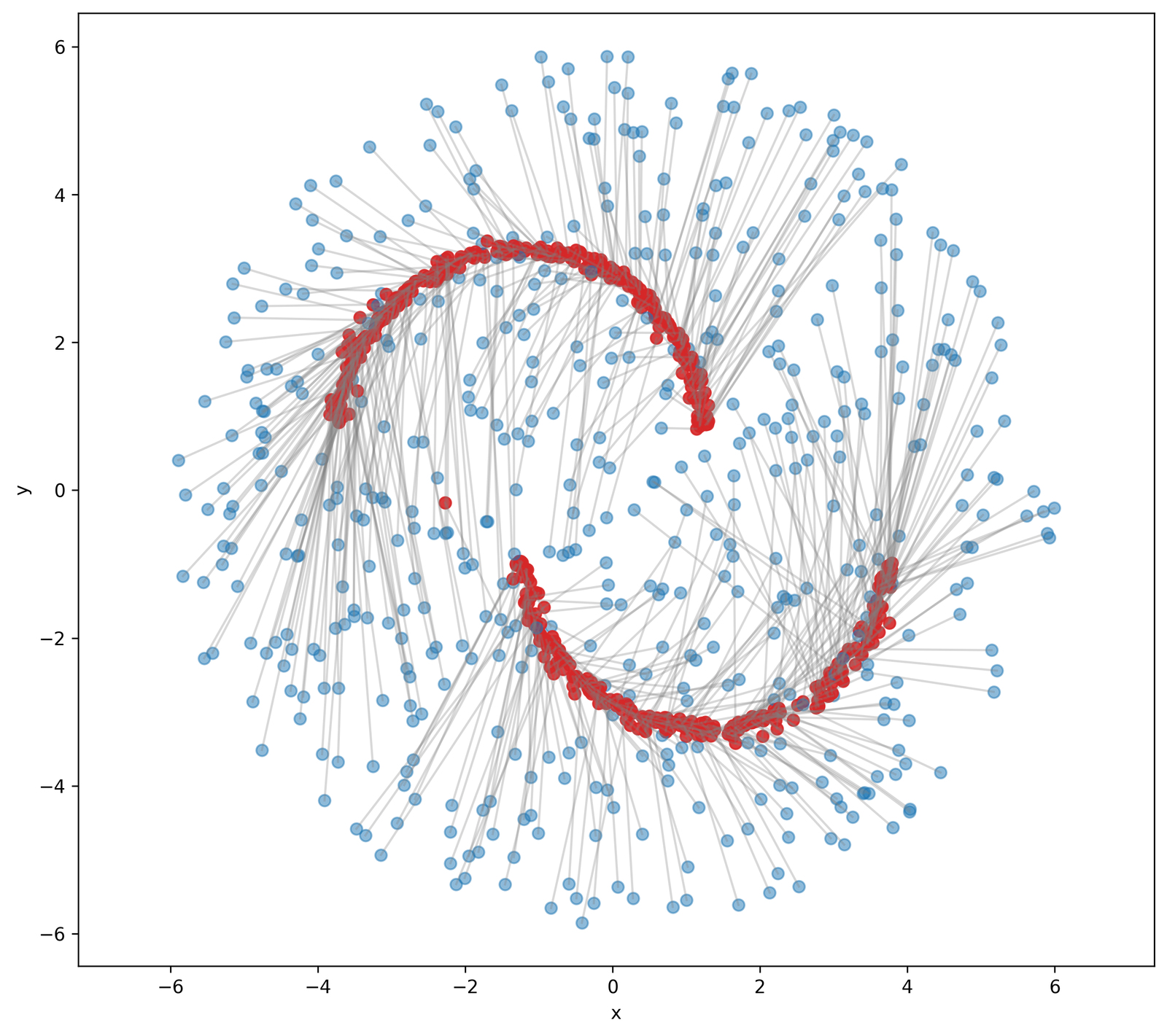} &
    \includegraphics[width=0.142\linewidth]{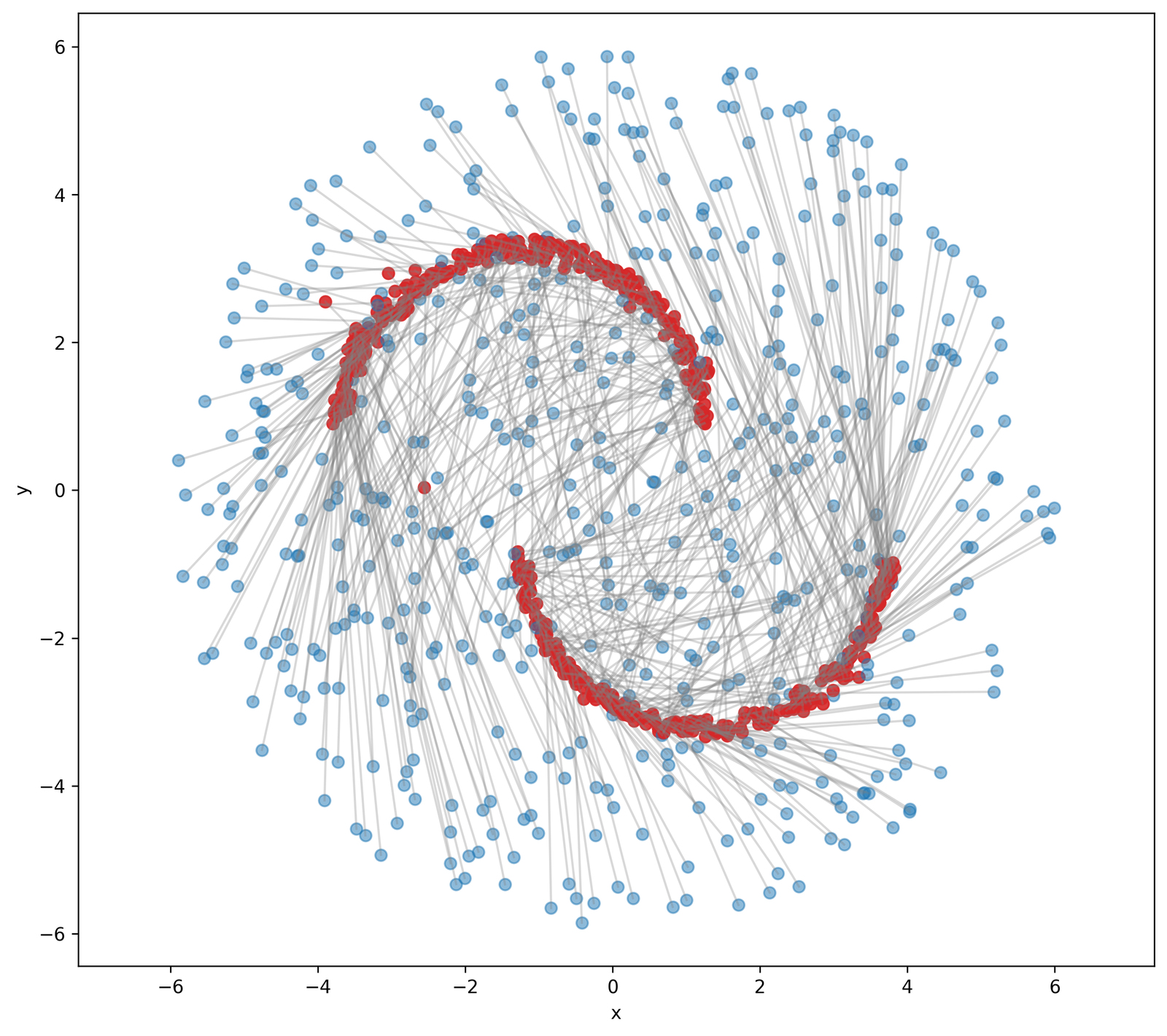} &
    \includegraphics[width=0.142\linewidth]{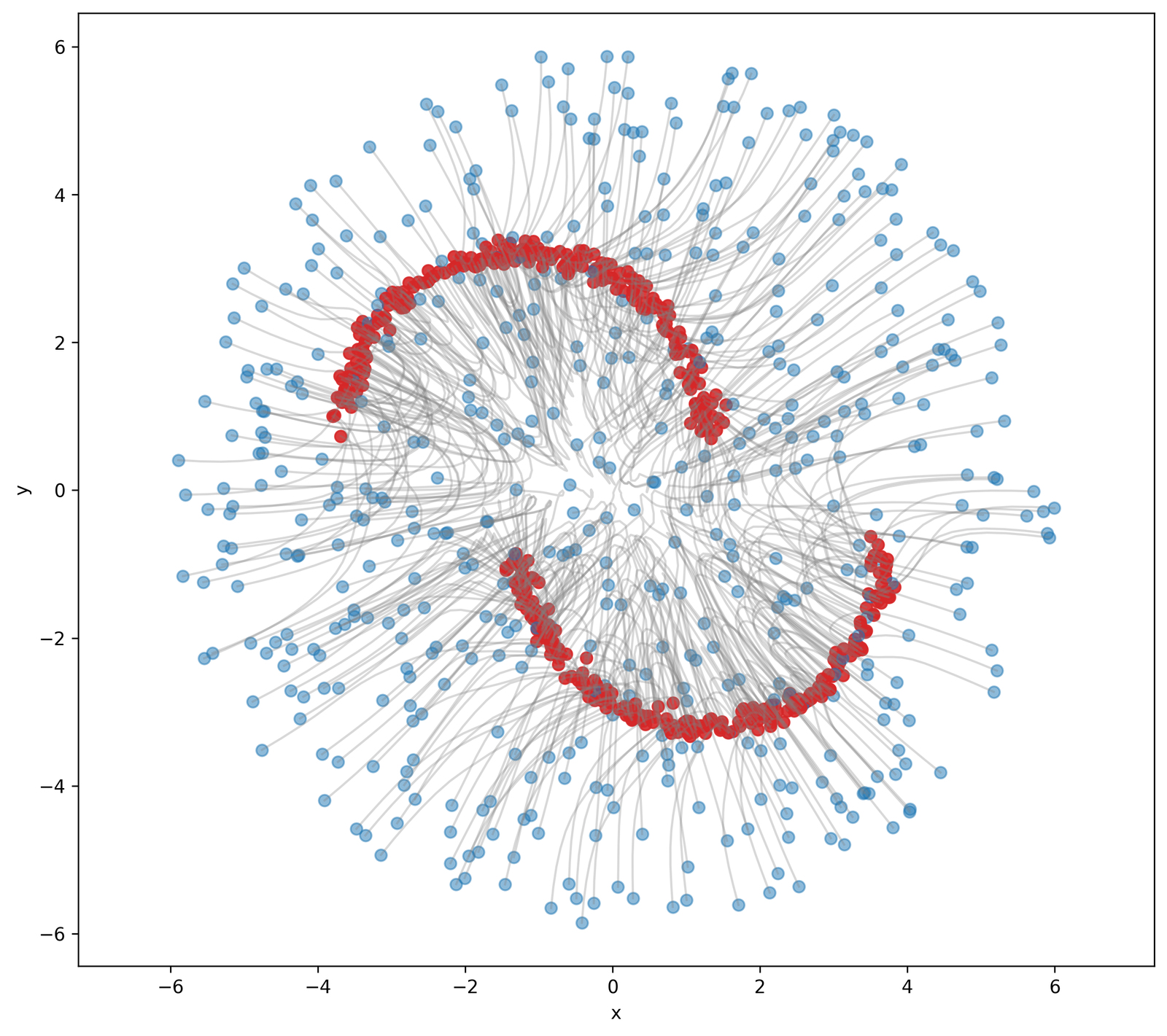} \\[0.25em]

    {\scriptsize \makecell{Ground Truth\\ \strut}} &
    {\scriptsize \makecell{Flow Matching\\Step 50\strut}} &
    {\scriptsize \makecell{Mean Flow\\Step 1\strut}} &
    {\scriptsize \makecell{Mean Flow\\Step 20\strut}} &
    {\scriptsize \makecell{Drift Model\\Step 1\strut}} &
    {\scriptsize \makecell{\textbf{DFM}\\Step 1\strut}} &
    {\scriptsize \makecell{\textbf{DFM}\\Step 20\strut}}
  \end{tabular}

  \vspace{-0.5em}
  \caption{\textbf{Generation Trajectory Visualization on 2D Synthetic Datasets.}
  Each row corresponds to one dataset, and each column corresponds to a different generation method.}
  \label{fig:viz_2d_synthetic}
  \vspace{-1em}
\end{figure*}

\section{Experiments} \label{sec:experiment}


We evaluate DFM across diverse settings, including synthetic data visualization, conditional generation on MNIST and FFHQ, large-scale ImageNet-1k synthesis, and robotic control. Across these tasks, DFM demonstrates strong one-step generation performance while consistently benefiting from the additional number of function evaluations (NFE), validating its flexibility, generation quality, and test-time scaling behavior.

\textbf{Synthetic Data and Visualization.}
Synthetic data experiments are conducted to visualize the generation trajectories and to provide an intuitive demonstration of the effectiveness of methods~\cite{zhang2025hierarchical,ma2025learningstraightflowsvariational,ma2026transition}. We simulate and visualize the generation trajectories of different methods on four two-dimensional benchmark datasets: the alphabet ``F'', the alphabet ``M'', the two-moons dataset, and the checkerboard dataset, as shown in Figure~\ref{fig:viz_2d_synthetic}. In these experiments, the source distribution, shown in blue, is a circular distribution, whereas the target distribution, shown in red, corresponds respectively to the shape of the letter ``F'', the shape of the letter ``M'', the upper--lower moons, and the checkerboard grid.
The same 3-layer MLP~\cite{rumelhart1986learning} is consistent for all methods.
The visualization results are consistent with the discussion in the previous section. Specifically, our method aims to enable generation with arbitrary step sizes and an arbitrary number of inference steps by interpolating between the Drift Model and Flow Matching. Notably, even one-step generation produces promising results, while increasing NFE further improves sample quality and distribution coverage.
\definecolor{lightmethodgray}{gray}{0.6}

\begin{table*}[t]
\centering
\scriptsize
\setlength{\tabcolsep}{2pt}
\renewcommand{\arraystretch}{1}

\begin{minipage}[t]{0.48\textwidth}
\vspace{0pt}
\centering
\begin{adjustbox}{width=\linewidth,center}
\begin{tabular}{l c c c c}
\Xhline{4\arrayrulewidth}
\multirow{2}{*}{\textbf{Method$^{\text{NFE}}$}} 
& \multicolumn{2}{c}{\textbf{MNIST}} 
& \multicolumn{2}{c}{\textbf{FFHQ}} \\
\cline{2-5}
& \textbf{EMD $\downarrow$} 
& \textbf{Accuracy $\uparrow$} 
& \textbf{EMD $\downarrow$} 
& \textbf{FID $\downarrow$} \\
\hline\hline

Flow Matching$^{50}$~\cite{lipman2023flow} 
& 37.2 & 100.0\% & 198.4 & 77.4 \\

Mean Flow$^{1}$~\cite{geng2025mean} 
& 68.6 & 96.7\% & 258.7 & 131.2 \\

Mean Flow$^{2}$~\cite{geng2025mean} 
& 48.1 & 99.1\% & 211.7 & 85.4 \\

Mean Flow$^{5}$~\cite{geng2025mean} 
& 37.5 & 100.0\% & 204.2 & 82.1 \\

Drift Model$^{1}$~\cite{deng2026generative} 
& 37.3 & 100.0\% & 225.5 & 116.2 \\

\rowcolor{cvprblue!15}
DFM$^{1}$
& 37.3 & 100.0\% & 225.5 & 116.2 \\

\rowcolor{cvprblue!15}
DFM$^{2}$
& \textbf{37.2} & \textbf{100.0\%} & 215.7 & 80.4 \\

\rowcolor{cvprblue!15}
DFM$^{5}$
& \textbf{37.2} & \textbf{100.0\%} & \textbf{196.2} & \textbf{75.9} \\

\Xhline{4\arrayrulewidth}
\end{tabular}
\end{adjustbox}

\vspace{0.4em}
\caption{
Quantitative comparison of different generation methods on MNIST and FFHQ datasets.
}
\label{tab:mnist_ffhq_comparison}

\vspace{0.8em}
\caption{
Quantitative Comparison with Different Generation Methods on ImageNet $256 \times 256$ Dataset.
}
\label{tab:image256}

\end{minipage}
\hfill
\begin{minipage}[t]{0.50\textwidth}
\vspace{0pt}
\centering
\begin{adjustbox}{width=\linewidth,center}
\begin{tabular}{lcccc}
\Xhline{4\arrayrulewidth}
\textbf{Method} & \textbf{\# Params.} & \textbf{NFE} & \textbf{FID $\downarrow$} & \textbf{IS $\uparrow$} \\
\hline\hline

\textcolor{lightmethodgray}{DiT-XL/2~\cite{esser2024scalingrectifiedflowtransformers}}
& \textcolor{lightmethodgray}{675M} & \textcolor{lightmethodgray}{$250\times2$} & \textcolor{lightmethodgray}{2.27} & \textcolor{lightmethodgray}{278.2} \\

\textcolor{lightmethodgray}{SiT-XL/2~\cite{ma2024sit}}
& \textcolor{lightmethodgray}{675M} & \textcolor{lightmethodgray}{$250\times2$} & \textcolor{lightmethodgray}{2.06} & \textcolor{lightmethodgray}{270.3} \\

\textcolor{lightmethodgray}{SiT-XL/2+REPA~\cite{yu2024representation}}
& \textcolor{lightmethodgray}{675M} & \textcolor{lightmethodgray}{$250\times2$} & \textcolor{lightmethodgray}{1.42} & \textcolor{lightmethodgray}{305.7} \\

\textcolor{lightmethodgray}{LightningDiT-XL/2~\cite{yao2025reconstruction}}
& \textcolor{lightmethodgray}{675M} & \textcolor{lightmethodgray}{$250\times2$} & \textcolor{lightmethodgray}{1.35} & \textcolor{lightmethodgray}{295.3} \\

\textcolor{lightmethodgray}{RAE+DiT$^{\mathrm{DH}}$-XL/2~\cite{zheng2025diffusion}}
& \textcolor{lightmethodgray}{839M} & \textcolor{lightmethodgray}{$50\times2$} & \textbf{\textcolor{lightmethodgray}{1.13}} & \textcolor{lightmethodgray}{262.6} \\

\hline

iCT-XL/2~\cite{song2024improved} 
& 675M & 1 & 34.24 & -- \\

Shortcut-XL/2~\cite{frans2025one} 
& 675M & 1 & 10.60 & -- \\

MeanFlow-XL/2~\cite{geng2025mean} 
& 676M & 1 & 3.43 & -- \\

AdvFlow-XL/2~\cite{lin2025adversarial} 
& 673M & 1 & 2.38 & 284.2 \\

iMeanFlow-XL/2~\cite{geng2025improved} 
& 610M & 1 & 1.72 & 282.0 \\

S-VFM-XL/2~\cite{ma2025learningstraightflowsvariational} 
& 677M & 1 & 3.31 & -- \\

Drifting Model, L/2~\cite{deng2026generative} 
& 463M & 1 & 1.54 & 258.9 \\

\rowcolor{cvprblue!15}
DFM, L/2 
& 463M & 1 & \textbf{1.52} & 259.4 \\

\hline

iCT-XL/2~\cite{song2024improved} 
& 675M & 2 & 20.30 & -- \\

iMM-XL/2~\cite{zhou2025inductive} 
& 675M & $1\times2$ & 7.77 & -- \\

MeanFlow-XL/2~\cite{geng2025mean} 
& 676M & 2 & 2.93 & -- \\

AdvFlow-XL/2~\cite{lin2025adversarial} 
& 673M & 2 & 2.11 & 288.7 \\

iMeanFlow-XL/2~\cite{geng2025improved} 
& 610M & 2 & 1.54 & 285.1 \\

S-VFM-XL/2~\cite{ma2025learningstraightflowsvariational} 
& 677M & 2 & 2.86 & -- \\

\rowcolor{cvprblue!15}
DFM, L/2 
& 463M & 2 & \textbf{1.45} & 268.3 \\

\hline

\rowcolor{cvprblue!15}
DFM, L/2 
& 463M & 5 & 1.34 & 279.8 \\

\hline

\rowcolor{cvprblue!15}
DFM, L/2 
& 463M & 10 & 1.31 & 287.4 \\

\Xhline{4\arrayrulewidth}
\end{tabular}
\end{adjustbox}
\end{minipage}

\vspace{-1em}
\end{table*}

\textbf{Evaluation for Conditional Generation (MNIST).}
\begin{figure*}[htbp]
  \centering
  \setlength{\tabcolsep}{0pt}
  \renewcommand{\arraystretch}{0.95}

  \begin{tabular}{ccc}
    \begin{minipage}{0.33\linewidth}
      \centering
      \includegraphics[width=0.35\linewidth]{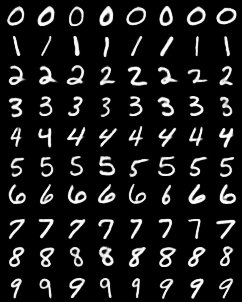}
      \hfill
      \includegraphics[width=0.62\linewidth]{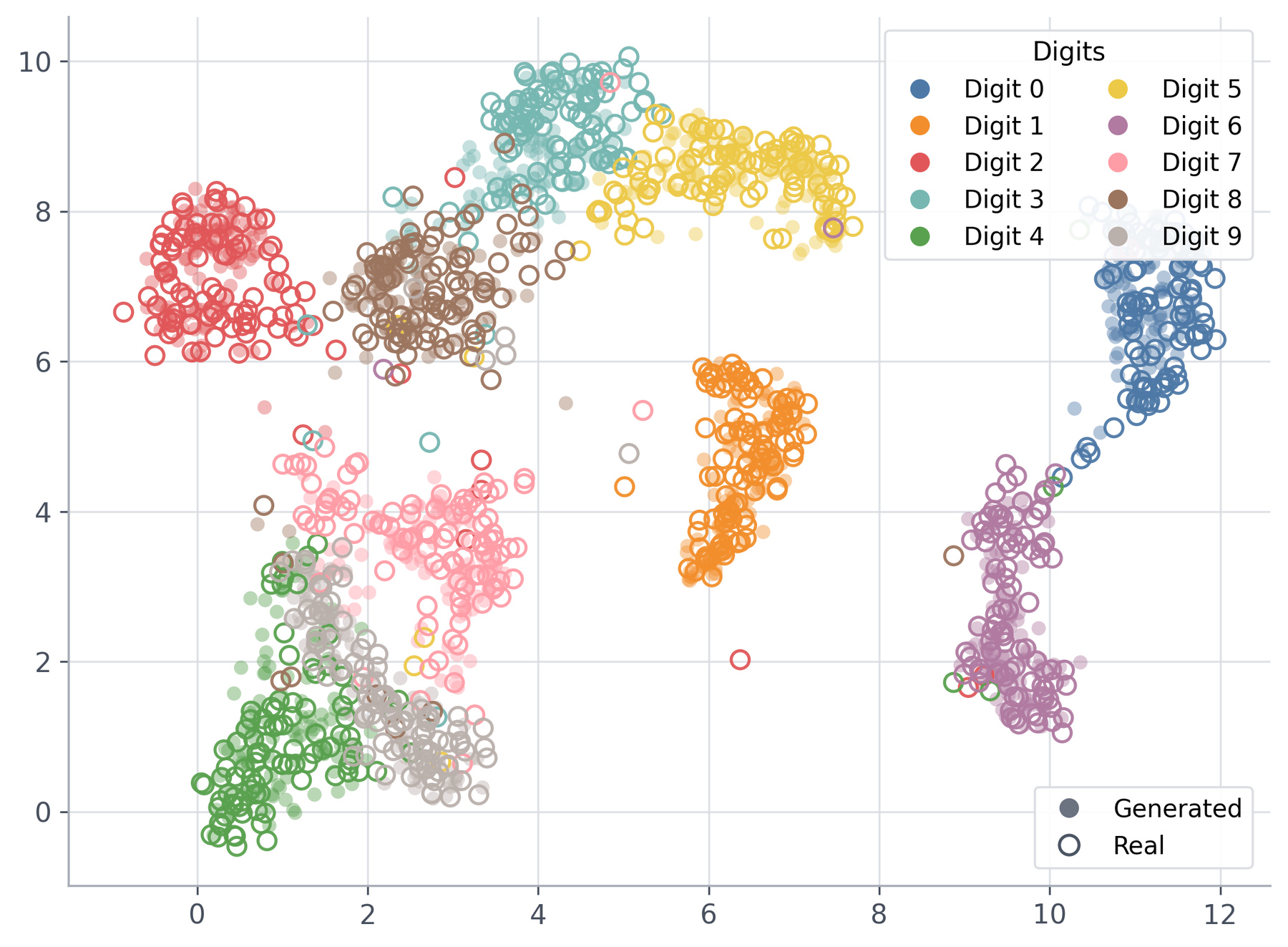}
    \end{minipage}
    &
    \begin{minipage}{0.33\linewidth}
      \centering
      \includegraphics[width=0.35\linewidth]{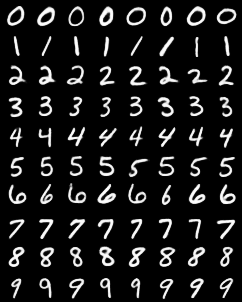}
      \hfill
      \includegraphics[width=0.62\linewidth]{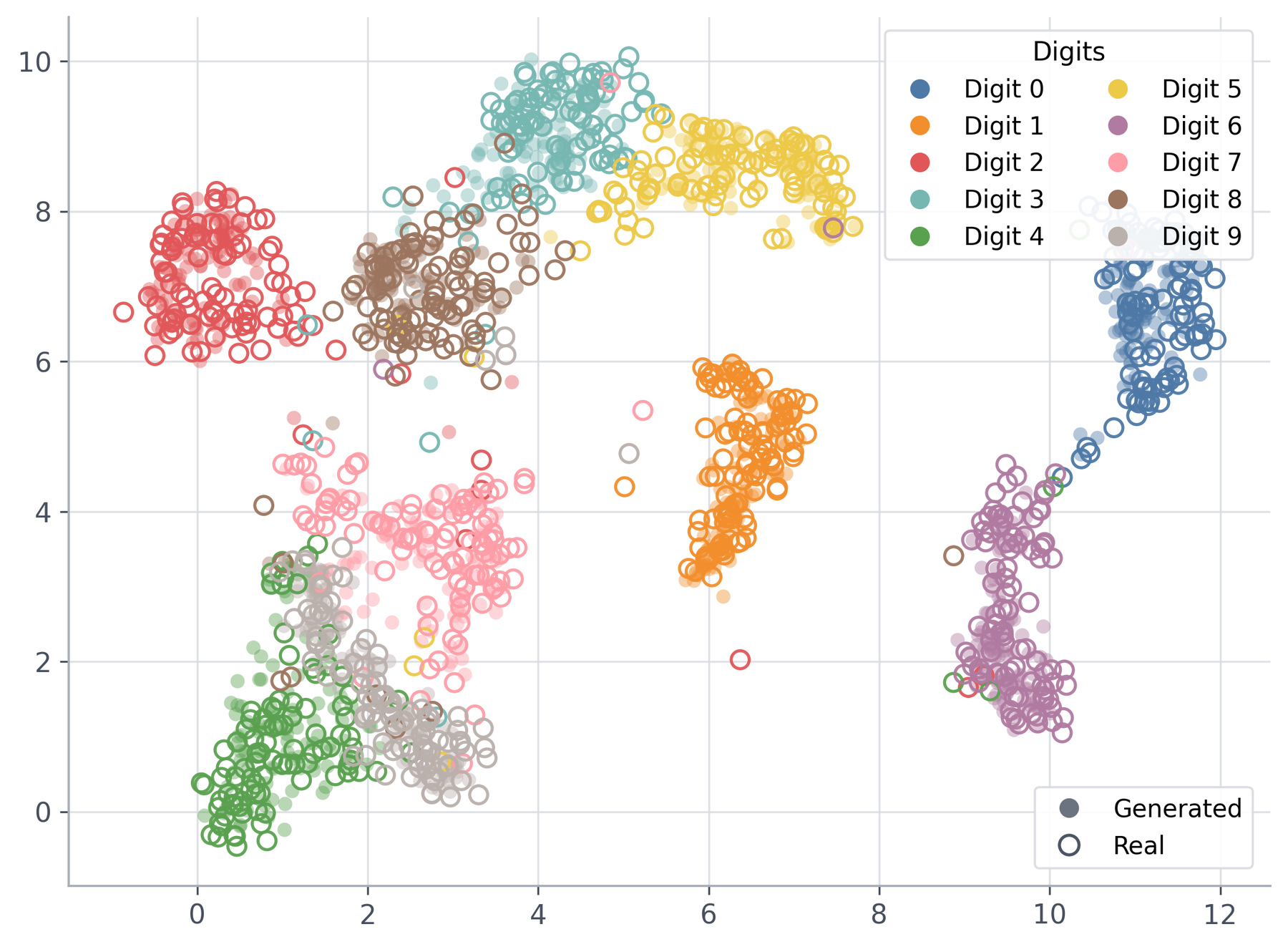}
    \end{minipage}
    &
    \begin{minipage}{0.33\linewidth}
      \centering
      \includegraphics[width=0.35\linewidth]{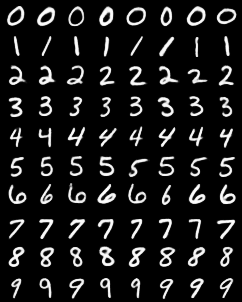}
      \hfill
      \includegraphics[width=0.62\linewidth]{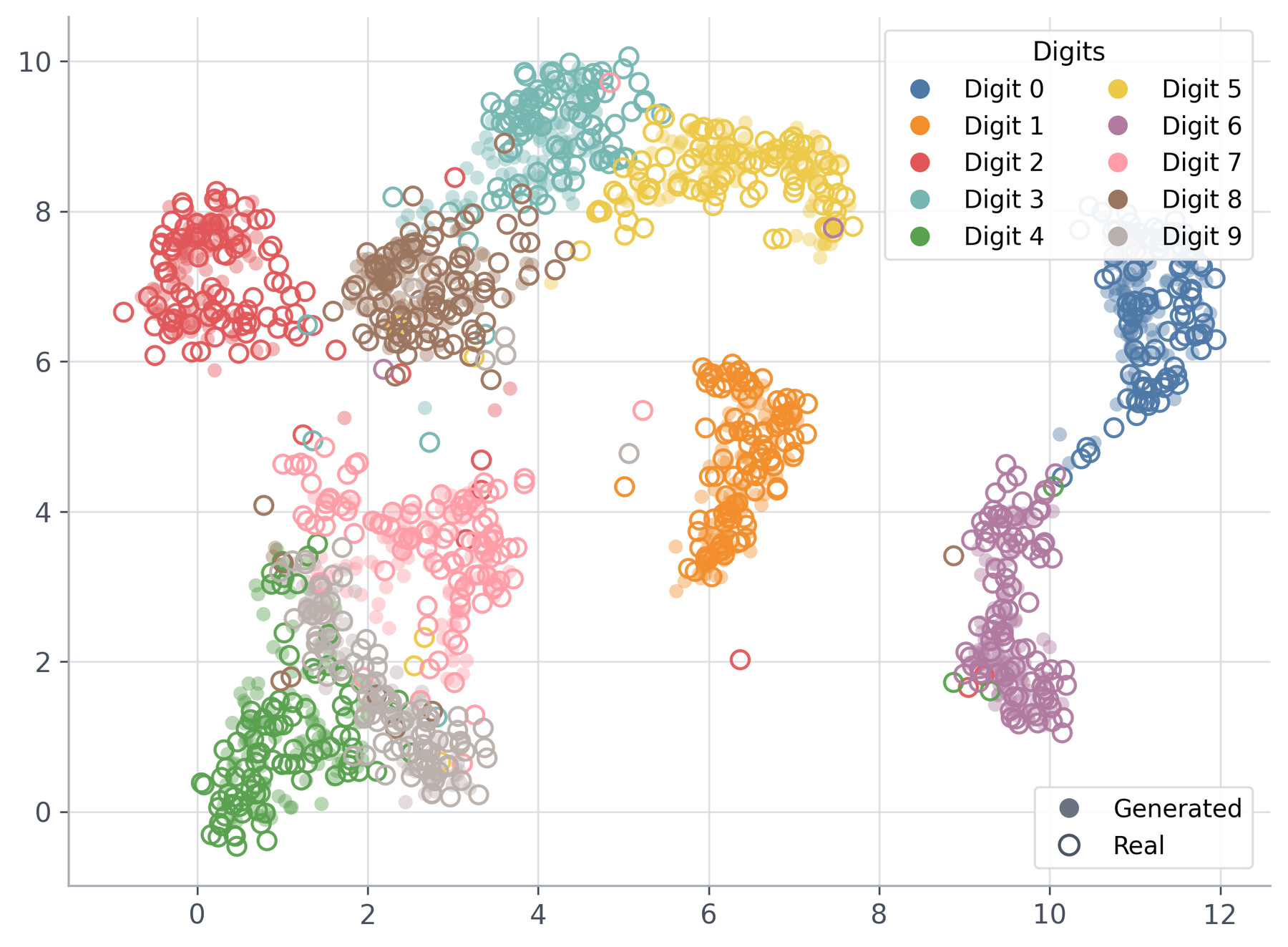}
    \end{minipage}
    \\[0.25em]

    {\scriptsize NFE = 1 (Drift Model~\cite{deng2026generative,he2026sinkhorn})\strut} &
    {\scriptsize NFE = 2\strut} &
    {\scriptsize NFE = 10\strut}
  \end{tabular}

  \vspace{-0.5em}
  \caption{\textbf{MNIST Generation and Latent Space Visualization under Different NFE.}
  Each column corresponds to a different inference step. Within each column, the left image shows the class grid and the right image shows the UMAP visualization~\cite{mcinnes2018umap}.}
  \label{fig:mnist_inference_steps}
  \vspace{-1em}
\end{figure*}

We evaluate class-conditional generation quality on the MNIST dataset~\cite{lecun2002gradient} and report the average per-class EMD, i.e., the squared 2-Wasserstein distance $W_2^2$, in latent space, together with class-conditional generation accuracy across different numbers of inference steps in Table~\ref{tab:mnist_ffhq_comparison}.
The 16-dimensional encoder~\cite{kingma2013auto} and decoder~\cite{kingma2013auto} are first trained on this dataset, after which the generation task is performed in the learned latent space.
Each generator is implemented as the same 3-layer MLP~\cite{rumelhart1986learning} and trained under the same settings.
Accuracy is computed using a classifier trained on the training split, while EMD is computed between the test split and the generated samples.
Generation results under different NFE, $[1,2,10]$, are shown in Figure~\ref{fig:mnist_inference_steps}, including class-wise generated digit grids and UMAP visualizations~\cite{mcinnes2018umap} of the latent representations.

\begin{figure*}[htbp]
  \centering
  \setlength{\tabcolsep}{0pt}
  \renewcommand{\arraystretch}{0.95}
\vspace{-1em}
  \begin{tabular}{cccc}
    \includegraphics[width=0.24\linewidth]{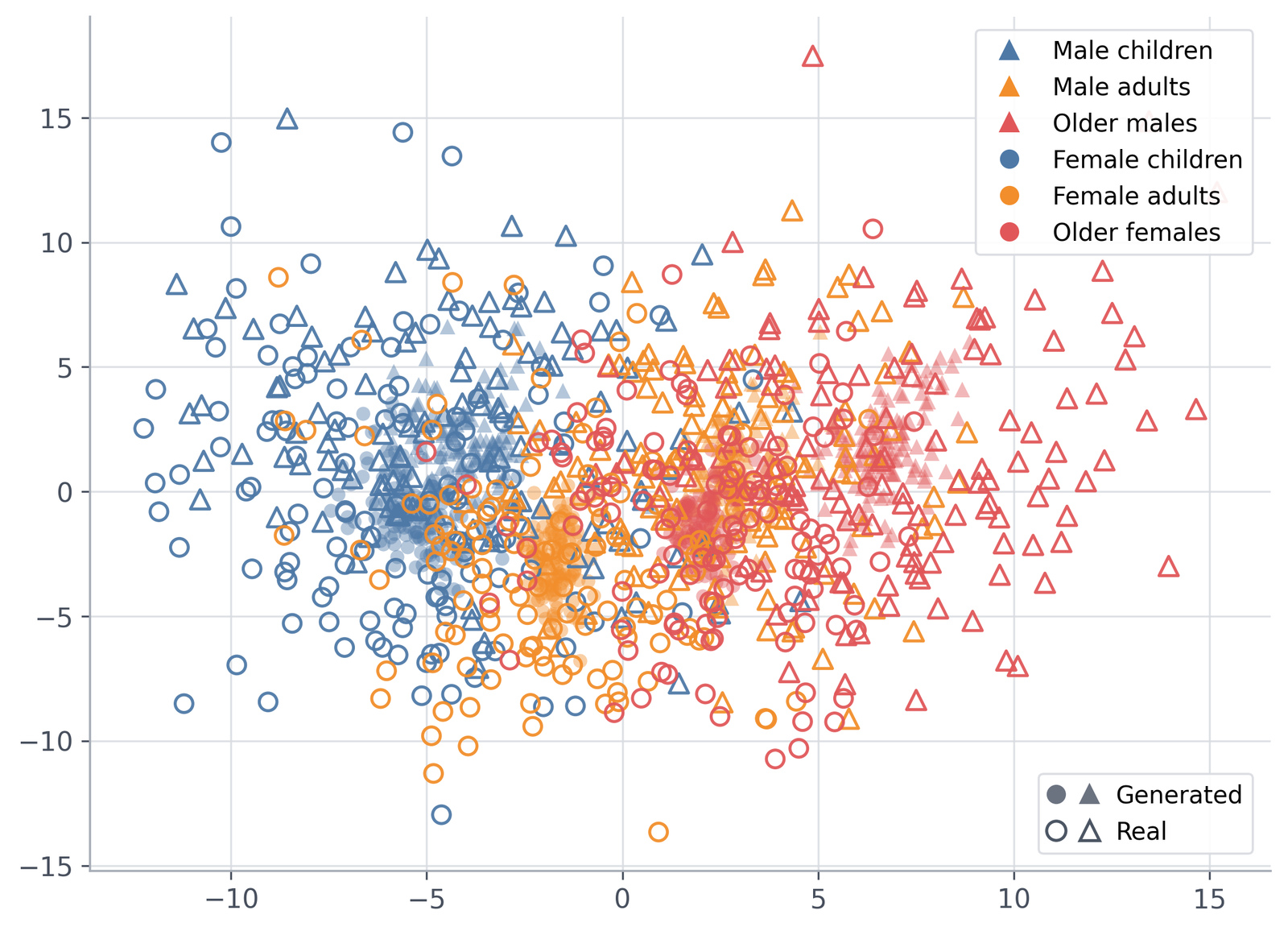} &
    \includegraphics[width=0.24\linewidth]{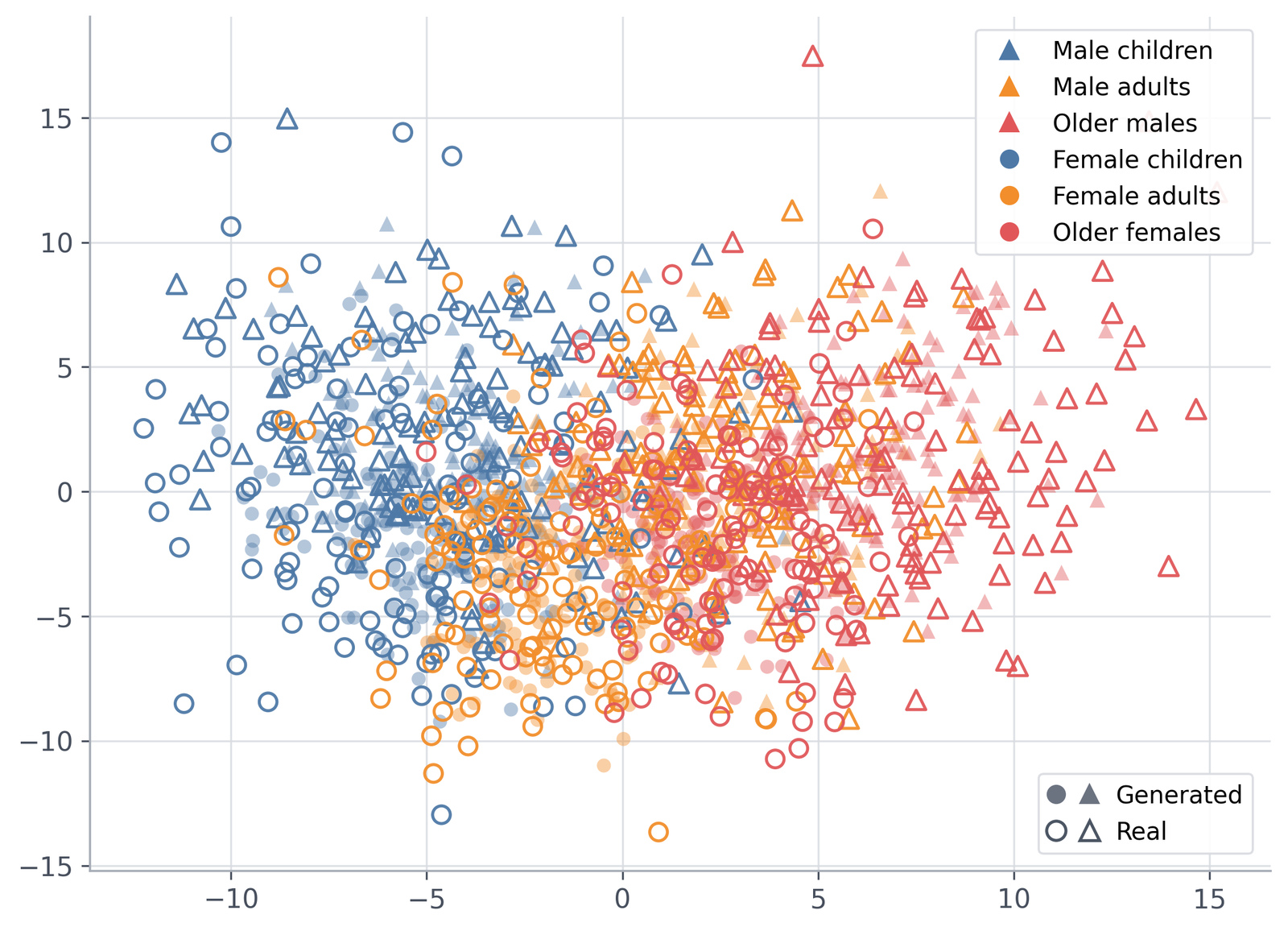} &
    \includegraphics[width=0.24\linewidth]{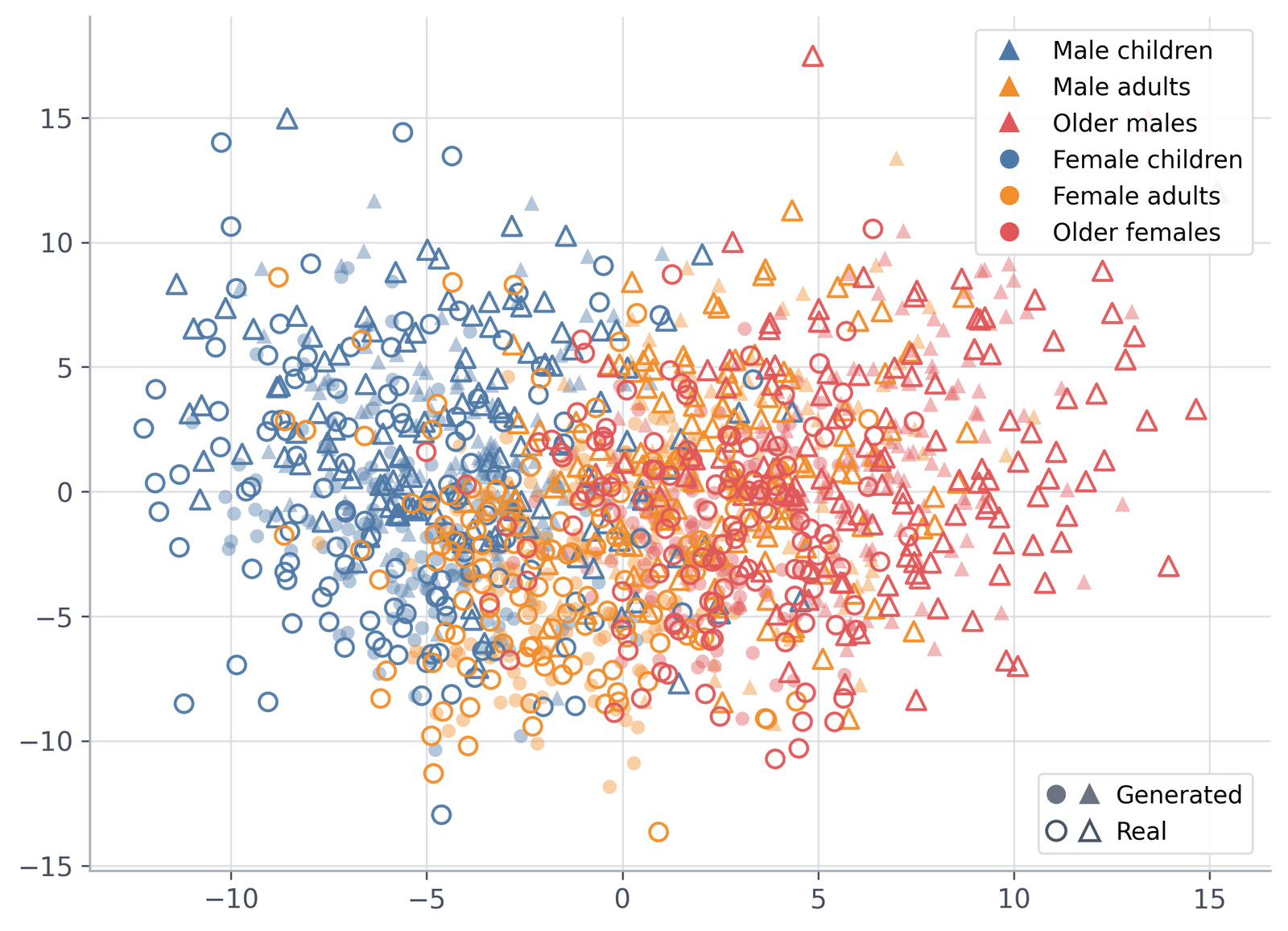} &
    \includegraphics[width=0.24\linewidth]{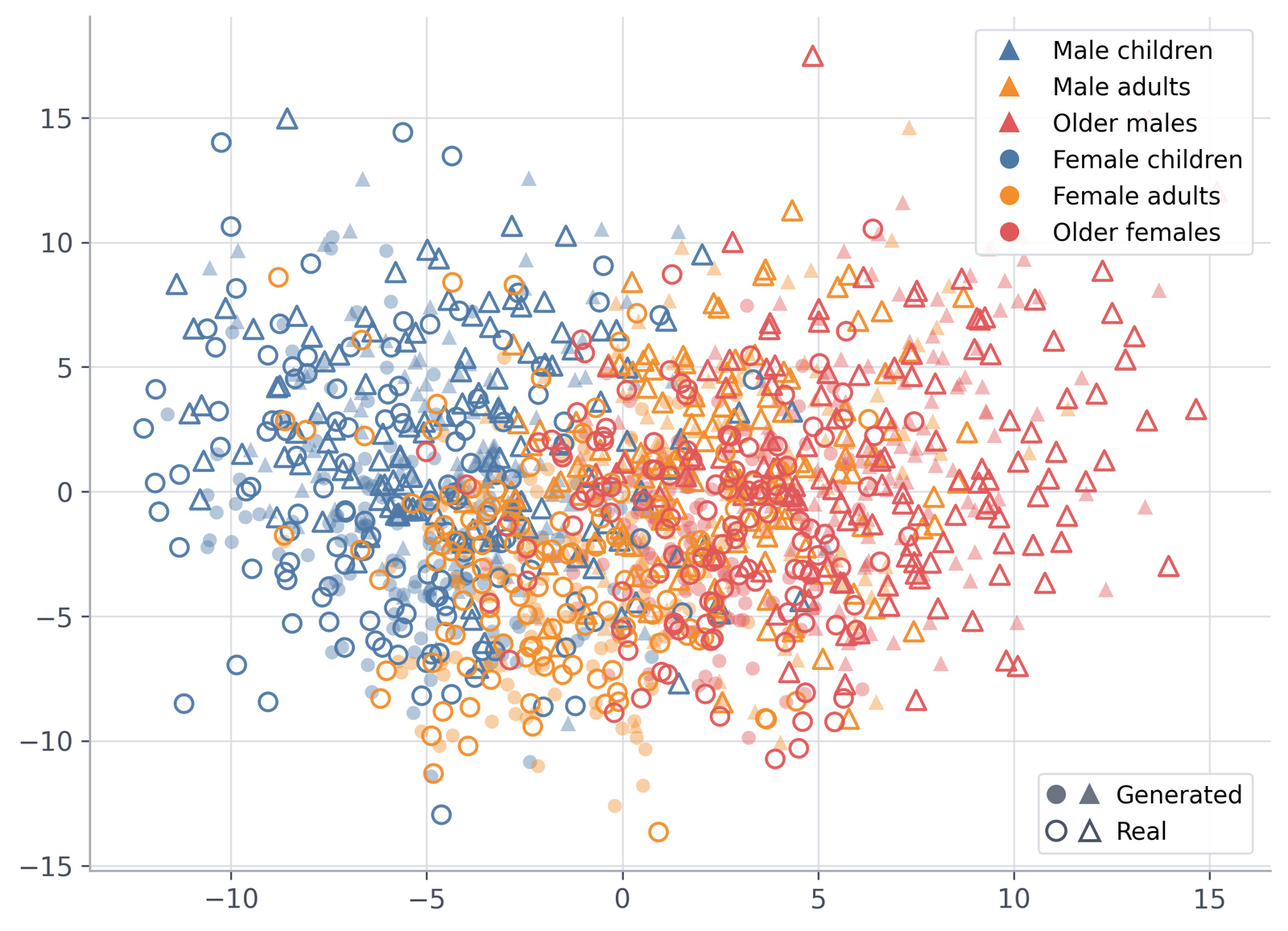} \\[-0.2em]

    \includegraphics[width=0.24\linewidth]{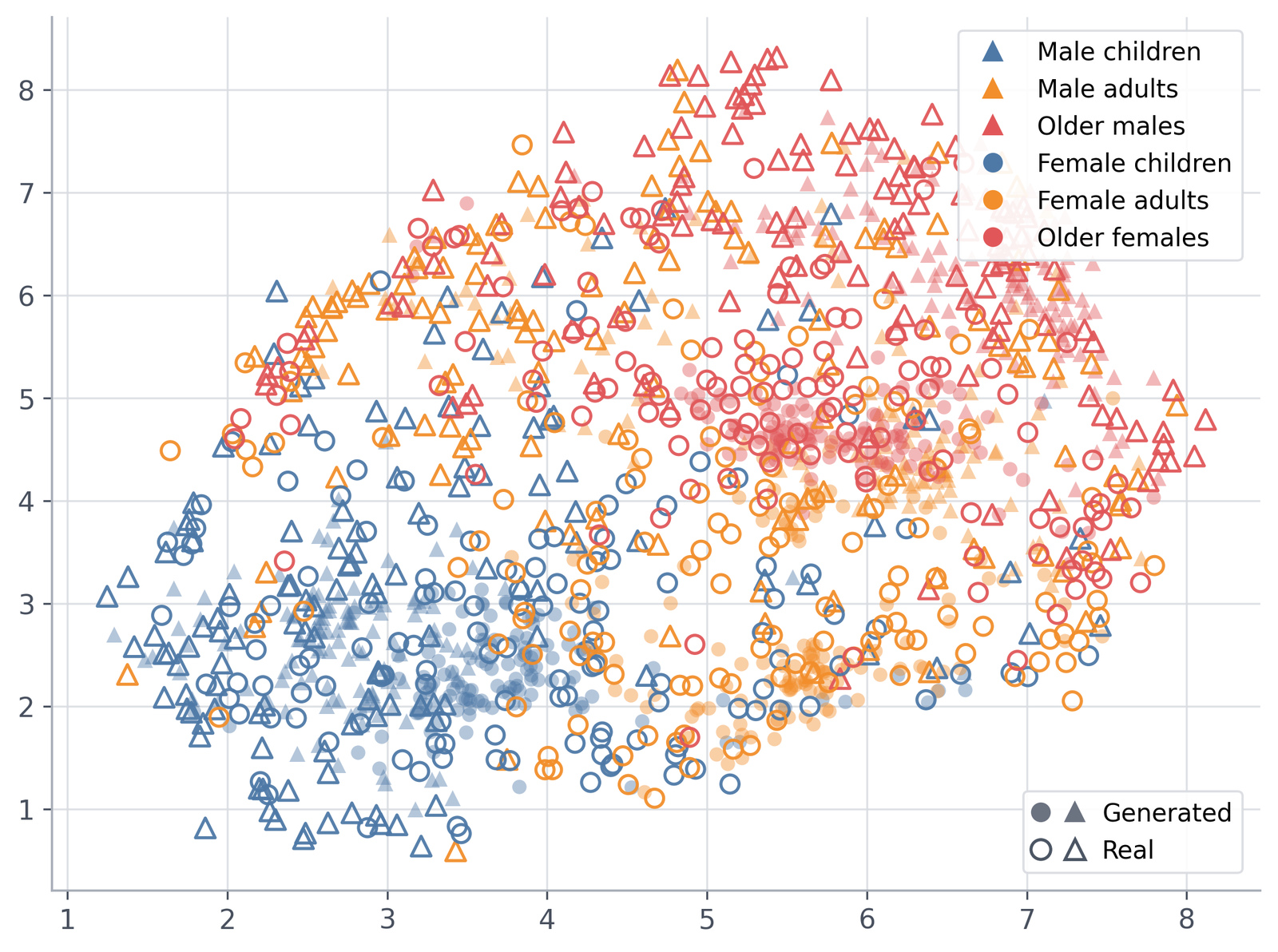} &
    \includegraphics[width=0.24\linewidth]{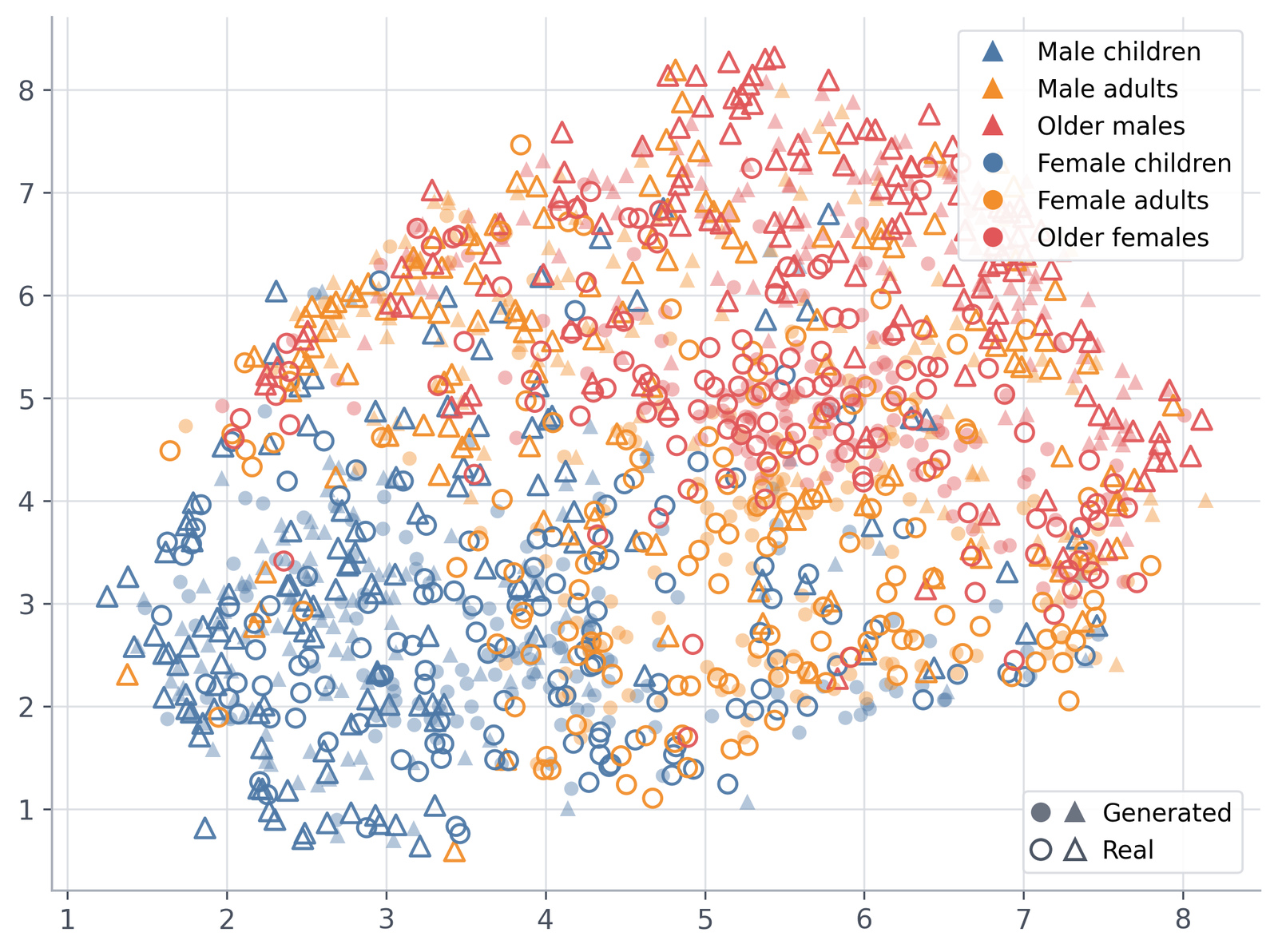} &
    \includegraphics[width=0.24\linewidth]{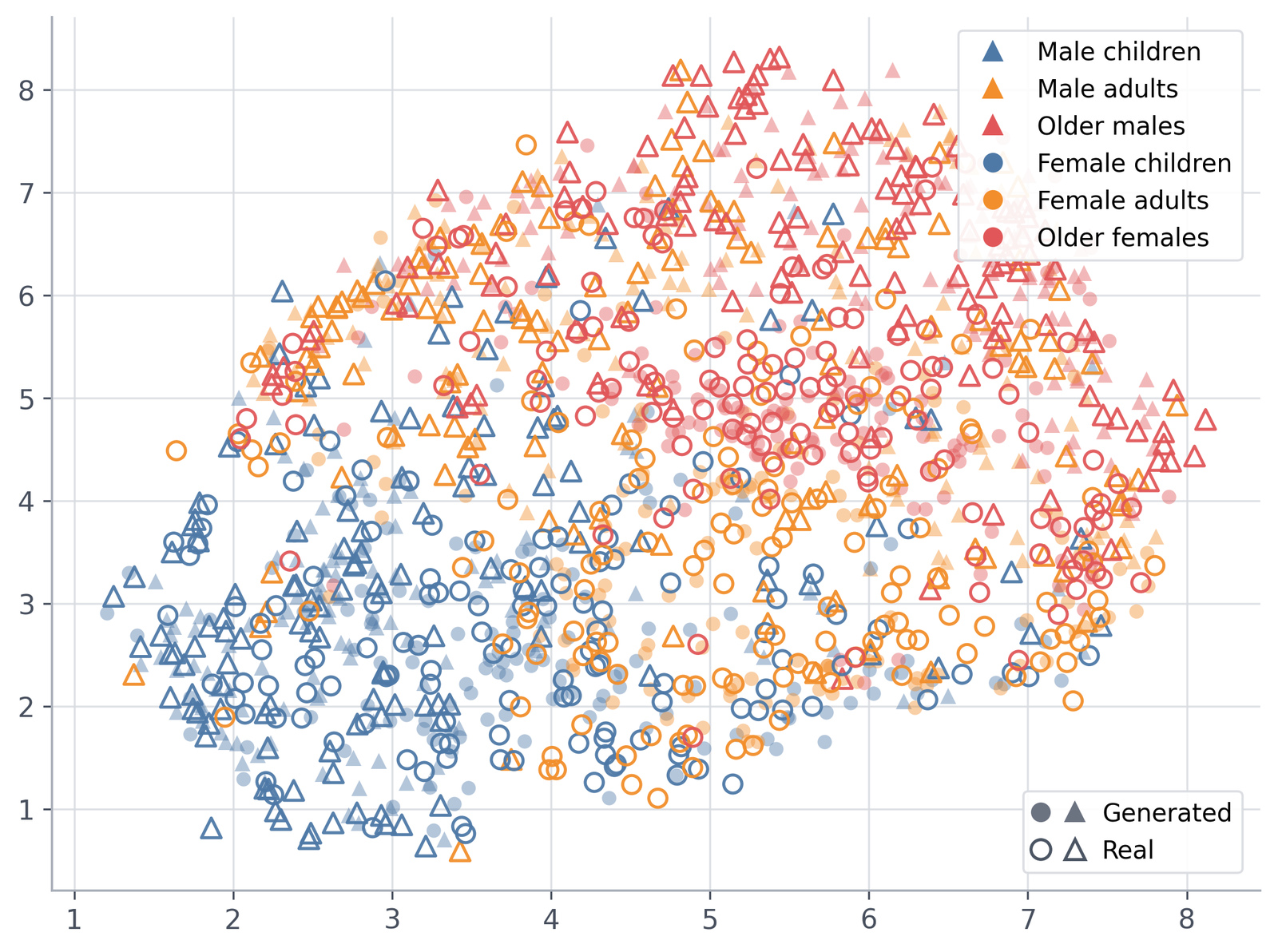} &
    \includegraphics[width=0.24\linewidth]{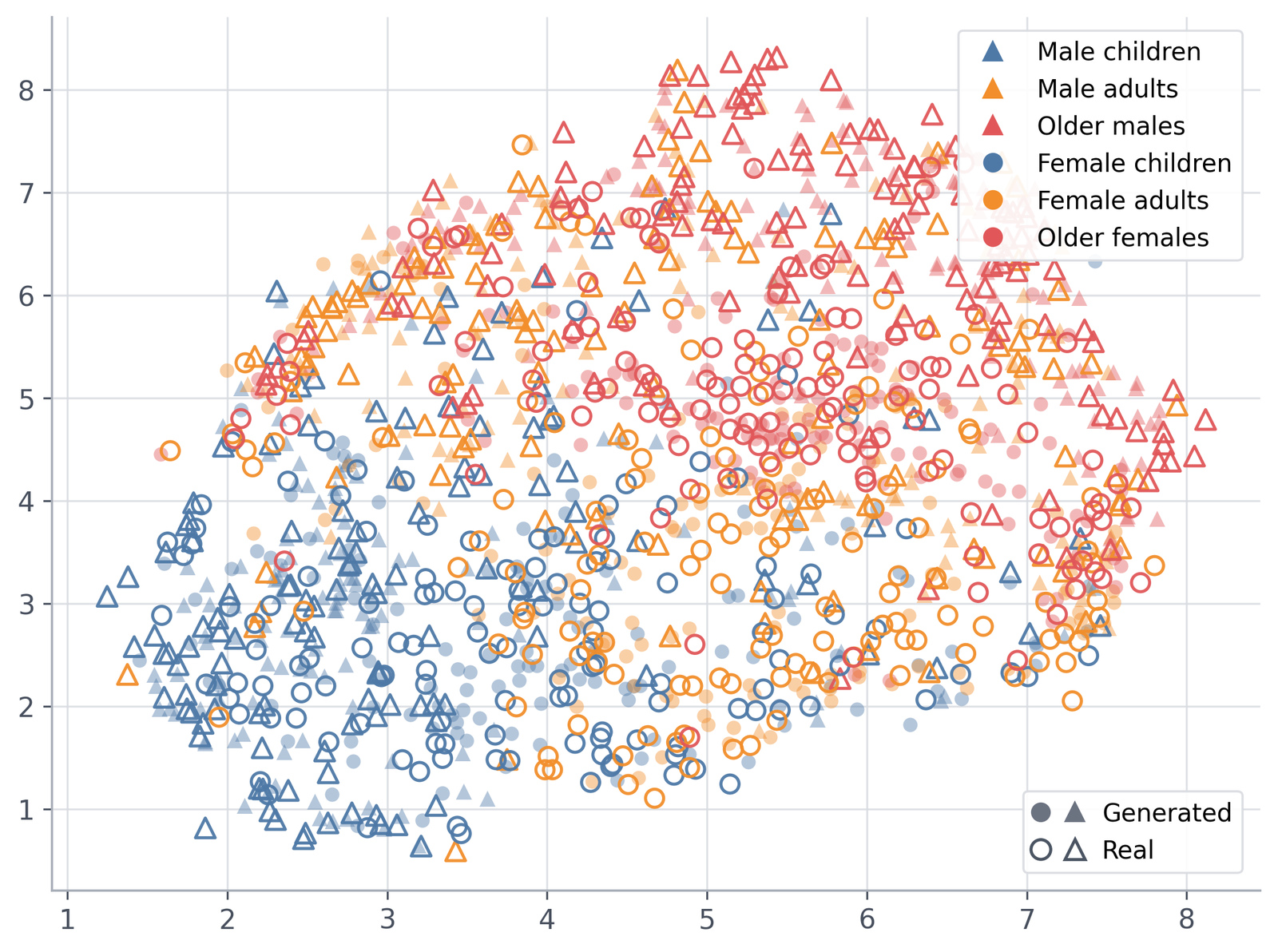} \\[0.25em]

    {\scriptsize NFE = 1 (Drift Model~\cite{deng2026generative,he2026sinkhorn})\strut} &
    {\scriptsize NFE = 2\strut} &
    {\scriptsize NFE = 5\strut} &
    {\scriptsize NFE = 10\strut}
  \end{tabular}

  \vspace{-0.5em}
  \caption{\textbf{FFHQ Latent Space Visualization under Different Inference Steps.}
  Each column corresponds to a different inference step. The top row shows PCA visualizations~\cite{pearson1901liii,hotelling1933analysis}, and the bottom row shows UMAP visualizations~\cite{mcinnes2018umap}.}
  \label{fig:ffhq_latent_inference_steps}
  \vspace{-0.5em}
\end{figure*}

\textbf{Evaluation for Conditional Generation (FFHQ).}
We evaluate class-conditional image generation on FFHQ~\cite{karras2019style} using a pre-trained Adversarial Latent Autoencoder (ALAE)~\cite{pidhorskyi2020adversarial}.
In our pipeline, we train and evaluate in the ALAE latent space ($\mathbb{R}^{512}$); image-space metrics are computed after decoding with the same frozen ALAE decoder at $1024 \times 1024$.
We consider six demographic classes: \{\textit{Male-Children}, \textit{Male-Adult}, \textit{Male-Old}, \textit{Female-Children}, \textit{Female-Adult}, \textit{Female-Old}\}, and train a conditional latent generator $T_{t,r}^\theta: \mathbb{R}^{512+64} \to \mathbb{R}^{512}$, implemented as a 3-layer MLP with hidden width 1024 and learned class embeddings.
We report latent EMD and image FID~\cite{heusel2017gans}, computed between decoded generated images and decoded real images, in Table~\ref{tab:mnist_ffhq_comparison}.
Generation results under different NFE are shown in Figure~\ref{fig:ffhq_latent_inference_steps}, including PCA~\cite{pearson1901liii,hotelling1933analysis} and UMAP~\cite{mcinnes2018umap} visualizations of the latent representations.

These results demonstrate the flexibility of DFM: even one-step generation, which recovers the Drift Model~\cite{he2026sinkhorn}, achieves strong performance compared with other baselines, while increasing the number of inference steps leads to clear improvements, i.e., test-time scaling similar to Flow Matching~\cite{liu2023flow,geng2025mean}.
Moreover, increasing the number of inference steps enhances the coverage of the generated data distribution, which is more evident in Figure~\ref{fig:ffhq_latent_inference_steps}.

\textbf{Visual Generation (ImageNet).}
To test robustness and scalability on large-scale data, we conduct experiments on the ImageNet-1k $256\times256$ dataset~\cite{krizhevsky2012imagenet}. 
Following prior works~\cite{deng2026generative,frans2025one}, all models are implemented in the latent space of a pre-trained VAE tokenizer~\cite{rombach2022highresolutionimagesynthesislatent}. 
For $256\times256$ images, the tokenizer maps each image to a latent representation of size $32\times32\times4$, which serves as the input to the generative model.
The DFM loss in Eq.~\eqref{eq:dfm_loss} is computed in the feature space of a latent-MAE encoder~\cite{he2022masked}, following the Drift Model~\cite{deng2026generative}. 
We adopt the Drift Model~\cite{deng2026generative} as the backbone architecture and follow the same training and evaluation protocol.
We evaluate Fr\'echet Inception Distance (FID)~\cite{heusel2017gans} and Inception Score (IS)~\cite{salimans2016improved} on 50K randomly generated images, as reported in Table~\ref{tab:image256}. 
The results show that our method achieves competitive performance when $\mathrm{NFE}=1$. 
Increasing the NFE further improves generation quality, demonstrating effective test-time scaling.

\textbf{Robotic Control.}
Beyond image generation, we further evaluate our method on robotic control~\cite{mandlekar2021matters,florence2022implicit,shafiullah2022behavior,gupta2019relay}. 
Our experimental design and protocols follow \textit{Diffusion Policy}~\cite{chi2025diffusion}, a multi-step diffusion-based generator, and \textit{Drift Policy}~\cite{deng2026generative}, a one-step drift-based generator, while replacing the core control policy with our \textit{Drift Flow Matching Policy}. 
We directly compute the DFM loss in Eq.~\eqref{eq:dfm_loss} on the raw representations for control tasks, without using an additional feature space~\cite{chi2025diffusion,deng2026generative}. 
The results are reported in Table~\ref{tab:robotics_control}. 
They demonstrate that our method preserves the one-step inference capability of \textit{Drift Policy} while exhibiting clear test-time scaling behavior
across diverse benchmarks~\cite{mandlekar2021matters,florence2022implicit,shafiullah2022behavior,gupta2019relay}.

\begin{table}[t]
\centering
\scriptsize
\resizebox{\textwidth}{!}{
\begin{tabular}{cllcccccc}
\toprule
& & & \multicolumn{1}{c}{Diffusion Policy}
& \multicolumn{1}{c}{Drift Policy}
& \multicolumn{4}{c}{\cellcolor{cvprblue!15}Drift Flow Matching Policy} \\
\cmidrule(lr){4-4}
\cmidrule(lr){5-5}
\cmidrule(lr){6-9}
& Task & Setting
& NFE: 100
& NFE: 1
& \cellcolor{cvprblue!15}NFE: 1
& \cellcolor{cvprblue!15}NFE: 2
& \cellcolor{cvprblue!15}NFE: 5
& \cellcolor{cvprblue!15}NFE: 10 \\
\midrule

\multirow{8}{*}{\rotatebox[origin=c]{90}{{Single-Stage}}}
& Lift
& State  & 0.98 & 1.00 & \cellcolor{cvprblue!15}1.00 & \cellcolor{cvprblue!15}1.00 & \cellcolor{cvprblue!15}1.00 & \cellcolor{cvprblue!15}1.00 \\
& 
& Visual & 1.00 & 1.00 & \cellcolor{cvprblue!15}1.00 & \cellcolor{cvprblue!15}1.00 & \cellcolor{cvprblue!15}1.00 & \cellcolor{cvprblue!15}1.00 \\

& Can
& State  & 0.96 & 0.98 & \cellcolor{cvprblue!15}0.98 & \cellcolor{cvprblue!15}1.00 & \cellcolor{cvprblue!15}1.00 & \cellcolor{cvprblue!15}1.00 \\
& 
& Visual & 0.97 & 0.99 & \cellcolor{cvprblue!15}0.99 & \cellcolor{cvprblue!15}1.00 & \cellcolor{cvprblue!15}1.00 & \cellcolor{cvprblue!15}1.00 \\

& ToolHang
& State  & 0.30 & 0.38 & \cellcolor{cvprblue!15}0.41 & \cellcolor{cvprblue!15}0.77 & \cellcolor{cvprblue!15}0.81 & \cellcolor{cvprblue!15}0.86 \\
& 
& Visual & 0.73 & 0.67 & \cellcolor{cvprblue!15}0.67 & \cellcolor{cvprblue!15}0.76 & \cellcolor{cvprblue!15}0.80 & \cellcolor{cvprblue!15}0.82 \\

& PushT
& State  & 0.91 & 0.86 & \cellcolor{cvprblue!15}0.86 & \cellcolor{cvprblue!15}0.91 & \cellcolor{cvprblue!15}0.93 & \cellcolor{cvprblue!15}0.94 \\
& 
& Visual & 0.84 & 0.86 & \cellcolor{cvprblue!15}0.86 & \cellcolor{cvprblue!15}0.87 & \cellcolor{cvprblue!15}0.89 & \cellcolor{cvprblue!15}0.90 \\

\midrule

\multirow{6}{*}{\rotatebox[origin=c]{90}{{Multi-Stage}}}
& BlockPush
& Phase 1 & 0.36 & 0.56 & \cellcolor{cvprblue!15}0.56 & \cellcolor{cvprblue!15}0.61 & \cellcolor{cvprblue!15}0.62 & \cellcolor{cvprblue!15}0.65 \\
& 
& Phase 2 & 0.11 & 0.16 & \cellcolor{cvprblue!15}0.16 & \cellcolor{cvprblue!15}0.24 & \cellcolor{cvprblue!15}0.29 & \cellcolor{cvprblue!15}0.33 \\

& Kitchen
& Phase 1 & 1.00 & 1.00 & \cellcolor{cvprblue!15}1.00 & \cellcolor{cvprblue!15}1.00 & \cellcolor{cvprblue!15}1.00 & \cellcolor{cvprblue!15}1.00 \\
& 
& Phase 2 & 1.00 & 1.00 & \cellcolor{cvprblue!15}1.00 & \cellcolor{cvprblue!15}1.00 & \cellcolor{cvprblue!15}1.00 & \cellcolor{cvprblue!15}1.00 \\
& 
& Phase 3 & 1.00 & 0.99 & \cellcolor{cvprblue!15}0.99 & \cellcolor{cvprblue!15}1.00 & \cellcolor{cvprblue!15}1.00 & \cellcolor{cvprblue!15}1.00 \\
& 
& Phase 4 & 0.99 & 0.96 & \cellcolor{cvprblue!15}0.98 & \cellcolor{cvprblue!15}0.98 & \cellcolor{cvprblue!15}1.00 & \cellcolor{cvprblue!15}1.00 \\

\bottomrule
\end{tabular}
}
\caption{
Robotics Control: Comparison with \textit{Diffusion Policy}~\citep{chi2025diffusion} and \textit{Drift Policy}~\cite{deng2026generative} under the same experimental setting and evaluation protocol, with only the core policy replaced by our \textit{Drift Flow Matching Policy}.
The table includes four single-stage tasks and two multi-stage tasks.
Success rates are reported as averages over the last 10 checkpoints.
}
\label{tab:robotics_control}
\vspace{-2em}
\end{table}

\textbf{Ablation Study.}
Finally, we conduct ablation experiments on several key variants and design choices, shown in Table~\ref{tab:ablation}, including the time-pair sampling distribution $(t,r)\sim\rho$, the kernel temperature $\tau$, the model parameterization, and the number of sampled time pairs $G$ and the number of positive/negative samples $n_g$ in training time. See Appendix~\S~\ref{app:ablation}.
\section{Conclusion} \label{sec:conclusion}

In this work, we introduce \textbf{Drift Flow Matching (DFM)} as a unified perspective that connects efficient one-step drifting generation with scalable multi-step flow-based generation.
DFM unlocks the test-time scaling capability of Drift Models, allowing generation to adapt to different quality--efficiency requirements.
This establishes a more flexible generative paradigm, where direct generation and iterative refinement are no longer separate design choices but can be viewed within a single framework.
Extensive experiments across different tasks and datasets demonstrate the effectiveness and generality of DFM.

\clearpage

\bibliographystyle{unsrtnat}
\bibliography{references}

\clearpage

\appendix
\addcontentsline{toc}{section}{Appendix} 

\noindent\rule{\textwidth}{0.8pt}
\startcontents[appendix]
\printcontents[appendix]{l}{1}{\section*{Appendix Contents}}
\noindent\rule{\textwidth}{0.8pt}

\section{Preliminary} \label{sec:appendix}

\subsection{Flow Matching} \label{proof:sec:fm}
\paragraph{Random variables and realizations.}
We work in $\mathbb{R}^d$. Uppercase letters (e.g., $X_t, Z$) denote random variables (RVs), and lowercase letters (e.g., $x_t, z$) denote their realizations (points/values). For a density (or probability law) of an RV $X_t$, we write $p_t(\cdot)$, and for a conditional density we write $p_{t\mid Z}(\cdot\mid z)$. Expectations are denoted by $\mathbb{E}[\cdot]$.

\paragraph{Source/target distributions and coupling.}
Let $X_0 \sim p_0$ be the \emph{source} distribution (e.g., standard Gaussian noise) and $X_1 \sim p_1$ be the \emph{target} distribution (e.g., images).
A generative model constructs a continuous path of distributions $\{p_t\}_{t\in[0,1]}$ that transports $p_0$ to $p_1$.
Let $(X_0,X_1)$ be any coupling on $\mathbb{R}^d$ with joint density $\pi$ whose marginals are $p_0$ and $p_1$ (not necessarily independent). In this work, we follow the standard Flow Matching setting, the \emph{source} and the \emph{target} distributions are independent: $\pi(x_0, x_1) = p_0(x_0) \, p_1(x_1)$.

\paragraph{Conditioning variable and conditional paths~\cite{lipman2023flow}.}
We use a conditioning RV $Z$ to index conditional paths.
Conditioned on $Z=z$, we obtain conditional coupling $(X_0^Z, X_1^Z)$ and a conditional path $\{X_t^Z\}_{t\in[0,1]}$ with conditional density $p_{t\mid Z}(\cdot\mid z)$.
The marginal path is $\{X_t\}_{t\in[0,1]}$, with density $p_t(\cdot)$, satisfying:
\begin{equation}
\label{proof:eq:bayes_xt}
p_t(x_t) \;=\; \int p_{t\mid Z}(x_t\mid z)\,p_Z(z) dz,
\qquad
X_t \sim p_t,\ \ X_t\mid(Z=z)\sim p_{t\mid Z}(\cdot\mid z).
\end{equation}

\paragraph{Interpolant.}
We consider a general interpolant between the \emph{marginal} endpoints $X_0$ and $X_1$,
specified by scalar functions $\alpha:[0,1]\to\mathbb{R}$ and $\beta:[0,1]\to\mathbb{R}$:
\begin{equation}
\label{proof:eq:generalInterpolant}
X_t \;=\; \alpha(t)\,X_0 + \beta(t)\,X_1,
\qquad t\in[0,1].
\end{equation}
We assume the boundary conditions $\alpha(0)=1,\ \beta(0)=0$ and $\alpha(1)=0,\ \beta(1)=1$, so that $X_{t=0}=X_0$ and $X_{t=1}=X_1$.
If $\alpha,\beta$ are differentiable, then
\begin{equation}
\label{proof:eq:generalInterpolant_dot}
\frac{d}{dt}X_t
\;=\;
\dot{\alpha}(t)\,X_0 + \dot{\beta}(t)\,X_1.
\end{equation}
Conditioned on $Z=z$, the same schedules induce the conditional path
$X_t^Z=\alpha(t)X_0^Z+\beta(t)X_1^Z$ and $\frac{d}{dt}X_t^Z=\dot{\alpha}(t)X_0^Z+\dot{\beta}(t)X_1^Z$.

\paragraph{Flow Matching vector fields~\cite{lipman2023flow}.}
Let $v(x_t,t\mid z)\in\mathbb{R}^d$ denote a \emph{conditional} velocity field that transports the conditional density $p_{t\mid Z}(\cdot\mid z)$ along time.
The corresponding \emph{marginal} velocity field is defined by conditional expectation:
\begin{equation}
\label{proof:eq:marginalV}
v(x_t,t)
\;=\;
\int v(x_t,t\mid z)\,p_{Z\mid t}(z\mid x_t)\,dz
\;=\;
\mathbb{E}\!\left[\,v(X_t,t\mid Z)\,\middle|\,X_t=x_t\,\right].
\end{equation}
A marginal trajectory, i.e., Flow Matching generation trajectory, follows the ODE:
\begin{equation}
\label{proof:eq:ode_marginal}
\frac{d x_t}{dt} = v(x_t,t), \qquad t\in[0,1],
\end{equation}
and similarly a conditional trajectory follows $\frac{d x_t^z}{dt}=v(x_t,t\mid z)$.

\paragraph{Continuity (transport) equations~\cite{lipman2023flow, lipman2024flowmatchingguidecode}.}
The evolution of densities induced by these velocity fields is characterized by the continuity (transport) equation.
For the marginal density $p_t$,
\begin{equation}
\label{proof:eq:continuity_marginal}
\partial_t p_t(x_t) + \nabla\!\cdot\!\big(p_t(x_t)\,v(x_t,t)\big)=0,
\qquad t\in[0,1].
\end{equation}
Conditioned on $Z=z$, the conditional density $p_{t\mid Z}(\cdot\mid z)$ satisfies
\begin{equation}
\label{proof:eq:continuity_conditional}
\partial_t p_{t\mid Z}(x_t\mid z) + \nabla\!\cdot\!\big(p_{t\mid Z}(x_t\mid z)\,v(x_t,t\mid z)\big)=0,
\qquad t\in[0,1].
\end{equation}
Flow Matching learns a parameterized velocity field (or equivalent dynamics) so that the induced marginal path $\{p_t\}$ solves Eq.\ \eqref{proof:eq:continuity_marginal} with boundary conditions $p_{t=0}=p_0$ and $p_{t=1}=p_1$.

\paragraph{Learning Flow Matching.}
Flow Matching introduce a velocity field model $v^\theta(X_t,t)$ to learn $v(X_t,t)$, ideally, by minimizing the marginal Flow Matching loss:
\begin{equation}
\label{proof:eq:MFM}
\mathcal{L}_\mathrm{MFM}(\theta) = \mathbb{E}_{t,\;X_t\sim p_t}\,
D\Big(v(X_t,t),\,v^\theta(X_t,t)\Big)
\end{equation}
However, since the marginal velocity $v(X_t,t)$ in Eq.\eqref{proof:eq:marginalV} is not tractable, so the marginal loss Eq.\eqref{proof:eq:MFM} above cannot be computed as is. Instead, we minimize the conditional Flow Matching loss:
\begin{equation}
\label{proof:eq:CFM}
\mathcal{L}_\mathrm{CFM}(\theta) = \mathbb{E}_{t,\;Z,\;X_t\sim p_{t\mid Z}(\cdot \mid Z)}\,
D\Big(v(X_t,t \mid Z),\,v^\theta(X_t,t)\Big)
\end{equation}
The two losses Eq.\eqref{proof:eq:MFM} and Eq.\eqref{proof:eq:CFM} are equivalent for learning purposes, since their gradients coincide:
\begin{theorem}[Gradient equivalence of Flow Matching~\cite{lipman2024flowmatchingguidecode}]\label{proof:thm:MequivC_FM}
The gradients of the marginal Flow Matching loss and the conditional Flow Matching loss coincide:
\begin{equation}
\label{proof:eq:marginalEqvconditional_FM}
\nabla_\theta \mathcal{L}_\mathrm{MFM}(\theta)
\;=\;
\nabla_\theta \mathcal{L}_\mathrm{CFM}(\theta)
\end{equation}
In particular, the minimizer of the conditional Flow Matching loss is the marginal velocity $v(x_t, t)$.
\end{theorem}

\begin{remark}[Standard Flow Matching] \label{proof:remark:FM}
Flow Matching sets the conditioning variable to be the endpoint pair
\[
Z \;=\; (X_0,X_1),
\]
so that conditioning on $Z=z$ fixes the endpoint pair $(X_0^Z,X_1^Z) = (X_0,X_1)$.
Choosing the linear schedules $\alpha(t)=1-t$ and $\beta(t)=t$ yields the conditional path
\[
X_t^Z \;=\; (1-t)X_0^Z + tX_1^Z \;=\; (1-t)X_0 + tX_1,\qquad 0\le t\le 1
\]
hence the time-derivative of this conditional path is constant:
\[
\frac{d}{dt}X_t^Z \;=\; X_1^Z - X_0^Z \;=\; X_1 - X_0.
\]
Therefore, for $Z=(X_0,X_1)$, the conditional velocity used as supervision in the conditional Flow Matching loss is a constant:
\[
v(X_t,t\mid Z) \;=\; X_1^Z - X_0^Z \;=\; X_1 - X_0,
\]

With this setup, the Flow Matching objective in Eq.~\eqref{proof:eq:CFM} can be written explicitly as
\[
\mathcal{L}_\mathrm{FM}(\theta)
\;=\;
\mathbb{E}_{t,\;X_0\sim p_0,\;X_1\sim p_1}\,
D\Big(X_1 - X_0,\;v^\theta(X_t,t)\Big),
\qquad
X_t=(1-t)X_0+tX_1,
\]
i.e., one samples $t\sim\mathrm{Unif}[0,1]$, draws independent endpoints $X_0\sim p_0$ and $X_1\sim p_1$,
forms the interpolated state $X_t$, and regresses the model $v^\theta(X_t,t)$ to the constant target $X_1-X_0$.
By Theorem~\ref{proof:thm:MequivC_FM}, minimizing this conditional loss yields the marginal velocity field $v(x_t,t)$ in Eq.~\eqref{proof:eq:marginalV}.
\end{remark}

\subsection{Drift Method}

\paragraph{Drift Velocity Field~\cite{deng2026generative,he2026sinkhorn}.}
Let $p$ denote a fixed target distribution on $\mathbb{R}^d$, and let $q$ denote the current model distribution. The Drift is constructed as a velocity field
$V_{q,p} : \mathbb{R}^d \to \mathbb{R}^d$:
\begin{equation}
V_{q,p}(x) = V_p^{+}(x) - V_q^{-}(x),
\label{proof:eq:drift-velocity-field}
\end{equation}
where each term is a kernel-weighted (and normalized) average of displacement vectors~\cite{deng2026generative,he2026sinkhorn,lai2026unified}. Concretely, given a positive kernel
$k : \mathbb{R}^d \times \mathbb{R}^d \to \mathbb{R}_{+}$
and the normalizers
\[
Z_p(x) := \int k(x,y)\,dp(y), \qquad
Z_q(x) := \int k(x,y)\,dq(y),
\]
the Drift update is
\begin{equation}
V_p^{+}(x) := \frac{1}{Z_p(x)} \int k(x,y)(y-x)\,dp(y),
\qquad
V_q^{-}(x) := \frac{1}{Z_q(x)} \int k(x,y)(y-x)\,dq(y).
\label{proof:eq:drift-positive-negative-components}
\end{equation}

Intuitively, $V_p^{+}(x)$ attracts $x$ toward the target $p$, while $V_q^{-}(x)$ subtracts an analogous attraction toward the current model $q$, producing a repulsive (self-correction) effect~\cite{deng2026generative,he2026sinkhorn}.

We use the Gibbs kernel
$
k(x,y)=e^{-\frac{C(x,y)}{\tau}}
$\footnote{The kernel used in \cite{deng2026generative} is $e^{-\frac{\lVert x-y\rVert}{\tau}}$, whereas the kernel used in \cite{he2026sinkhorn} is $e^{-\frac{\frac{1}{2}\lVert x-y\rVert^2}{\tau}}$.}
where $C(x,y)$ is a cost function between $x$ and $y$; by default, we set
$
C(x,y)=\frac{1}{2}\lVert x-y\rVert^2
$
and use constant temperature $\tau>0$.

In the discrete setting, let
$
p=\frac{1}{n}\sum_{j=1}^n \delta_{y_j}
$ and $
q=\frac{1}{n}\sum_{i=1}^n \delta_{x_i}
$
be empirical measures with samples
$Y=\{y_j\}_{j=1}^n$
and
$X=\{x_i\}_{i=1}^n$.
Then
\begin{equation}
V_{q,p}(x_i)
=
\sum_{j=1}^n P_{XY}^{\mathrm{drift}}[i,j]\,y_j
-
\sum_{j=1}^n P_{XX}^{\mathrm{drift}}[i,j]\,x_j,
\label{proof:eq:discrete-drift-update}
\end{equation}
where
\[
P_{XY}^{\mathrm{drift}}[i,j]
:=
\frac{k(x_i,y_j)}{\sum_{l=1}^n k(x_i,y_l)},
\qquad
P_{XX}^{\mathrm{drift}}[i,j]
:=
\frac{k(x_i,x_j)}{\sum_{l=1}^n k(x_i,x_l)}.
\]

\paragraph{Drift generative model.}
Drift generative models define a one-step generative model:
\begin{equation}
f_\theta : \mathbb{R}^d \to \mathbb{R}^d,
\qquad
\epsilon \mapsto x_\theta := f_\theta(\epsilon),
\label{proof:eq:one-step-generative-model}
\end{equation}
where $\epsilon \sim p_\epsilon$ for some prior distribution $p_\epsilon$. The drift flow is the probability path
$
q_\theta := (f_\theta)_{\#}p_\epsilon
$
generated by $V_{q_\theta,p}$~\cite{lai2026unified}:
\begin{equation}
\dot{x}_\theta = V_{q_\theta,p}(x_\theta),
\qquad
\forall x_\theta = f_\theta(\epsilon),
\label{proof:eq:drift-flow-ode}
\end{equation}
which is approximated by its forward Euler process in the time discretization scheme~\cite{lai2026unified}:
\begin{equation}
x_{\theta_{k+1}} = x_{\theta_k} + V_{q_\theta,p}(x_{\theta_k}).
\label{proof:eq:forward-euler-drift}
\end{equation}

The training loss is then set to be
\begin{equation}
\mathcal{L}^{\mathrm{drift}}
:=
\frac{1}{2}\mathbb{E}_{\epsilon}
\left[
\left\|
f_\theta(\epsilon)
-
\operatorname{sg}\!\left(
f_\theta(\epsilon)+V_{q_\theta,p_{\mathrm{data}}}(f_\theta(\epsilon))
\right)
\right\|^2
\right],
\label{proof:eq:drift-training-loss}
\end{equation}
where $\operatorname{sg}$ is the stop-gradient operator, and the optimization scheme is based on gradient descent:
\begin{equation}
\theta \leftarrow \theta - \eta \nabla_\theta \mathcal{L}^{\mathrm{drift}},
\qquad
\nabla_\theta \mathcal{L}^{\mathrm{drift}}
=
-\mathbb{E}_{\epsilon}
\left[
J_f(\theta,\epsilon)^\top
V_{q_\theta,p_{\mathrm{data}}}(f_\theta(\epsilon))
\right],
\label{proof:eq:drift-gradient-update}
\end{equation}
where $\eta$ is the learning rate, and $J_{f}(\theta,\epsilon) \in \mathbb{R}^{d \times \dim(\theta)}$ denotes the Jacobian. It can be verified that the above parameter update~\eqref{proof:eq:drift-gradient-update} induces the drift flow~\eqref{proof:eq:drift-flow-ode}. Proof and related details are shown in \cite{he2026sinkhorn}.

\begin{remark}[Identifiability of the drift objective]
\label{proof:remark:drift_identifiability}
The drifting objective Eq.~\eqref{proof:eq:drift-training-loss} minimizes the magnitude of the drift field, i.e., $\|V_{q,p}\|^2$.
By construction, the equilibrium condition $q=p$ implies $V_{q,p}(x)=0$ for all $x$~\cite{deng2026generative,he2026sinkhorn,lai2026unified}.
However, the converse implication does not hold in general and depends on the choice of kernel and the distribution family.
Under suitable non-degeneracy conditions (e.g., sufficiently expressive kernels or feature representations), the condition $V_{q,p} \approx 0$ can serve as a useful surrogate for distribution matching, in the sense that it encourages $q \approx p$. This heuristic is supported both empirically and by recent theoretical analyses of drifting-based methods~\cite{deng2026generative,he2026sinkhorn,lai2026unified}.
We note that stronger drift constructions (e.g., those derived from optimal transport objectives) can improve identifiability by making the zero-drift condition more directly correspond to distribution matching~\cite{he2026sinkhorn,lai2026unified}. Our formulation is compatible with such improvements.
\end{remark}

\section{Relation to Prior Work}
\label{app:relation2prior}

\begin{figure}[htbp]
    \centering
    \includegraphics[width=1.0\linewidth]{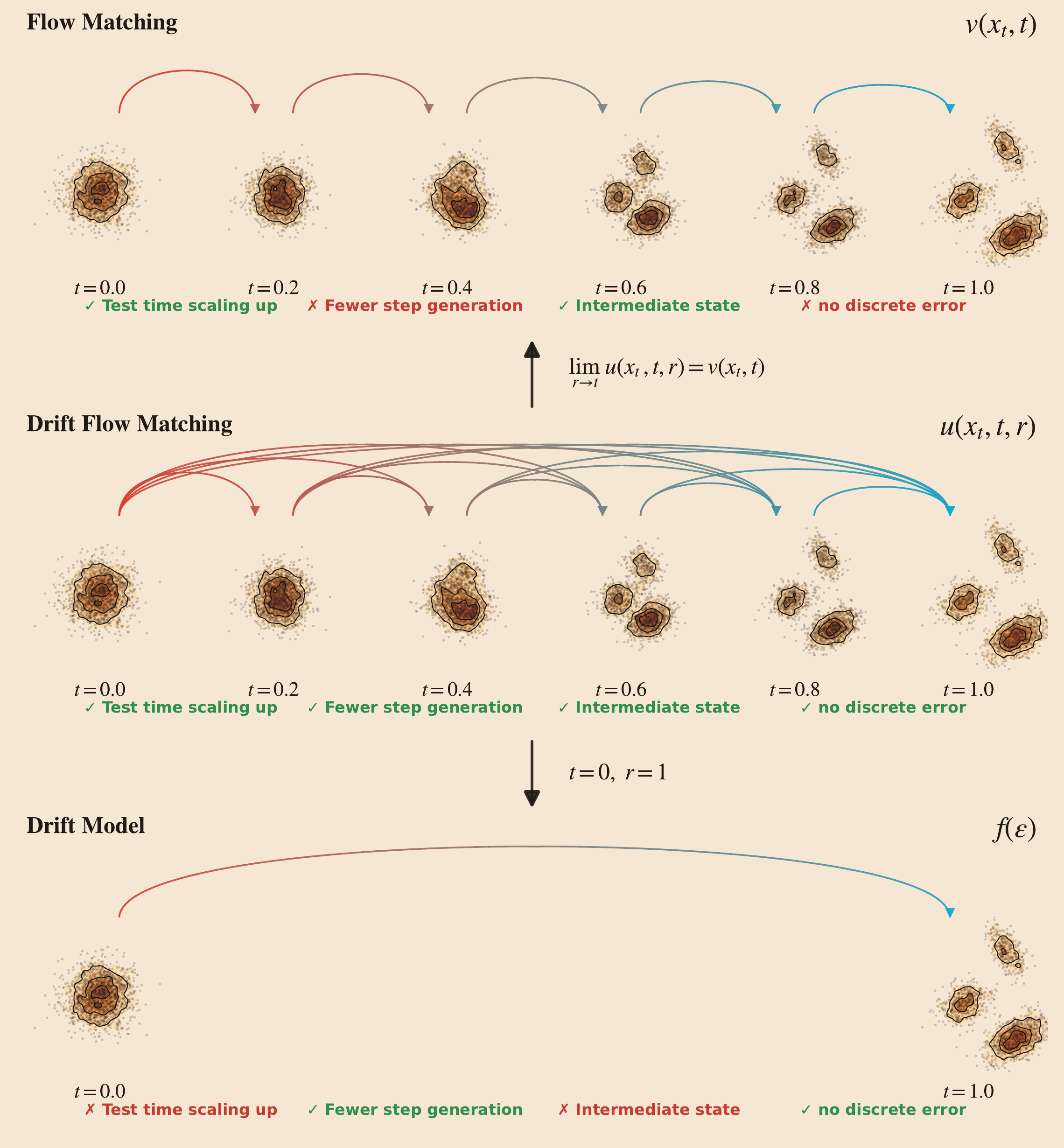}
    \caption{Drift Flow Matching Comparing with Flow Matching and Drift Model.}
    \label{fig:illustration}
\end{figure}

\paragraph{Flow Matching.}
Flow Matching~\cite{lipman2023flow,lipman2024flowmatchingguidecode} learns the marginal velocity field $v(x_t, t)$ in Eq.~\eqref{proof:eq:ode_marginal}, which requires simulating the dynamics through multiple small steps to obtain the full generation trajectory.
It constructs a conditional path to derive a tractable conditional velocity target, and recovers the marginal dynamics via Theorem~\ref{proof:thm:MequivC_FM}.
\textbf{DFM} retains the same path construction, but changes the supervision mechanism. 
Instead of regressing a local velocity field, \textbf{DFM} uses the conditional path to construct tractable samples from marginal pairs $(p_t, p_r)$, and directly learns the transport between them through a drift-based objective.
In this sense, \textbf{DFM} shifts the role of the conditional--marginal relationship from \emph{velocity regression} to \emph{distribution-level transport}. 
Importantly, \textbf{DFM} remains consistent with Flow Matching in the infinitesimal-step limit, where it recovers the standard velocity-based formulation.

\paragraph{Drift Model.}
Drift models~\cite{deng2026generative,he2026sinkhorn,lai2026unified} learn an evolving mapping $f(\epsilon)$ in Eq.~\eqref{proof:eq:one-step-generative-model}, directly transporting samples from a source distribution to a target distribution, often in a single step. 
While this enables efficient generation, it limits flexibility in trajectory control and test-time scaling.
\textbf{DFM} generalizes this paradigm by introducing a continuum of intermediate distributions $\{p_t\}_{t \in [0,1]}$ and learning transports between arbitrary pairs $(p_t, p_r)$. 
This extends drift-based modeling from a fixed source--target mapping to a flexible, two-time transport mechanism.
As a result, \textbf{DFM} combines the efficiency of drift models with the flexibility of multi-step generation, enabling controllable trade-offs between generation speed and sample quality.

\paragraph{Mean Velocity Flow Matching.}
Mean Velocity Flow Matching~\cite{geng2025mean,geng2025improved,lu2026one,ma2026transition} establishes a relationship between the instantaneous velocity $v(x_t, t)$ in Flow Matching~\cite{lipman2023flow,lipman2024flowmatchingguidecode} and a mean velocity target $u(x_t,t,r)$ pointing toward a future time $r$. 
It optimizes an objective derived from this relationship, effectively linking local dynamics to future transport behavior.
However, such approaches typically rely on expensive Jacobian-vector product (JVP) computations, and are often implemented in a \emph{self-distillation} manner, which requires careful design and does not always guarantee stable or improved performance.
In contrast, the model parameterization of \textbf{DFM} shown in Eq.~\eqref{eq:dfm_meanV} and Eq.~\eqref{eq:dfm_transport_map} naturally admits an interpretation as a mean velocity field between two time points. 
Rather than enforcing consistency through auxiliary objectives, \textbf{DFM} directly learns this two-time transport via a drift-based objective over marginal pairs, leading to a simpler and more direct formulation.

\paragraph{Figure~\ref{fig:illustration} illustrates the relationship between \textbf{DFM} and prior work.}

\section{Ablation Study}
\label{app:ablation}

To provide a systematic analysis of the proposed method, we conduct ablation experiments on several key variants and design choices, shown in Table~\ref{tab:ablation}, including the time-pair sampling distribution $(t,r)\sim\rho$, the kernel temperature $\tau$, the model parameterization, and the number of sampled time pairs $G$ and the number of positive/negative samples $n_g$ in training time.

Across all ablation experiments, we adopt the Drift Model~\cite{deng2026generative} (B/2) as the backbone and train it for 100 epochs, following the same training and evaluation protocol.

\begin{table}[htbp]
\centering
\caption{Ablation results for each key design. The number reported in the table is the FID score. Lower FID is better.}
\label{tab:ablation}
\renewcommand{\arraystretch}{1.0}
\setlength{\tabcolsep}{5pt}
\scriptsize

\begin{subtable}{0.78\linewidth}
\centering
\caption{Time embedding.}
\label{tab:ablation_timeEmbedding}
\begin{tabular}{lcccc}
\toprule
\textbf{time embed} & \textbf{NFE=1} & \textbf{NFE=2} & \textbf{NFE=10} & \textbf{NFE=50} \\
\midrule
$(t, r)$       & 29.73 & 27.84 & 23.91 & 21.36 \\
$(t, t-r)$     & \cellcolor{gray!20}\textbf{29.06} & \cellcolor{gray!20}\textbf{26.97} & \cellcolor{gray!20}\textbf{23.48} & \cellcolor{gray!20}\textbf{20.82} \\
$(t, r, t-r)$  & 31.93 & 29.22 & 25.93 & 22.71 \\
\bottomrule
\end{tabular}
\end{subtable}

\vspace{0em}

\begin{subtable}{0.78\linewidth}
\centering
\caption{Time samplers.}
\label{tab:ablation_timeSampler}
\begin{tabular}{lcccc}
\toprule
\textbf{$t,r$ sampler} & \textbf{NFE=1} & \textbf{NFE=2} & \textbf{NFE=10} & \textbf{NFE=50} \\
\midrule
$\mathrm{uniform}(0, 1)$        & 35.90 & 33.20 & 28.31 & 25.07 \\
$\mathrm{lognorm}(-0.2, 1.0)$   & 33.83 & 31.92 & 27.15 & 24.01 \\
$\mathrm{lognorm}(-0.2, 1.2)$   & 34.72 & 32.74 & 28.03 & 24.88 \\
$\mathrm{lognorm}(-0.2, 1.4)$   & 35.36 & 33.28 & 28.64 & 25.31 \\
$\mathrm{lognorm}(-0.4, 1.0)$   & \cellcolor{gray!20}\textbf{29.06} & \cellcolor{gray!20}\textbf{26.97} & \cellcolor{gray!20}\textbf{23.48} & \cellcolor{gray!20}\textbf{20.82} \\
$\mathrm{lognorm}(-0.4, 1.2)$   & 29.79 & 27.81 & 24.22 & 21.30 \\
$\mathrm{lognorm}(-0.4, 1.4)$   & 31.18 & 29.03 & 25.37 & 22.46 \\
$\mathrm{lognorm}(-0.6, 1.0)$   & 31.64 & 29.52 & 25.71 & 22.83 \\
$\mathrm{lognorm}(-0.6, 1.2)$   & 32.08 & 30.01 & 25.96 & 23.12 \\
$\mathrm{lognorm}(-0.6, 1.4)$   & 32.46 & 30.35 & 26.14 & 23.37 \\
\bottomrule
\end{tabular}
\end{subtable}

\vspace{0em}

\begin{subtable}{0.78\linewidth}
\centering
\caption{Kernel temperature $\tau$.}
\label{tab:ablation_temp}
\begin{tabular}{lcccc}
\toprule
\textbf{$\tau$} & \textbf{NFE=1} & \textbf{NFE=2} & \textbf{NFE=10} & \textbf{NFE=50} \\
\midrule
$0.01$ & 36.24 & 34.10 & 29.90 & 26.42 \\
$0.05$ & 34.37 & 32.15 & 27.83 & 24.94 \\
$0.1$  & 33.21 & 31.08 & 26.47 & 23.58 \\
$0.5$  & 29.84 & 27.73 & 24.05 & 21.36 \\
$1.0$  & \cellcolor{gray!20}\textbf{29.06} & \cellcolor{gray!20}\textbf{26.97} & \cellcolor{gray!20}\textbf{23.48} & \cellcolor{gray!20}\textbf{20.82} \\
$10.0$ & 32.48 & 30.39 & 25.76 & 22.03 \\
\bottomrule
\end{tabular}
\end{subtable}

\vspace{0em}

\begin{subtable}{0.78\linewidth}
\centering
\caption{Model parameterization.}
\label{tab:ablation_model}
\begin{tabular}{lcccc}
\toprule
\textbf{model param.} & \textbf{NFE=1} & \textbf{NFE=2} & \textbf{NFE=10} & \textbf{NFE=50} \\
\midrule
$u^\theta(x_t,t,r)$ 
& \cellcolor{gray!20}\textbf{29.06} 
& \cellcolor{gray!20}\textbf{26.97} 
& \cellcolor{gray!20}\textbf{23.48} 
& \cellcolor{gray!20}\textbf{20.82} \\
$X^\theta(x_t,t,r)$ 
& 34.28 & 32.10 & 27.74 & 24.61 \\
\bottomrule
\end{tabular}
\end{subtable}

\vspace{0em}

\begin{subtable}{0.78\linewidth}
\centering
\caption{Number of sampled time pairs $G$ with $n_g=64$.}
\label{tab:ablation_G}
\begin{tabular}{lcccc}
\toprule
\textbf{$G$} & \textbf{NFE=1} & \textbf{NFE=2} & \textbf{NFE=10} & \textbf{NFE=50} \\
\midrule
$1$  & 30.84 & 28.75 & 25.26 & 22.59 \\
$2$  & 29.73 & 27.64 & 24.15 & 21.48 \\
$4$  & \cellcolor{gray!20}\textbf{29.06} & \cellcolor{gray!20}\textbf{26.97} & \cellcolor{gray!20}\textbf{23.48} & \cellcolor{gray!20}\textbf{20.82} \\
$8$  & 28.82 & 26.73 & 23.24 & 20.58 \\
$16$ & 28.65 & 26.56 & 23.07 & 20.41 \\
\bottomrule
\end{tabular}
\end{subtable}

\vspace{0em}

\begin{subtable}{0.78\linewidth}
\centering
\caption{Number of positive/negative samples $n_g$ with $G=4$.}
\label{tab:ablation_ng}
\begin{tabular}{lcccc}
\toprule
\textbf{$n_g$} & \textbf{NFE=1} & \textbf{NFE=2} & \textbf{NFE=10} & \textbf{NFE=50} \\
\midrule
$1$   & 220.31 & 178.42 & 57.08 & 26.04 \\
$4$   & 191.26 & 38.92  & 32.71 & 24.88 \\
$16$  & 125.43 & 32.96  & 27.46 & 22.54 \\
$64$  & \cellcolor{gray!20}\textbf{29.06} & \cellcolor{gray!20}\textbf{26.97} & \cellcolor{gray!20}\textbf{23.48} & \cellcolor{gray!20}\textbf{20.82} \\
$128$ & 27.42 & 25.83 & 22.95 & 20.61 \\
\bottomrule
\end{tabular}
\end{subtable}

\end{table}





\paragraph{Time Embedding.}
Since DFM models a two-time transport map $T_{t,r}^{\theta}$, the network must be conditioned on both the absolute time location and the transport interval~\cite{geng2025mean,geng2025improved}. Table~\ref{tab:ablation_timeEmbedding} shows that all time-embedding variants produce reasonable FID scores, confirming that DFM is not sensitive to a single specific encoding choice. Among them, embedding $(t,t-r)$ performs best. This suggests that explicitly providing the current time $t$ together with the step size $t-r$ gives the model a more direct representation of the marginal transport problem. Directly embedding $(t,r)$ is also competitive, while adding all three variables $(t,r,t-r)$ does not further improve performance.

\paragraph{Time-pair Sampling Distribution.}
Prior work~\cite{esser2024scaling, geng2025mean,ma2025stochastic,anonymous2025cadvae,ma2025beyond} has shown that the distribution used to sample $t$ influences the generation quality. We study the distribution used to sample $(r,t)$ in Table~\ref{tab:ablation_timeSampler}. Note that $(r,t)$ are first sampled independently, followed by a post-processing step that enforces $t > r$ by swapping. Table~\ref{tab:ablation_timeSampler} reports that a logit-normal sampler performs the best, consistent with observations on Flow Matching~\cite{esser2024scaling, geng2025mean}.

\paragraph{Kernel Temperature.}
The kernel temperature $\tau$ controls the locality of the positive and negative interactions in the drift field~\cite{deng2026generative,he2026sinkhorn,lai2026unified}. As shown in Table~\ref{tab:ablation_temp}, very small temperatures degrade performance because the drift estimate becomes overly local and sensitive to finite-sample noise. Conversely, very large temperatures also hurt performance, since the kernel weights become too smooth and weaken the sample-specific attraction--repulsion structure. The best result is obtained at $\tau=1.0$, indicating that an intermediate temperature provides the most effective drift supervision for matching the empirical marginal pair $(q_{t,r}^{\theta},p_r)$.

\paragraph{Model Parameterization.}
We compare the proposed mean-velocity parameterization~\cite{geng2025mean,geng2025improved} $u^\theta(x_t,t,r)$ with directly predicting the future state $X^\theta(x_t,t,r)$~\cite{ma2026transition}, as shown in Table~\ref{tab:ablation_model}. The mean-velocity form gives better results across all NFEs. This is consistent with the DFM formulation in Eq.~\eqref{eq:dfm_transport_map}, where the transport map is written as
$T_{t,r}^{\theta}(x_t)=x_t+(r-t)u^\theta(x_t,t,r)$.
Factoring out the interval length $(r-t)$ makes the model naturally compatible with different step sizes and preserves the connection to the infinitesimal Flow Matching limit. In contrast, direct state prediction does not explicitly encode this step-size structure and is less stable when evaluated under different NFEs.

\paragraph{Number of Time Pairs and Positive/Negative Samples.}
We ablate the grouped drift configuration, including the number of sampled time pairs $G$ and the number of positive/negative samples $n_g$ per group, as shown in Table~\ref{tab:ablation_G} and Table~\ref{tab:ablation_ng}, respectively.

Increasing $G$ improves performance because each minibatch covers more marginal transport problems. However, even when $G$ is reduced to $1$, DFM still remains effective. This is because SGD optimization~\cite{bottou2012stochastic} already samples different time-pair distributions across training iterations, so the role of $G$ is mainly to increase the time-pair diversity within each minibatch rather than to make the objective well-defined. Consequently, the gain from increasing $G$ becomes modest after $G=4$, indicating that a small number of time-pair groups is sufficient in practice.

The effect of $n_g$ is more fundamental. The value of $n_g$ determines how well the empirical marginal drift field in Eq.~\eqref{eq:dfm_group_drift} approximates the distribution-level transport between $q_{t,r}^{\theta}$ and $p_r$~\cite{deng2026generative,he2026sinkhorn,lai2026unified}. When $n_g$ is very small, the drift field degenerates toward a pointwise conditional transport signal. In the extreme case $n_g=1$, the group contains only one positive and one negative sample, so the drift supervision effectively reduces to a Flow-Matching-like conditional velocity target~\cite{liu2023flow,lipman2024flow} rather than a marginal distribution-level drift~\cite{deng2026generative,he2026sinkhorn,lai2026unified}. As a result, one-step or few-step generation breaks down, since such conditional supervision is not sufficient to learn reliable large-step marginal transport~\cite{ma2025learningstraightflowsvariational,liu2023flow}. Nevertheless, at larger NFEs, the model still preserves a substantial part of the performance, consistent with the behavior of Flow Matching under multi-step sampling~\cite{liu2023flow,lipman2024flow}.
By contrast, increasing $n_g$ provides a better empirical estimate of the positive and negative drift components, making DFM closer to its intended marginal transport objective~\cite{deng2026generative,he2026sinkhorn,lai2026unified}. This substantially improves small-NFE generation, where accurate large-step transport is critical. Once $n_g$ is sufficiently large, the drift estimate becomes stable; performance therefore saturates around $n_g=64$, and increasing $n_g$ from $64$ to $128$ brings only marginal gains. Increasing NFE further improves performance in all cases, but the benefit is most meaningful when the drift estimate is already reliable.

\section{Conditional Generation}
\label{app:conditionalG}

\paragraph{Conditional target distributions.}
The formulation in the main paper is presented in the unconditional setting. In that case, the endpoint distribution is
\(X_1\sim p_1\), where \(p_1=p_{\mathrm{data}}\) is the full target distribution, while the source distribution \(p_0\) is fixed, e.g., a Gaussian distribution.
Conditional generation changes only the target distribution used for the transport problem.
Let \(c\) denote an external generation condition, such as a class label for ImageNet, MNIST, or FFHQ, an observation and previous actions in robotic control, or more general conditioning inputs such as prompts.
For a fixed condition \(c\), the target endpoint becomes the conditional data distribution
\begin{equation}
p_1(\cdot\mid c)
=
p_{\mathrm{data}}(\cdot\mid c),
\qquad
X_0\sim p_0,
\qquad
X_1\sim p_{\mathrm{data}}(\cdot\mid c).
\end{equation}
Thus, the source distribution remains unchanged, while the target distribution is restricted to samples compatible with \(c\).

It is important to distinguish this generation condition \(c\) from the conditional-path variable \(Z\) used in Flow Matching and in the DFM construction.
In Eq.~\eqref{eq:bayes_xt} and Eq.~\eqref{eq:dfm_conditional_pair}, \(Z=(X_0,X_1)\) indexes conditional paths whose mixture forms the marginal path.
By contrast, \(c\) is an observed input that specifies which conditional data distribution should be generated.
For each fixed \(c\), the conditional endpoint coupling is
\begin{equation}
(X_0,X_1)\mid c
\sim
\pi_c(x_0,x_1)
=
p_0(x_0)\,p_{\mathrm{data}}(x_1\mid c).
\end{equation}
Using the same interpolant as Eq.~\eqref{eq:generalInterpolant}, and equivalently the same two-time construction as Eq.~\eqref{eq:dfm_conditional_pair}, this coupling induces
\begin{equation}
X_{t\mid c}
=
\alpha(t)X_0+\beta(t)X_1,
\qquad
X_{r\mid c}
=
\alpha(r)X_0+\beta(r)X_1,
\qquad
(X_0,X_1)\mid c\sim\pi_c.
\end{equation}
We denote the corresponding marginal distributions at times \(t\) and \(r\) by \(p_{t\mid c}\) and \(p_{r\mid c}\).
The conditional DFM problem is therefore the same distribution transport problem as in the main text, but applied at each fixed condition:
\[
p_{t\mid c}
\longrightarrow
p_{r\mid c}.
\]

\paragraph{Conditional DFM objective.}
The transport map is conditioned on \(c\) by feeding the same condition input into the model:
\begin{equation}
T_{t,r}^{\theta,c}(x_t)
:=
x_t+(r-t)u^\theta(x_t,t,r,c),
\qquad
q_{t,r}^{\theta,c}
:=
\left(T_{t,r}^{\theta,c}\right)_{\#}p_{t\mid c}.
\end{equation}
The training objective is obtained from Eq.~\eqref{eq:dfm_loss} by replacing
\[
p_t,\ p_r,\ q_{t,r}^{\theta}
\quad\text{with}\quad
p_{t\mid c},\ p_{r\mid c},\ q_{t,r}^{\theta,c}.
\]
Likewise, the grouped empirical drift in Eq.~\eqref{eq:dfm_group_drift} and the group-wise loss in Eq.~\eqref{eq:dfm_group_loss} are computed within each fixed condition and each sampled time pair.
Positive samples are drawn from the conditional target-time marginal \(p_{r\mid c}\), while negative samples are generated by applying the same conditional model \(T_{t,r}^{\theta,c}\) to samples from \(p_{t\mid c}\).
Therefore, both the positive and negative empirical measures correspond to the same condition \(c\), and samples from different conditions should not be mixed when estimating the conditional drift field.

\paragraph{High-dimensional conditions.}
For class-conditional generation, a minibatch can usually contain multiple data samples for the same label, giving an empirical approximation of \(p_{\mathrm{data}}(\cdot\mid c)\).
For high-dimensional or structured conditions, such as robotic observations with previous actions or future prompt-based generation, sampling the full distribution of \(c\) may be difficult or unnecessary.
Often, the available training datum is a single observed pair \((c,x_1)\).
In this case, for that fixed condition, we use the empirical conditional target
\[
\widehat p_{\mathrm{data}}(\cdot\mid c)
=
\delta_{x_1}.
\]
This does not require modeling or sampling the distribution of \(c\).
Instead, \(c\) is treated as the observed condition input, while multiple source samples can still be drawn from \(p_0\).
Given \(x_{0,i}\sim p_0\), we construct multiple conditional positive samples through the same interpolant:
\begin{equation}
x_{t,i}
=
\alpha(t)x_{0,i}+\beta(t)x_1,
\qquad
x_{r,i}
=
\alpha(r)x_{0,i}+\beta(r)x_1.
\end{equation}
The corresponding negative samples are produced by the conditional model under the same \(c\):
\begin{equation}
\widehat x_{r,i}
=
T_{t,r}^{\theta,c}(x_{t,i}).
\end{equation}
Thus, even when only one target sample is observed for a complex condition, DFM can still form a conditional empirical marginal pair through multiple source samples and train the drift-based transport from \(p_{t\mid c}\) to \(p_{r\mid c}\).

\section{Wasserstein geometry of Drift Flow Matching}
\label{app:dfm_ot}

In this subsection, we record a basic $W_2$-geometric property of \textbf{DFM}. 
The key observation is that the same conditional path construction as in Flow Matching induces, for every pair $(t,r)\in\Delta$, a natural feasible coupling between the intermediate marginals $p_t$ and $p_r$. 
This immediately yields an explicit upper bound on the quadratic transport cost between $p_t$ and $p_r$. 
Under the linear interpolant, this bound simplifies to a quantitative $W_2$-Lipschitz estimate and a uniform discrete Wasserstein-action bound along the whole \textbf{DFM} path.

For probability measures $\mu,\nu\in\mathcal P_2(\mathbb R^d)$, we write
\begin{equation}
W_2^2(\mu,\nu)
:=
\inf_{\eta\in\Pi(\mu,\nu)}
\int_{\mathbb R^d\times\mathbb R^d}\|x-y\|^2\,d\eta(x,y),
\label{proof:eq:w2_squared_def}
\end{equation}
where $\Pi(\mu,\nu)$ denotes the set of couplings with marginals $\mu$ and $\nu$. 
Thus, $W_2^2(\mu,\nu)$ is the minimal quadratic transport cost between $\mu$ and $\nu$, and
\begin{equation}
W_2(\mu,\nu):=\bigl(W_2^2(\mu,\nu)\bigr)^{1/2}
\label{proof:eq:w2_def}
\end{equation}
is the associated $2$-Wasserstein distance.

For intuition, $W_2(\mu,\nu)$ is a distance-type quantity, whereas $W_2^2(\mu,\nu)$ is the corresponding quadratic transport cost. 
Accordingly, $W_2^2$ is not additive along a partition in general, while an action-type quantity is additive over time intervals. 
In dynamic optimal transport, the Wasserstein action is the time integral of the squared transport speed along a distribution path. 
For a discrete time interval $[t_m,t_{m+1}]$, the quantity
\[
\frac{W_2^2(p_{t_m},p_{t_{m+1}})}{t_{m+1}-t_m}
\]
is therefore the natural discrete analogue of the action on that interval: it is the squared transport distance normalized by the time length. 
Summing over a partition yields the total discrete Wasserstein action of the path.

We emphasize that the results below provide \emph{geometric upper bounds} induced by a feasible coupling, and show that the \textbf{DFM} path inherits from the Flow Matching construction a quantitatively controlled $W_2$ geometry.

\begin{proposition}[Canonical intermediate coupling induced by the Flow Matching path construction]
\label{prop:dfm_canonical_coupling}
Let $(X_0,X_1)\sim \pi$ be the endpoint coupling used in Eq.~\eqref{eq:dfm_conditional_pair}, and let
\[
X_t^Z = \alpha(t)X_0^Z+\beta(t)X_1^Z,
\qquad
X_r^Z = \alpha(r)X_0^Z+\beta(r)X_1^Z
\]
be the conditional pair in Eq.~\eqref{eq:dfm_conditional_pair}. Define
\[
\gamma_{t,r}:=\mathrm{Law}(X_t^Z,X_r^Z).
\]
Then $\gamma_{t,r}\in \Pi(p_t,p_r)$. Consequently,
\begin{equation}
W_2^2(p_t,p_r)
\;\le\;
\mathbb{E}\!\left[\|X_r^Z-X_t^Z\|^2\right]
=
\mathbb{E}\!\left[\|\Delta\alpha_{t,r}X_0+\Delta\beta_{t,r}X_1\|^2\right],
\label{proof:eq:dfm_w2_general_bound}
\end{equation}
where
\[
\Delta\alpha_{t,r}:=\alpha(r)-\alpha(t),
\qquad
\Delta\beta_{t,r}:=\beta(r)-\beta(t).
\]
\end{proposition}

\begin{proof}
By construction, the first marginal of $\gamma_{t,r}$ is the law of $X_t^Z$, namely $p_t$, and the second marginal is the law of $X_r^Z$, namely $p_r$. Therefore $\gamma_{t,r}\in \Pi(p_t,p_r)$.

By the definition of the $2$-Wasserstein distance,
\[
W_2^2(p_t,p_r)
=
\inf_{\eta\in\Pi(p_t,p_r)}
\int \|x-y\|^2\,d\eta(x,y)
\le
\int \|x-y\|^2\,d\gamma_{t,r}(x,y).
\]
Using the definition of $\gamma_{t,r}$,
\[
\int \|x-y\|^2\,d\gamma_{t,r}(x,y)
=
\mathbb{E}\!\left[\|X_r^Z-X_t^Z\|^2\right].
\]
Moreover, by Eq.~\eqref{eq:dfm_conditional_pair},
\[
X_r^Z-X_t^Z
=
(\alpha(r)-\alpha(t))X_0^Z+(\beta(r)-\beta(t))X_1^Z
=
\Delta\alpha_{t,r}X_0+\Delta\beta_{t,r}X_1,
\]
which proves Eq.~\eqref{proof:eq:dfm_w2_general_bound}.
\end{proof}

\begin{corollary}[General $W_2$ continuity bound]
\label{cor:dfm_w2_continuity_general}
Assume $X_0,X_1\in L^2$, and let
\[
M_2:=\Big(\mathbb{E}\big[(\|X_0\|+\|X_1\|)^2\big]\Big)^{1/2}.
\]
Then for any $(t,r)\in\Delta$,
\begin{equation}
W_2(p_t,p_r)
\;\le\;
\bigl(|\Delta\alpha_{t,r}|+|\Delta\beta_{t,r}|\bigr)\,M_2.
\label{proof:eq:dfm_w2_continuity_general}
\end{equation}
In particular, if $\alpha$ and $\beta$ are Lipschitz on $[0,1]$ with constants $L_\alpha$ and $L_\beta$, then
\begin{equation}
W_2(p_t,p_r)
\;\le\;
(L_\alpha+L_\beta)\,M_2\,|r-t|,
\qquad \forall\, (t,r)\in\Delta,
\label{proof:eq:dfm_w2_lipschitz_general}
\end{equation}
so $\{p_t\}_{t\in[0,1]}\subset\mathcal P_2(\mathbb R^d)$ is a $W_2$-Lipschitz path, and hence a $W_2$-absolutely-continuous path.
\end{corollary}

\begin{proof}
By Proposition~\ref{prop:dfm_canonical_coupling},
\[
W_2(p_t,p_r)
\le
\Big(\mathbb{E}\!\left[\|\Delta\alpha_{t,r}X_0+\Delta\beta_{t,r}X_1\|^2\right]\Big)^{1/2}.
\]
Using the triangle inequality,
\[
\|\Delta\alpha_{t,r}X_0+\Delta\beta_{t,r}X_1\|
\le
\bigl(|\Delta\alpha_{t,r}|+|\Delta\beta_{t,r}|\bigr)(\|X_0\|+\|X_1\|).
\]
Taking the $L^2$ norm yields Eq.~\eqref{proof:eq:dfm_w2_continuity_general}. 
If $\alpha$ and $\beta$ are Lipschitz, then
\[
|\Delta\alpha_{t,r}|\le L_\alpha |r-t|,
\qquad
|\Delta\beta_{t,r}|\le L_\beta |r-t|,
\]
and Eq.~\eqref{proof:eq:dfm_w2_lipschitz_general} follows immediately.
\end{proof}

\begin{corollary}[Linear-interpolant specialization]
\label{cor:dfm_w2_linear_bound}
Assume the linear interpolant $\alpha(t)=1-t$ and $\beta(t)=t$. Then for any $(t,r)\in\Delta$,
\begin{equation}
W_2^2(p_t,p_r)
\;\le\;
(r-t)^2\,\mathbb{E}\!\left[\|X_1-X_0\|^2\right].
\label{proof:eq:dfm_w2_linear_bound}
\end{equation}
Hence the path $\{p_t\}_{t\in[0,1]}$ is $W_2$-Lipschitz, with
\begin{equation}
W_2(p_t,p_r)
\;\le\;
|r-t|\,\Big(\mathbb{E}\|X_1-X_0\|^2\Big)^{1/2}.
\label{proof:eq:dfm_w2_lipschitz}
\end{equation}
\end{corollary}

\begin{proof}
Under the linear interpolant,
\[
\Delta\alpha_{t,r}=-(r-t),
\qquad
\Delta\beta_{t,r}=r-t,
\]
hence
\[
X_r^Z-X_t^Z=(r-t)(X_1-X_0).
\]
Substituting this identity into Eq.~\eqref{proof:eq:dfm_w2_general_bound} gives Eq.~\eqref{proof:eq:dfm_w2_linear_bound}. Taking square roots yields Eq.~\eqref{proof:eq:dfm_w2_lipschitz}.
\end{proof}


\begin{proposition}[Discrete Wasserstein-action bound along the DFM path]
\label{prop:dfm_action_bound}
Assume the linear interpolant $\alpha(t)=1-t$ and $\beta(t)=t$. Let
\[
0=t_0<t_1<\cdots<t_N=1
\]
be any partition of $[0,1]$. Then
\begin{equation}
\sum_{m=0}^{N-1}
\frac{W_2^2(p_{t_m},p_{t_{m+1}})}{t_{m+1}-t_m}
\;\le\;
\mathbb{E}\!\left[\|X_1-X_0\|^2\right].
\label{proof:eq:dfm_discrete_action_bound}
\end{equation}
In this sense, each quantity
\[
\frac{W_2^2(p_{t_m},p_{t_{m+1}})}{t_{m+1}-t_m}
\]
can be interpreted as the discrete Wasserstein action of the interval $[t_m,t_{m+1}]$.
\end{proposition}

\begin{proof}
For each interval $[t_m,t_{m+1}]$, Corollary~\ref{cor:dfm_w2_linear_bound} gives
\[
W_2^2(p_{t_m},p_{t_{m+1}})
\le
(t_{m+1}-t_m)^2\,\mathbb{E}\!\left[\|X_1-X_0\|^2\right].
\]
Dividing both sides by $(t_{m+1}-t_m)$ and summing over $m$, we obtain
\[
\sum_{m=0}^{N-1}
\frac{W_2^2(p_{t_m},p_{t_{m+1}})}{t_{m+1}-t_m}
\le
\sum_{m=0}^{N-1}
(t_{m+1}-t_m)\,\mathbb{E}\!\left[\|X_1-X_0\|^2\right].
\]
Since $\sum_{m=0}^{N-1}(t_{m+1}-t_m)=1$, Eq.~\eqref{proof:eq:dfm_discrete_action_bound} follows.
\end{proof}

\begin{remark}[Interpretation and relation to Flow Matching]
\label{remark:dfm_ot_interpretation}
Proposition~\ref{prop:dfm_canonical_coupling} shows that the intermediate path $\{p_t\}_{t\in[0,1]}$ used by \textbf{DFM} inherits, from the same conditional path construction as in Flow Matching, a natural coupling structure between any two intermediate marginals. 
Therefore, for every pair $(p_t,p_r)$, one obtains an explicit feasible coupling $\gamma_{t,r}$ and hence an explicit Wasserstein upper bound.

The general estimate in Eq.~\eqref{proof:eq:dfm_w2_general_bound} does not require the linear interpolant. 
It shows that the Wasserstein geometry of the \textbf{DFM} path is already controlled at the level of the general Flow Matching path construction. 
Under additional regularity of $\alpha$ and $\beta$, this yields $W_2$-continuity of the path; under the linear interpolant, it further simplifies to the explicit scaling law in Eq.~\eqref{proof:eq:dfm_w2_linear_bound}, which implies both the $W_2$-Lipschitz estimate in Eq.~\eqref{proof:eq:dfm_w2_lipschitz} and the finite discrete Wasserstein-action bound in Eq.~\eqref{proof:eq:dfm_discrete_action_bound}.

Thus, even when the endpoint coupling $\pi$ is not an optimal coupling, \textbf{DFM} still decomposes endpoint generation into a family of short-range transport problems along a $W_2$-absolutely-continuous path. 
The significance of these results is therefore not global optimality, but rather that \textbf{DFM} inherits from Flow Matching a geometrically controlled transport path whose local marginal pairs satisfy explicit Wasserstein bounds.
\end{remark}

\begin{remark}[Comparison with the whole FM process under the same endpoint coupling]
\label{remark:dfm_vs_fm_whole_process}
Under the same endpoint coupling $\pi=\mathrm{Law}(X_0,X_1)$ used to construct the Flow Matching path, the comparison between \textbf{DFM} and Flow Matching should be made at the level of the whole transport process, rather than at the level of a single local pair $(p_t,p_r)$.
In particular, Proposition~\ref{prop:dfm_action_bound} yields
\[
\sum_{m=0}^{N-1}
\frac{W_2^2(p_{t_m},p_{t_{m+1}})}{t_{m+1}-t_m}
\le
\mathbb{E}\!\left[\|X_1-X_0\|^2\right].
\]
The left-hand side is the discrete Wasserstein action of the entire \textbf{DFM} path, whereas the right-hand side is the quadratic transport cost associated with the endpoint coupling $\pi$.
Therefore, relative to the same endpoint coupling, \textbf{DFM} does not increase the total quadratic transport action; instead, it decomposes the endpoint transport into a sequence of short-range transport steps with controlled $W_2$ geometry.

\end{remark}

\section{Infinitesimal-step limit of Drift Flow Matching}
\label{app:dfm_infinitesimal_limit}

In this section, we justify Eq.~\eqref{eq:dfm_infinitesimal_limit}.
The main distinction is that the infinitesimal limit concerns the \emph{two-time transport velocity}
\[
u^\theta(x_t,t,r)
=
\frac{T_{t,r}^\theta(x_t)-x_t}{r-t},
\]
which is used by the DFM transport map during generation, rather than the drift correction field \(V_{q_{t,r}^\theta,p_r}\) used as the distribution-matching signal in training. The drift field pushes the predicted distribution \(q_{t,r}^\theta\) toward the target marginal \(p_r\), whereas
\(u^\theta\) parameterizes the actual displacement from time \(t\) to time \(r\).

\begin{definition}[Canonical Flow Matching transport and mean velocity]
\label{def:fm_canonical_mean_velocity}
Let \(\Phi_{t,r}\) denote the flow map induced by the marginal Flow Matching ODE
\begin{equation}
\frac{d x_\tau}{d\tau}=v(x_\tau,\tau),
\qquad
x_\tau\big|_{\tau=t}=x_t,
\label{proof:eq:dfm_limit_fm_ode}
\end{equation}
where \(v(x_\tau,\tau)\) is the marginal Flow Matching velocity in Eq.~\eqref{proof:eq:marginalV}. Then
\begin{equation}
\Phi_{t,r}(x_t)
=
x_t+\int_t^r v(x_\tau,\tau)\,d\tau,
\qquad
x_\tau=\Phi_{t,\tau}(x_t).
\label{proof:eq:dfm_limit_flow_map}
\end{equation}
The corresponding canonical two-time mean velocity is defined as
\begin{equation}
u(x_t,t,r)
:=
\frac{\Phi_{t,r}(x_t)-x_t}{r-t}
=
\frac{1}{r-t}\int_t^r v(x_\tau,\tau)\,d\tau .
\label{proof:eq:dfm_limit_mean_velocity}
\end{equation}
\end{definition}

\begin{proposition}[Mean-velocity representation of the canonical transport]
\label{prop:dfm_mean_velocity_representation}
Suppose the two-time transport from \(p_t\) to \(p_r\) coincides with the canonical marginal Flow Matching flow map \(\Phi_{t,r}\). Then the associated DFM velocity equals the average of the marginal Flow Matching velocity along the trajectory over the interval \([t,r]\), namely
\begin{equation}
u(x_t,t,r)
=
\frac{1}{r-t}\int_t^r v(x_\tau,\tau)\,d\tau .
\end{equation}
\end{proposition}

\begin{proof}
By Definition~\ref{def:fm_canonical_mean_velocity},
\[
\Phi_{t,r}(x_t)-x_t
=
\int_t^r v(x_\tau,\tau)\,d\tau .
\]
Dividing both sides by \(r-t\) gives the desired expression.
\end{proof}

\begin{lemma}[First-order expansion of the Flow Matching marginal path]
\label{lem:fm_density_expansion}
Assume the marginal Flow Matching path \(p_t\) satisfies the continuity equation
\begin{equation}
\partial_t p_t+\nabla\cdot(p_t v)=0 .
\label{proof:eq:dfm_limit_fm_continuity}
\end{equation}
Then, for \(h>0\) sufficiently small,
\begin{equation}
p_{t+h}
=
p_t
-
h\,\nabla\cdot(p_t v)
+
o(h),
\label{proof:eq:dfm_limit_density_expansion_fm}
\end{equation}
in the distributional sense.
\end{lemma}

\begin{proof}
The result follows directly from the first-order Taylor expansion of the path \(p_t\) in time and the continuity equation \(\partial_t p_t=-\nabla\cdot(p_t v)\).
\end{proof}

\begin{lemma}[First-order expansion of a local DFM pushforward]
\label{lem:dfm_pushforward_expansion}
Consider a local DFM map of the form
\begin{equation}
T_{t,t+h}^\theta(x)
=
x+h\,u^\theta(x,t,t+h),
\label{proof:eq:dfm_limit_local_map}
\end{equation}
and let \(X_t\sim p_t\). Denote the pushforward distribution by
\[
q_{t,t+h}^\theta
:=
(T_{t,t+h}^\theta)_\# p_t .
\]
Then, for \(h>0\) sufficiently small,
\begin{equation}
q_{t,t+h}^\theta
=
p_t
-
h\,\nabla\cdot\!\left(
p_t u^\theta(\cdot,t,t+h)
\right)
+
o(h),
\label{proof:eq:dfm_limit_density_expansion_dfm}
\end{equation}
in the distributional sense.
\end{lemma}

\begin{proof}
For any smooth compactly supported test function \(\varphi\), the pushforward definition gives
\[
\int \varphi(y)\,dq_{t,t+h}^\theta(y)
=
\int \varphi\!\left(x+h\,u^\theta(x,t,t+h)\right)p_t(x)\,dx .
\]
Using the first-order Taylor expansion,
\[
\varphi\!\left(x+h\,u^\theta(x,t,t+h)\right)
=
\varphi(x)
+
h\,\nabla\varphi(x)\cdot u^\theta(x,t,t+h)
+
o(h).
\]
Therefore,
\[
\int \varphi(y)\,dq_{t,t+h}^\theta(y)
=
\int \varphi(x)p_t(x)\,dx
+
h\int \nabla\varphi(x)\cdot u^\theta(x,t,t+h)p_t(x)\,dx
+
o(h).
\]
By integration by parts,
\[
\int \nabla\varphi(x)\cdot u^\theta(x,t,t+h)p_t(x)\,dx
=
-\int \varphi(x)\,
\nabla\cdot\!\left(p_t u^\theta(\cdot,t,t+h)\right)(x)\,dx .
\]
Hence
\[
q_{t,t+h}^\theta
=
p_t
-
h\,\nabla\cdot\!\left(
p_t u^\theta(\cdot,t,t+h)
\right)
+
o(h),
\]
which proves the claim.
\end{proof}

\begin{proposition}[Infinitesimal distributional consistency]
\label{prop:dfm_distributional_consistency}
Assume that the learned DFM transport matches the nearby Flow Matching marginal transport to first order, namely
\begin{equation}
q_{t,t+h}^\theta
=
p_{t+h}
+
o(h).
\label{proof:eq:dfm_limit_distributional_matching}
\end{equation}
Then DFM and Flow Matching induce the same infinitesimal marginal transport in the sense that
\begin{equation}
\nabla\cdot\!\left(
p_t u^\theta(\cdot,t,t+h)
\right)
=
\nabla\cdot(p_t v)
+
o(1).
\label{proof:eq:dfm_limit_same_continuity}
\end{equation}
\end{proposition}

\begin{proof}
By Lemma~\ref{lem:fm_density_expansion},
\[
p_{t+h}
=
p_t
-
h\,\nabla\cdot(p_t v)
+
o(h).
\]
By Lemma~\ref{lem:dfm_pushforward_expansion},
\[
q_{t,t+h}^\theta
=
p_t
-
h\,\nabla\cdot\!\left(
p_t u^\theta(\cdot,t,t+h)
\right)
+
o(h).
\]
Using the assumption \(q_{t,t+h}^\theta=p_{t+h}+o(h)\), we obtain
\[
p_t
-
h\,\nabla\cdot\!\left(
p_t u^\theta(\cdot,t,t+h)
\right)
+
o(h)
=
p_t
-
h\,\nabla\cdot(p_t v)
+
o(h).
\]
Subtracting \(p_t\) and dividing by \(h\) gives
\[
\nabla\cdot\!\left(
p_t u^\theta(\cdot,t,t+h)
\right)
=
\nabla\cdot(p_t v)
+
o(1).
\]
\end{proof}

\begin{theorem}[Infinitesimal-step limit of DFM]
\label{thm:dfm_infinitesimal_limit}
Let \(h=r-t\). Assume that the learned DFM map is first-order consistent with the canonical local flow map induced by the marginal Flow Matching ODE, i.e.,
\begin{equation}
T_{t,t+h}^\theta(x_t)
=
\Phi_{t,t+h}(x_t)
+
o(h).
\label{proof:eq:dfm_limit_first_order_consistency}
\end{equation}
Assume further that \(v(x_\tau,\tau)\) is continuous along the marginal trajectory \(\tau\mapsto x_\tau=\Phi_{t,\tau}(x_t)\). Then
\begin{equation}
\lim_{h\to 0}
u^\theta(x_t,t,t+h)
=
v(x_t,t).
\label{proof:eq:dfm_limit_pointwise_h}
\end{equation}
Equivalently,
\begin{equation}
\lim_{r\to t}
u^\theta(x_t,t,r)
=
v(x_t,t).
\label{proof:eq:dfm_limit_pointwise_r}
\end{equation}
Thus, in the infinitesimal-step limit, the DFM two-time transport velocity recovers the marginal velocity field of standard Flow Matching.
\end{theorem}

\begin{proof}
By the DFM transport parameterization in Eq.~\eqref{eq:dfm_transport_map},
\[
T_{t,t+h}^\theta(x_t)
=
x_t+h\,u^\theta(x_t,t,t+h).
\]
By the Flow Matching flow-map representation in Eq.~\eqref{proof:eq:dfm_limit_flow_map},
\[
\Phi_{t,t+h}(x_t)
=
x_t+\int_t^{t+h}v(x_\tau,\tau)\,d\tau .
\]
Using the first-order consistency assumption
\[
T_{t,t+h}^\theta(x_t)
=
\Phi_{t,t+h}(x_t)
+
o(h),
\]
we obtain
\[
x_t+h\,u^\theta(x_t,t,t+h)
=
x_t+\int_t^{t+h}v(x_\tau,\tau)\,d\tau
+
o(h).
\]
Canceling \(x_t\) gives
\begin{equation}
h\,u^\theta(x_t,t,t+h)
=
\int_t^{t+h}v(x_\tau,\tau)\,d\tau
+
o(h).
\label{proof:eq:dfm_limit_compare_displacement}
\end{equation}
Dividing by \(h\), we have
\begin{equation}
u^\theta(x_t,t,t+h)
=
\frac{1}{h}
\int_t^{t+h}v(x_\tau,\tau)\,d\tau
+
o(1).
\label{proof:eq:dfm_limit_average_plus_error}
\end{equation}
Since \(v(x_\tau,\tau)\) is continuous along the trajectory,
\[
\lim_{h\to 0}
\frac{1}{h}
\int_t^{t+h}v(x_\tau,\tau)\,d\tau
=
v(x_t,t).
\]
Therefore,
\[
\lim_{h\to 0}
u^\theta(x_t,t,t+h)
=
v(x_t,t).
\]
Since \(h=r-t\), this is equivalently
\[
\lim_{r\to t}
u^\theta(x_t,t,r)
=
v(x_t,t).
\]
\end{proof}

\begin{remark}[Distributional versus pointwise consistency]
\label{rem:dfm_distributional_vs_pointwise}
Proposition~\ref{prop:dfm_distributional_consistency} shows that if the DFM pushforward \(q_{t,t+h}^\theta\) matches \(p_{t+h}\) to first order, then DFM and Flow Matching agree at the level of infinitesimal marginal transport.
This is a distributional statement and only identifies the divergence
\(\nabla\cdot(p_t u^\theta)\), not necessarily the pointwise velocity field.
The pointwise limit in Theorem~\ref{thm:dfm_infinitesimal_limit} requires the stronger first-order consistency assumption
\(T_{t,t+h}^\theta(x_t)=\Phi_{t,t+h}(x_t)+o(h)\).
\end{remark}

\begin{remark}[Role of the drift field]
\label{rem:dfm_drift_field_role}
The drift correction field \(V_{q_{t,r}^\theta,p_r}\) does not itself converge to the Flow Matching velocity \(v(x_t,t)\). Instead, it provides the distribution-level training signal that updates the predicted target-time distribution \(q_{t,r}^\theta\) toward \(p_r\). The infinitesimal Flow Matching limit concerns the learned transport velocity
\[
u^\theta(x_t,t,r)
=
\frac{T_{t,r}^\theta(x_t)-x_t}{r-t},
\]
which is the velocity used by DFM during generation.
\end{remark}

\section{Gradient Descent in Drift Field for Drift Flow Matching}
\label{app:dfm_gradient_descent}

We justify the statement following Eq.~\eqref{eq:dfm_gradient}. For a fixed sampled pair $(t,r)$, the DFM objective is exactly the Drift stop-gradient loss applied to the predicted target-time particles $\widehat x_r$, with current model distribution $q_{t,r}^\theta$ and target marginal $p_r$. 

\begin{proposition}[Stop-gradient DFM loss induces the DFM drift ODE in particle space]  \label{propo:dfmloss_induces_ode}
Fix a sampled time pair $(t,r)$. Let $\widehat x_r^1,\ldots,\widehat x_r^n \in \mathbb{R}^d$ be free particles and $x_r^1,\ldots,x_r^n\in \mathbb{R}^d\sim p_r$ be target particles (no parametrization restriction), and let
\[
q_{t,r}
=
\frac{1}{n}\sum_{i=1}^n \delta_{\widehat x_r^i},
\qquad
p_r
=
\frac{1}{n}\sum_{i=1}^{n}\delta_{x_{r}^{i}}
\]
be their empirical distribution and target distribution at time $r$, respectively. Let $V_{q_{t,r},p_r} : \mathbb{R}^d \to \mathbb{R}^d$ denote the DFM drift field for this time pair. Consider the stop-gradient objective
\[
\mathcal{L}_{\mathrm{DFM}}(\widehat x_r^1,\ldots,\widehat x_r^n)
:=
\frac{1}{2}\sum_{i=1}^n
\left\|
\widehat x_r^i
-
\operatorname{sg}\!\left(
\widehat x_r^i
+
V_{q_{t,r},p_r}(\widehat x_r^i)
\right)
\right\|^2.
\]
Then the gradient with respect to each particle is
\[
\nabla_{\widehat x_r^i}\mathcal{L}_{\mathrm{DFM}}
=
-
V_{q_{t,r},p_r}(\widehat x_r^i).
\]
Here $s \ge 0$ denotes the auxiliary gradient-flow time associated with continuous-time descent on $\mathcal{L}_{\mathrm{DFM}}$ for the fixed pair $(t,r)$; it is independent of the model time indices $t$ and $r$. Hence the continuous-time gradient flow in particle space satisfies
\[
\dot{\widehat x}_{r,s}^i
=
-
\nabla_{\widehat x_r^i}\mathcal{L}_{\mathrm{DFM}}(\widehat x_{r,s})
=
V_{q_{t,r},p_r}(\widehat x_{r,s}^i).
\]
Therefore, when the target-time outputs are viewed directly as particles, the DFM stop-gradient loss induces the DFM drift field, which pushes the predicted distribution $q_{t,r}$ toward the target marginal $p_r$.
\end{proposition}

\begin{proof}
Fix $i$ and denote
\[
s_i
:=
\operatorname{sg}\!\left(
\widehat x_r^i
+
V_{q_{t,r},p_r}(\widehat x_r^i)
\right).
\]
By definition of $\operatorname{sg}(\cdot)$, $s_i$ is treated as constant when differentiating with respect to $\widehat x_r^i$. Therefore,
\[
\nabla_{\widehat x_r^i}
\frac{1}{2}
\left\|
\widehat x_r^i-s_i
\right\|^2
=
\widehat x_r^i-s_i.
\]
Since
\[
\widehat x_r^i-s_i
=
-
V_{q_{t,r},p_r}(\widehat x_r^i),
\]
we obtain
\[
\nabla_{\widehat x_r^i}\mathcal{L}_{\mathrm{DFM}}
=
-
V_{q_{t,r},p_r}(\widehat x_r^i).
\]
The gradient-flow identity then gives
\[
\dot{\widehat x}_{r,s}^i
=
-
\nabla_{\widehat x_r^i}\mathcal{L}_{\mathrm{DFM}}(\widehat x_{r,s})
=
V_{q_{t,r},p_r}(\widehat x_{r,s}^i).
\]
This is exactly the DFM drift ODE in particle space for the fixed pair $(t,r)$.
\end{proof}

\begin{proposition}[Gradient of the DFM stop-gradient loss for $T_{t,r}^\theta$]
Fix a sampled pair $(t,r)$ and let $x_r^1,\ldots,x_r^n\in \mathbb{R}^d\sim p_r$ be target particles. Define the predicted target-time particles by
\[
\widehat x_r^i(\theta)
:=
T_{t,r}^\theta(x_t^i) \in \mathbb{R}^d,
\qquad
q_{t,r}^\theta
=
\frac{1}{n}\sum_{i=1}^n \delta_{\widehat x_r^i(\theta)}.
\]
Consider the stop-gradient regression objective
\[
\mathcal{L}_{\mathrm{DFM}}^{(t,r)}(\theta)
:=
\frac{1}{2}\sum_{i=1}^n
\left\|
\widehat x_r^i(\theta)
-
\operatorname{sg}\!\left(
\widehat x_r^i(\theta)
+
V_{q_{t,r}^\theta,p_r}(\widehat x_r^i(\theta))
\right)
\right\|^2.
\]
Then its gradient is
\[
\nabla_\theta \mathcal{L}_{\mathrm{DFM}}^{(t,r)}(\theta)
=
-
\sum_{i=1}^n
J_{T_{t,r}}(\theta,x_t^i)^\top
V_{q_{t,r}^\theta,p_r}(\widehat x_r^i(\theta)),
\]
where $J_{T_{t,r}}(\theta,x_t^i)\in\mathbb{R}^{d\times \dim(\theta)}$ denotes the Jacobian of $T_{t,r}^\theta(x_t^i)$ with respect to $\theta$. Equivalently, after reinstating the empirical normalization and taking expectation over $(t,r)\sim\rho$ and $X_t\sim p_t$, we recover Eq.~\eqref{eq:dfm_gradient}:
\[
\nabla_\theta \mathcal{L}_{\mathrm{DFM}}(\theta)
=
-
\mathbb{E}_{(t,r)\sim\rho,\;X_t\sim p_t}
\left[
J_{T_{t,r}}(\theta,X_t)^\top
V_{q_{t,r}^\theta,p_r}\!\left(T_{t,r}^\theta(X_t)\right)
\right].
\]
\end{proposition}

\begin{proof}
Let
\[
\widehat x_r^i(\theta)=T_{t,r}^\theta(x_t^i)
\]
and denote
\[
s_i
:=
\operatorname{sg}\!\left(
\widehat x_r^i(\theta)
+
V_{q_{t,r}^\theta,p_r}(\widehat x_r^i(\theta))
\right).
\]
By definition of $\operatorname{sg}(\cdot)$, $s_i$ is treated as constant when differentiating with respect to $\theta$. Hence
\[
\nabla_\theta
\frac{1}{2}
\left\|
\widehat x_r^i(\theta)-s_i
\right\|^2
=
J_{T_{t,r}}(\theta,x_t^i)^\top
\bigl(\widehat x_r^i(\theta)-s_i\bigr).
\]
Since
\[
\widehat x_r^i(\theta)-s_i
=
-
V_{q_{t,r}^\theta,p_r}(\widehat x_r^i(\theta)),
\]
summing over $i$ yields
\[
\nabla_\theta \mathcal{L}_{\mathrm{DFM}}^{(t,r)}(\theta)
=
-
\sum_{i=1}^n
J_{T_{t,r}}(\theta,x_t^i)^\top
V_{q_{t,r}^\theta,p_r}(\widehat x_r^i(\theta)).
\]
Dividing by the minibatch size and taking expectation over the sampled pair $(t,r)$ and the marginal state $X_t\sim p_t$ gives Eq.~\eqref{eq:dfm_gradient}.
\end{proof}

\begin{remark}[Particle-space drift versus parameter-space update]
\label{rem:dfm_particle_vs_parameter_update}
Proposition~\ref{propo:dfmloss_induces_ode} shows that, if the predicted target-time particles \(\widehat x_r^i\) were optimized directly, the stop-gradient DFM loss would induce the particle-space drift ODE
\[
\dot{\widehat x}_{r,s}^i
=
V_{q_{t,r},p_r}(\widehat x_{r,s}^i)
\]
for each fixed sampled pair \((t,r)\). Under the parameterized map \(\widehat x_r^i=T_{t,r}^\theta(x_t^i)\), however, the output-space motion is constrained by the model Jacobian. Therefore, a gradient step in \(\theta\) is not generally identical to the explicit particle update
\[
\widehat x_r^i
\mapsto
\widehat x_r^i+\eta V_{q_{t,r}^\theta,p_r}(\widehat x_r^i).
\]
Instead, Eq.~\eqref{eq:dfm_gradient} shows that the parameter-space descent direction is the pullback of the DFM drift field through \(T_{t,r}^\theta\). Thus the stop-gradient objective supplies the correct distribution-level drift signal, while the realizable update is mediated by the parameterization.
\end{remark}

\section{Pseudocode of Drift Flow Matching}
\label{app:dfm_pseudocode}
The training and inference procedures of \textbf{DFM} are summarized in pseudocode.
The training pipeline is presented in Algorithm~\ref{alg:train}, where the group drift computation is detailed in Algorithm~\ref{alg:drift}.
The inference-time generation procedure is presented in Algorithm~\ref{alg:inference}.

\begin{figure}[t]
\centering
\refstepcounter{algorithm}
\label{alg:train}
\begin{tcolorbox}[
    colback=algbg,
    colframe=algframe,
    boxrule=0.8pt,
    sharp corners,
    enhanced,
    width=\linewidth,
    left=6pt,right=6pt,top=6pt,bottom=6pt
]
{\bfseries Algorithm 1:} {\bfseries Drift Flow Matching Training}

\vspace{0.45em}

{\algcodefont

\cmt{\# u\^{}theta: Drift Flow Matching network, corresponding to $u^\theta(x_t,t,r)$} \\
\cmt{\# N: batch size; G: number of time-pair groups; B = N / G: samples per group; D: data dimension} \\
\cmt{\# grouped tensors have shape [G, B, *]; flattened tensors have shape [N, *]} \\
\cmt{\# rho: distribution over time pairs $(t,r) \sim \rho$ with $0 \le t \le r \le 1$} \\
\cmt{\# alpha(.), beta(.): interpolation schedules in the conditional path construction} \\

\vspace{0.55em}
\code{(x0, x1) <- sample\_endpoint\_pairs(p0, p1)} \hfill \cmt{\# corresponds to $(X_0, X_1) \sim p_0(x_0)p_1(x_1)$, shape [N, D]} \\
\code{(t\_raw, r\_raw) <- sample\_time\_pairs(rho, G)} \hfill \cmt{\# one sampled time pair per group, shape [G]} \\

\vspace{0.35em}
\code{x0\_grp = }\fn{reshape\_grouped}\code{(x0, G)} \hfill \cmt{\# grouped endpoint samples from $X_0$, shape [G, B, D]} \\
\code{x1\_grp = }\fn{reshape\_grouped}\code{(x1, G)} \hfill \cmt{\# grouped endpoint samples from $X_1$, shape [G, B, D]} \\

\code{t\_grp, r\_grp = }\fn{expand\_groupwise}\code{(t\_raw, r\_raw, B)} \hfill \cmt{\# broadcast group-wise times, shape [G, B, 1]} \\
\code{xt\_grp = alpha(t\_grp) * x0\_grp + beta(t\_grp) * x1\_grp} \hfill \cmt{\# corresponds to $X_t = \alpha(t)X_0 + \beta(t)X_1$, shape [G, B, D]} \\
\code{xr\_grp = alpha(r\_grp) * x0\_grp + beta(r\_grp) * x1\_grp} \hfill \cmt{\# corresponds to $X_r = \alpha(r)X_0 + \beta(r)X_1$, shape [G, B, D]} \\

\vspace{0.35em}
\code{xt\_flat = }\fn{flatten\_grouped}\code{(xt\_grp)} \hfill \cmt{\# flattened source states $X_t$, shape [N, D]} \\
\code{t\_flat = }\fn{flatten\_grouped}\code{(t\_grp)} \hfill \cmt{\# flattened time tensor for $t$, shape [N, 1]} \\
\code{r\_flat = }\fn{flatten\_grouped}\code{(r\_grp)} \hfill \cmt{\# flattened step tensor for $r$, shape [N, 1]} \\

\vspace{0.35em}
\code{xr\_hat\_flat = xt\_flat + (r\_flat - t\_flat) * }\fn{u}\code{(xt\_flat, t\_flat, r\_flat)} \hfill \cmt{\# corresponds to $\widehat X_r^\theta = X_t + (r-t)u^\theta(X_t,t,r)$, shape [N, D]} \\
\code{xr\_hat\_grp = }\fn{reshape\_grouped}\code{(xr\_hat\_flat, G)} \hfill \cmt{\# grouped predicted states, i.e. empirical $q_{t,r}^\theta$, shape [G, B, D]} \\

\vspace{0.35em}
\code{V\_grp = }\fn{compute\_grouped\_drift}\code{(xr\_hat\_grp, xr\_grp, xr\_hat\_grp)} \hfill \cmt{\# grouped Drift field $V_{q_{t,r}^\theta,p_r}$, shape [G, B, D]} \\
\code{x\_target\_grp = }\fn{stopgrad}\code{(xr\_hat\_grp + V\_grp)} \hfill \cmt{\# corresponds to $\operatorname{sg}(\widehat X_r^\theta + V_{q_{t,r}^\theta,p_r})$, shape [G, B, D]} \\
\code{x\_target\_flat = }\fn{flatten\_grouped}\code{(x\_target\_grp)} \hfill \cmt{\# flattened stop-gradient target, shape [N, D]} \\

\vspace{0.35em}
\code{L\_DFM = }\fn{mse\_loss}\code{(xr\_hat\_flat, x\_target\_flat)} \hfill \cmt{\# corresponds to the grouped realization of $\mathcal{L}_{\mathrm{DFM}}$} \\

\code{theta <- }\fn{optimizer\_step}\code{(theta, grad(L\_DFM))} \hfill \cmt{\# gradient update for model parameters $\theta$}
}
\end{tcolorbox}
\end{figure}
\begin{figure}[t]
\centering
\refstepcounter{algorithm}
\label{alg:drift}
\begin{tcolorbox}[
    colback=algbg,
    colframe=algframe,
    boxrule=0.8pt,
    sharp corners,
    enhanced,
    width=\linewidth,
    left=6pt,right=6pt,top=6pt,bottom=6pt
]
{\bfseries Algorithm 2:} {\bfseries Grouped Drift Field}

\vspace{0.45em}

{\algcodefont

\code{def }\fn{compute\_grouped\_drift}\code{(x, pos, neg, temp\_pos, temp\_neg, sinkhorn\_iters):} \\

\cmt{\# We adopt a grouped variant of the Sinkhorn Drifting Field~\cite{he2026sinkhorn} as our drift field. The only modification is that the input data now includes an additional dimension G, which indexes groups, resulting in a tensor of shape [G, B, D]. The original Sinkhorn Drifting Field~\cite{he2026sinkhorn} procedure is then applied independently over the B and D dimensions [$\cdot$, B, D] within each group. Notably, when sinkhorn\_iters = 1, our formulation reduces to that of~\cite{deng2026generative}} \\

\cmt{\# x:   [G, B, D], grouped query states, corresponding to xr\_hat\_grp} \\
\cmt{\# pos: [G, B, D], grouped positive states, corresponding to xr\_grp} \\
\cmt{\# neg: [G, B, D], grouped negative states, corresponding to xr\_hat\_grp} \\
\cmt{\# temp\_pos, temp\_neg: positive / negative kernel temperatures} \\
\cmt{\# sinkhorn\_iters: number of Sinkhorn iterations} \\

\vspace{0.55em}
\code{G, B, D = shape(x)} \\

\vspace{0.55em}
\cmt{\# pairwise distances for positive transport} \\
\code{dist\_pos = }\fn{cdist}\code{(x, pos)} \hfill \cmt{\# [G, B, B]} \\
\code{logits\_pos = -dist\_pos / temp\_pos} \hfill \cmt{\# Gibbs logits for positive coupling} \\

\vspace{0.35em}
\cmt{\# uniform marginals for grouped positive OT} \\
\code{row\_marginal\_pos = }\fn{full}\code{([B], 1.0 / B)} \\
\code{col\_marginal\_pos = }\fn{full}\code{([B], 1.0 / B)} \\

\vspace{0.35em}
\cmt{\# Sinkhorn coupling for x -> pos, batched over groups} \\
\code{plan\_pos = }\fn{sinkhorn\_from\_logits\_batched}\code{(} \\
\hspace*{1.5em}\code{logits\_pos, row\_marginal\_pos, col\_marginal\_pos, sinkhorn\_iters)} \hfill \cmt{\# [G, B, B]} \\
\code{weights\_pos = }\fn{row\_normalize\_plan}\code{(plan\_pos)} \hfill \cmt{\# [G, B, B]} \\
\code{drift\_pos = weights\_pos @ pos} \hfill \cmt{\# [G, B, D]} \\

\vspace{0.55em}
\cmt{\# pairwise distances for negative transport} \\
\code{dist\_neg = }\fn{cdist}\code{(x, neg)} \hfill \cmt{\# [G, B, B]} \\
\code{logits\_neg = -dist\_neg / temp\_neg} \hfill \cmt{\# Gibbs logits for negative coupling} \\

\vspace{0.35em}
\cmt{\# uniform marginals for grouped negative OT} \\
\code{row\_marginal\_neg = }\fn{full}\code{([B], 1.0 / B)} \\
\code{col\_marginal\_neg = }\fn{full}\code{([B], 1.0 / B)} \\

\vspace{0.35em}
\cmt{\# Sinkhorn coupling for x -> neg, batched over groups.} \\
\code{plan\_neg = }\fn{sinkhorn\_from\_logits\_batched}\code{(} \\
\hspace*{1.5em}\code{logits\_neg, row\_marginal\_neg, col\_marginal\_neg, sinkhorn\_iters)} \hfill \cmt{\# [G, B, B]} \\
\code{weights\_neg = }\fn{row\_normalize\_plan}\code{(plan\_neg)} \hfill \cmt{\# [G, B, B]} \\
\code{drift\_neg = weights\_neg @ neg} \hfill \cmt{\# [G, B, D]} \\

\vspace{0.55em}
\cmt{\# grouped cross-minus-self drift} \\
\code{V\_grp = drift\_pos - drift\_neg} \hfill \cmt{\# corresponds to $V_{q_{t,r}^\theta,p_r}$, shape [G, B, D]} \\
\code{return V\_grp}
}
\end{tcolorbox}
\end{figure}
\begin{figure}[t]
\centering
\refstepcounter{algorithm}
\label{alg:inference}
\begin{tcolorbox}[
    colback=algbg,
    colframe=algframe,
    boxrule=0.8pt,
    sharp corners,
    enhanced,
    width=\linewidth,
    left=6pt,right=6pt,top=6pt,bottom=6pt
]
{\bfseries Algorithm 3:} {\bfseries Drift Flow Matching Inference}

\vspace{0.45em}

{\algcodefont

\cmt{\# u\^{}theta: trained Drift Flow Matching network, corresponding to $u^\theta(x_t,t,r)$} \\
\cmt{\# N: number of generated samples; D: data dimension; M: number of inference steps} \\
\cmt{\# time\_grid = [t\_0, t\_1, ..., t\_M] with $0=t_0<t_1<\cdots<t_M=1$} \\

\vspace{0.55em}
\code{x = }\fn{sample\_source\_states}\code{(p0, N)} \hfill \cmt{\# corresponds to $x_{t_0} \sim p_0$, shape [N, D]} \\
\code{time\_grid = }\fn{build\_time\_grid}\code{(M)} \hfill \cmt{\# inference time grid from $0$ to $1$} \\

\vspace{0.35em}
\code{for m in range(M):} \\
\hspace*{1.5em}\code{t\_cur, t\_next = time\_grid[m], time\_grid[m + 1]} \\
\hspace*{1.5em}\code{t\_cur\_batch, t\_next\_batch = }\fn{expand\_timesteps}\code{(t\_cur, t\_next, N)} \hfill \cmt{\# both have shape [N, 1]} \\
\hspace*{1.5em}\code{x = x + (t\_next\_batch - t\_cur\_batch) * }\fn{u}\code{(x, t\_cur\_batch, t\_next\_batch)} \\
\code{end for} \\

\vspace{0.35em}
\code{return x} \hfill \cmt{\# output samples at final time $t_M = 1$}
}
\end{tcolorbox}
\end{figure}

\section{Limitations, Broader Impacts, and Reproducibility}
\label{app:checklist_support}

\paragraph{Limitations.}
Like all generative models, DFM can learn and amplify harmful, biased, private, or otherwise unsafe patterns present in the training data. This limitation is inherited from both the data distribution and the generative modeling objective, and it should be considered when applying the method to sensitive image, language, or control domains.

\paragraph{Broader impacts.}
DFM can improve the quality--efficiency trade-off of generative modeling by allowing users to choose between efficient one-step generation and more accurate multi-step refinement. Potential negative impacts include misuse of synthetic media, privacy leakage from training data, amplification of dataset biases, and unsafe deployment of generated robot policies; practical mitigations include careful data governance, respecting dataset and model licenses, safety evaluation before deployment, human oversight in high-risk applications, and controlled release when models or outputs could be misused.

\paragraph{Reproducibility and experimental details.}
The paper provides the full method formulation together with training, grouped-drift, and inference pseudocode in Appendix~\S\ref{app:dfm_pseudocode}. For the image-generation experiments, the model backbone is kept consistent with the corresponding Drift Model baseline whenever applicable; for small- and medium-scale experiments such as synthetic data, MNIST, and FFHQ, the main text specifies the datasets, latent representations, metrics, and core model settings, while for large-scale ImageNet and robotic-control experiments, all training and evaluation protocols follow the specified baselines except for the necessary DFM algorithmic and model changes. All datasets used in the experiments are publicly available, and the complete project code will be made public after acceptance.

\paragraph{Statistics, compute, and assets.}
Except for the robotic-control experiments, reported experimental results are averaged over six random seeds; the robotic-control table reports averages over the last 10 checkpoints following the stated evaluation protocol. Small- and medium-scale experiments, including synthetic data, MNIST, and FFHQ, were run on 2 NVIDIA A100 80GB GPUs, while large-scale ImageNet and robotic-control experiments were run on 16 NVIDIA A100 80GB GPUs. Existing datasets, pretrained models, codebases, and baselines used in the paper, including MNIST, FFHQ, ImageNet, ALAE, VAE or latent-MAE components, robotics benchmarks, and Drift or Diffusion Policy baselines, are cited to their original sources and used in accordance with their licenses and terms; the new project code will be released as a documented asset after acceptance.

\clearpage
\section{Generation Result}
\label{app:generation_result}

The FFHQ generated results in Figures~\ref{fig:ffhq_fake_samples_steps_1_2}
and~\ref{fig:ffhq_fake_samples_steps_5_10} correspond to the FFHQ
conditional-generation setting reported in Sec.~\ref{sec:experiment},
Table~\ref{tab:mnist_ffhq_comparison}, and
Figure~\ref{fig:ffhq_latent_inference_steps}. The drift temperature is fixed
to $\tau=1.0$, consistent with the temperature setting selected in the
ablation study in Table~\ref{tab:ablation_temp}.

\begin{figure*}[p]
  \centering
  \includegraphics[width=0.98\textwidth]{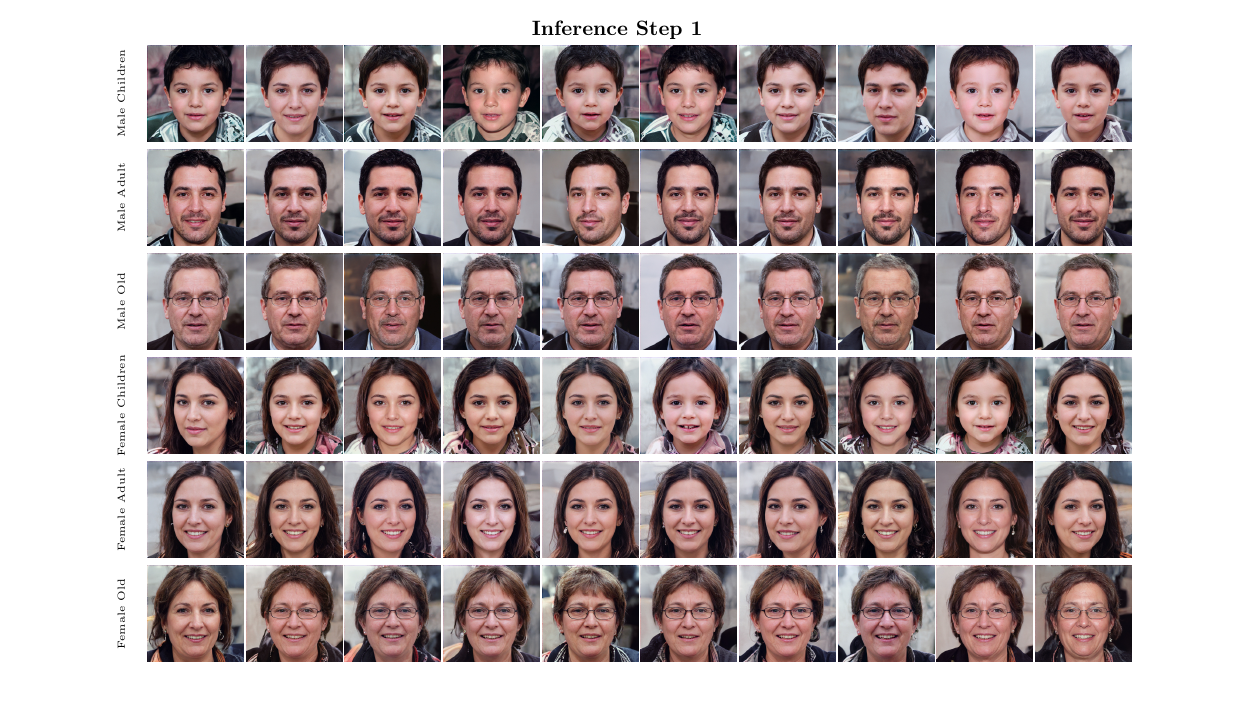}

  \vspace{0.35em}
  \includegraphics[width=0.98\textwidth]{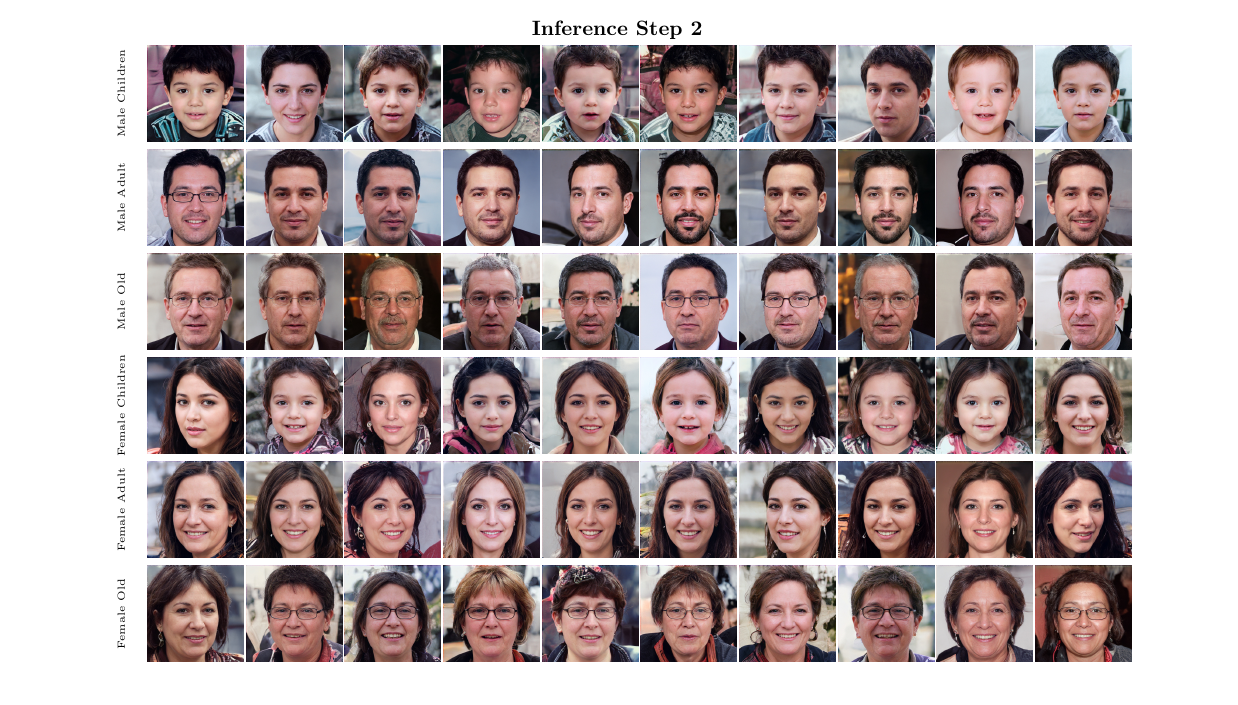}

  \caption{\textbf{FFHQ generated results at low inference steps.}
  Each row corresponds to one FFHQ class and each column shows one selected generated result. The panels show inference steps 1 (Drift Model~\cite{deng2026generative,he2026sinkhorn}) and 2.}
  \label{fig:ffhq_fake_samples_steps_1_2}
\end{figure*}

\clearpage

\begin{figure*}[p]
  \centering
  \includegraphics[width=0.98\textwidth]{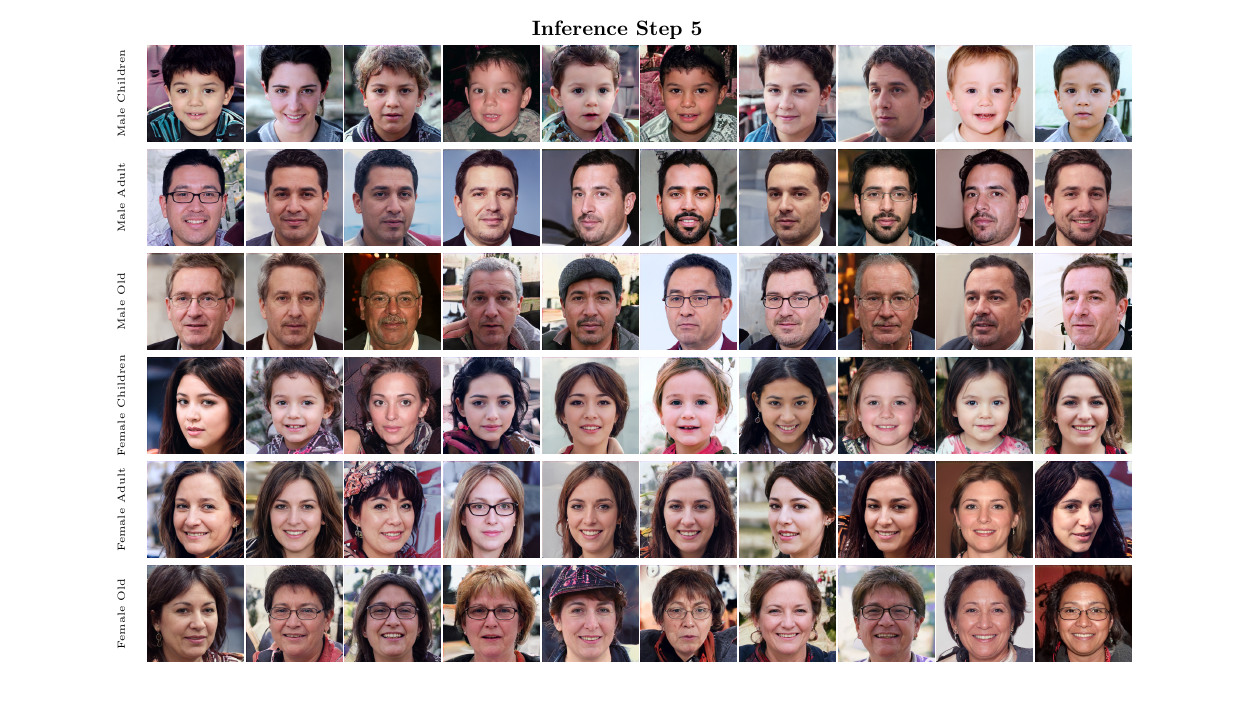}

  \vspace{0.35em}
  \includegraphics[width=0.98\textwidth]{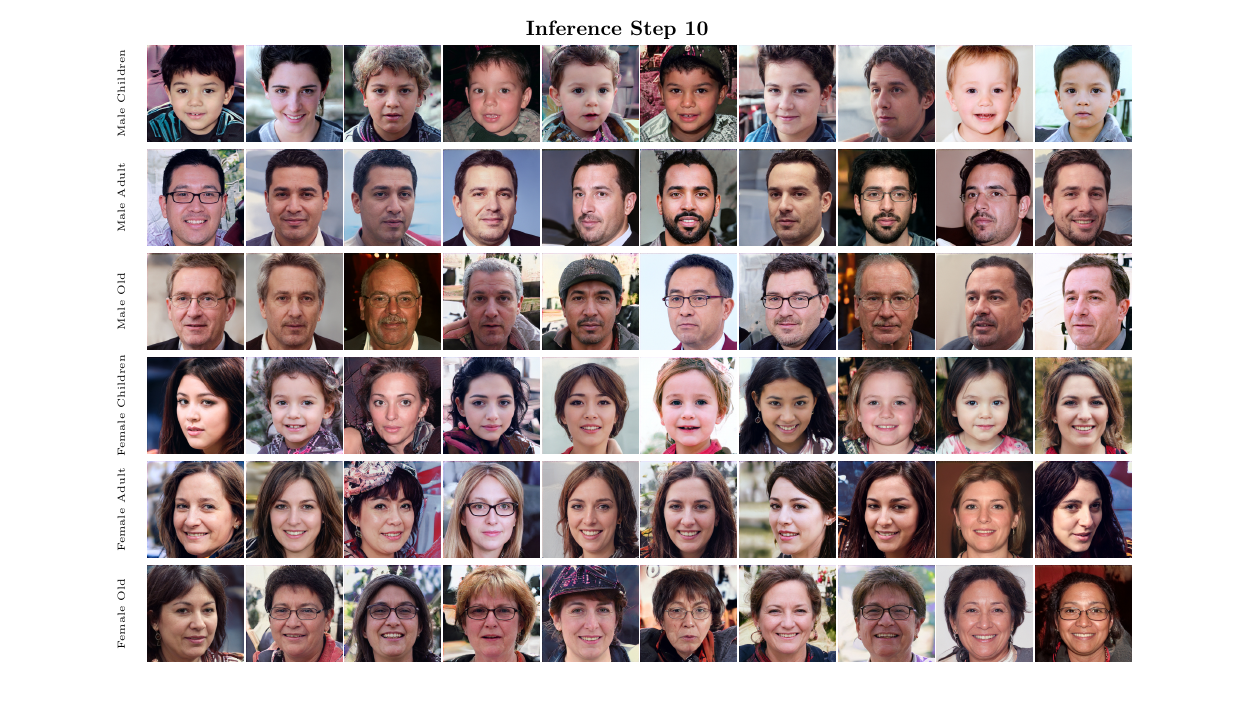}

  \caption{\textbf{FFHQ generated results at higher inference steps.}
  Each row corresponds to one FFHQ class and each column shows one selected generated result. The panels show inference steps 5 and 10.}
  \label{fig:ffhq_fake_samples_steps_5_10}
\end{figure*}

\clearpage

\clearpage



\end{document}